\documentclass[nohyperref]{article}

\usepackage{microtype}
\usepackage{graphicx}
\usepackage{subfigure}
\usepackage{booktabs} 

\usepackage{hyperref}
\hypersetup{
  colorlinks,
  linkcolor={red!50!black},
  citecolor={blue!50!black},
  urlcolor={blue!50!black}
}

\usepackage[accepted]{icml2023}
\usepackage{algorithmic}

\usepackage{amsmath}
\usepackage{amssymb}
\usepackage{mathtools}
\usepackage{amsthm}
\usepackage{bm} 
\usepackage[inline]{enumitem} 
\usepackage{dsfont}
\usepackage{setspace}
\usepackage{float}

\usepackage{upgreek}
\usepackage{xspace}

\def\gl{\mathfrak{gl}}
\def\sllie{\mathfrak{sl}}
\def\SL{\mathrm{SL}}
\def\imag{\emph{i}}
\def\su{\mathfrak{su}}
\def\so{\mathfrak{so}}

\def\GL{\mathrm{GL}}
\def\Aut{\mathrm{Aut}}
\def\Ad{\mathrm{Ad}}
\def\ad{\mathrm{ad}}
\def\SO{\mathrm{SO}}
\def\SU{\mathrm{SU}}

\newcommand{\ooint}[1]{\left(#1\right)}
\newcommand{\ccint}[1]{\left[#1\right]}

\usepackage{xcolor}

\definecolor{oxprimary}{HTML}{002147}
\definecolor{oxsecondry}{HTML}{a79d96}
\definecolor{oxtertiary}{HTML}{f3f1ee}
\definecolor{oxlightprimary}{HTML}{122f53}
\definecolor{oxverylightblue}{HTML}{f0f5f8}

\definecolor{oxblack}{HTML}{000000}
\definecolor{oxveryoffblack}{HTML}{333333}
\definecolor{oxmidgrey}{HTML}{7a736e}
\definecolor{oxdarkgrey}{HTML}{a6a6a6}
\definecolor{oxlightgrey}{HTML}{e0ded9}
\definecolor{oxvlightgrey}{HTML}{f9f8f5}
\definecolor{oxwhite}{HTML}{ffffff}

\definecolor{tabblue}{HTML}{1f77b4}
\definecolor{taborange}{HTML}{ff7f0e}
\definecolor{tabgreen}{HTML}{2ca02c}
\definecolor{tabred}{HTML}{d62728}
\definecolor{tabpurple}{HTML}{9467bd}
\definecolor{tabbrown}{HTML}{8c564b}
\definecolor{tabpink}{HTML}{e377c2}

\definecolor{brightgrey}{HTML}{f7f7f7}

\newcommand{\se}{\mathrm{SE}_{3}}

\newcommand{\RegisterPairedDelimiter}[3][1]{
\ifnum#1=1 \newcommand{#2}[2][-1]{%
\ifnum##1=-1 #3*{##2}\relax\fi%
\ifnum##1=0 #3{##2}\relax\fi%
\ifnum##1=1 #3[\big]{##2}\relax\fi%
\ifnum##1=2 #3[\Big]{##2}\relax\fi%
\ifnum##1=3 #3[\bigg]{##2}\relax\fi%
\ifnum##1=4 #3[\Bigg]{##2}\relax\fi%
}\fi%
\ifnum#1=2 \newcommand{#2}[3][-1]{%
\ifnum##1=-1 #3*{##2}{##3}\relax\fi%
\ifnum##1=0 #3{##2}{##3}\relax\fi%
\ifnum##1=1 #3[\big]{##2}{##3}\relax\fi%
\ifnum##1=2 #3[\Big]{##2}{##3}\relax\fi%
\ifnum##1=3 #3[\bigg]{##2}{##3}\relax\fi%
\ifnum##1=4 #3[\Bigg]{##2}{##3}\relax\fi%
}\fi%
}

\DeclarePairedDelimiter{\deldelim}{(}{)}
\RegisterPairedDelimiter{\del}{\deldelim}

\DeclarePairedDelimiter{\sbrdelim}{[}{]}
\RegisterPairedDelimiter{\sbr}{\sbrdelim}

\DeclarePairedDelimiter{\codelim}{[}{)}
\RegisterPairedDelimiter{\cobr}{\codelim}

\DeclarePairedDelimiter{\ocdelim}{(}{]}
\RegisterPairedDelimiter{\ocbr}{\ocdelim}

\DeclarePairedDelimiter{\dbrdelim}{\llbracket}{\rrbracket}
\RegisterPairedDelimiter{\dbr}{\dbrdelim}

\DeclarePairedDelimiter{\cbrdelim}{\{}{\}}
\RegisterPairedDelimiter{\cbr}{\cbrdelim}

\DeclarePairedDelimiter{\absdelim}{|}{|}
\RegisterPairedDelimiter{\abs}{\absdelim}

\DeclarePairedDelimiter{\normdelim}{\lVert}{\rVert}
\RegisterPairedDelimiter{\norm}{\normdelim}

\DeclarePairedDelimiterX{\innerproddelim}[3]{\langle}{\rangle}{#1,#2}
\RegisterPairedDelimiter[2]{\innerprod}{\innerproddelim}

\DeclarePairedDelimiterX{\dualproddelim}[3]{\langle}{\rangle}{#1\;\delimsize|\;\mathopen{}#2}
\RegisterPairedDelimiter[2]{\dualprod}{\dualproddelim}

\def\rset{\mathbb{R}}
\def\E{\mathbb{E}}
\def\rmd{\mathrm{d}}

\def\GL{\mathrm{GL}}

\def\so{\mathfrak{so}}
\def\se{\mathfrak{se}}
\def\nset{\mathbb{N}}

\def\cset{\mathbb{C}}

\def\TG{\mathrm{T}G}

\def\IGSO{\mathrm{IGSO}}
\def\IGSU{\mathrm{IGSU}}

\def\SO{\mathrm{SO}}
\def\SE{\mathrm{SE}}
\def\piinv{p_{\mathrm{inv}}}

\def\vareps{\varepsilon}
\def\eqsp{}

\newcommand{\dive}{\mathrm{div}}
\def\Id{\operatorname{Id}}
\def\trace{\operatorname{Tr}}
\newcommand{\rme}{\mathrm{e}}

\newcommand{\ensemble}[2]{\left\{#1\,:\eqsp #2\right\}}
\newcommand{\ensembleLigne}[2]{\{#1\,:\eqsp #2\}}

\newcommand{\normLigne}[1]{\| #1 \|}

\def\bfh{\mathbf{h}}
\def\bfz{\mathbf{z}}

\def\bfr{\mathbf{r}}

\def\bfX{\mathbf{X}}

\def\bfR{\mathbf{R}}

\def\bfx{\mathbf{x}}
\def\bfT{\mathbf{T}}

\def\bfB{\mathbf{B}}

\def\msa{\mathsf{A}}

\def\msk{\mathsf{K}}

\def\msu{\mathsf{U}}

\newcommand{\mcb}[1]{\mathcal{B}(#1)}

\def\rset{\mathbb{R}}

\def\cset{\mathbb{C}}

\def\nset{\mathbb{N}}

\def\S{\mathbb{S}}

\def\rmP{\mathrm{P}}

\def\rmd{\mathrm{d}}

\def\rme{\mathrm{e}}

\def\rmx{\mathrm{x}}
\def\rmr{\mathrm{r}}

\def\rmc{\mathrm{C}}

\def\rmK{\mathrm{K}}

\def\complementary{\mathrm{c}}

\newcommand{\R}{\mathbb R}

\newcommand{\M}{\mathcal M}
\newcommand{\dsm}{DSM}

\definecolor{RedOrange}{HTML}{F26035}

\newcommand{\dt}[2]{#1^{(#2)}}
\newcommand{\pred}[1]{\hat{#1}}
\newcommand{\sctm}{$\texttt{scTM}$}
\newcommand{\pdbtm}{$\texttt{pdbTM}$}
\newcommand{\scrmsd}{$\texttt{scRMSD}$}
\newcommand{\framediff}{$\mathrm{FrameDiff}$}
\newcommand{\framepred}{$\mathrm{FramePred}$}

\newcommand{\proteinmpnn}{$\mathrm{ProteinMPNN}$}
\newcommand{\esmfold}{$\mathrm{ESMFold}$}
\newcommand{\calpha}{\texttt{C}_\alpha}

\newcommand{\carbon}{\texttt{C}}
\newcommand{\nitrogen}{\texttt{N}}
\newcommand{\oxygen}{\texttt{O}}
\newcommand{\bbatoms}{$\nitrogen-\calpha - \carbon - \oxygen$}
\newcommand{\numres}{N}
\newcommand{\noisescale}{\zeta}
\newcommand{\numsteps}{N_{\mathrm{steps}}}
\newcommand{\numseqs}{N_{\mathrm{seq}}}
\newcommand{\res}{n}
\newcommand{\denovo}{\emph{de novo}}
\newcommand{\Pin}{\rmP_{\mathrm{pin}}}
\newcommand{\Tfinal}{\mathrm{T}_\mathrm{F}}

\usepackage{listings}
\usepackage{xcolor}

\definecolor{codegreen}{rgb}{0,0.6,0}
\definecolor{codegray}{rgb}{0.5,0.5,0.5}
\definecolor{codepurple}{rgb}{0.58,0,0.82}
\definecolor{backcolour}{rgb}{0.95,0.95,0.92}

\lstdefinestyle{mystyle}{
    backgroundcolor=\color{backcolour},   
    commentstyle=\color{codegreen},
    keywordstyle=\color{magenta},
    numberstyle=\tiny\color{codegray},
    stringstyle=\color{codepurple},
    basicstyle=\ttfamily\footnotesize,
    breakatwhitespace=false,         
    breaklines=true,                 
    captionpos=b,                    
    keepspaces=true,                 
    numbers=none,                    
    numbersep=5pt,                  
    showspaces=false,                
    showstringspaces=false,
    showtabs=false,                  
    tabsize=2
}
\usepackage{makecell}
\usepackage{diagbox}

\usepackage[capitalize,noabbrev]{cleveref}
\Crefname{algorithm}{Alg.}{Algs.}
\Crefname{equation}{Eq.}{Eqs.}
\Crefname{figure}{Fig.}{Figs.}
\Crefname{tabular}{Tab.}{Tabs.}
\Crefname{section}{Sec.}{Secs.}
\Crefname{proposition}{Prop.}{Props.}
\Crefname{appendix}{App.}{Apps.}
\usepackage[toc,page,header]{appendix}
\usepackage{minitoc}
\usepackage{autonum}

\theoremstyle{plain}
\newtheorem{theorem}{Theorem}[section]
\newtheorem{proposition}[theorem]{Proposition}
\newtheorem{lemma}[theorem]{Lemma}
\newtheorem{corollary}[theorem]{Corollary}
\theoremstyle{definition}
\newtheorem{definition}[theorem]{Definition}

\theoremstyle{remark}

\usepackage[textsize=tiny]{todonotes}

\newcommand{\appendixhead}{
  \centerline{\textbf{\LARGE Supplementary to: }\vspace{0.15in}}
  \centerline{\textbf{\LARGE $\SE(3)$ diffusion model}\vspace{0.1in}}
  \centerline{\textbf{\LARGE with application to protein backbone generation}\vspace{0.25in}}
}

\icmltitlerunning{SE(3) diffusion model with application to protein backbone generation}

\begin{document}

\setlength{\abovedisplayskip}{10pt}
\setlength{\belowdisplayskip}{10pt}

\doparttoc 
\faketableofcontents 


\twocolumn[
\icmltitle{$\SE(3)$ diffusion model with application to protein backbone generation}
%
%

\icmlsetsymbol{equal}{*}

\begin{icmlauthorlist}

\icmlauthor{Jason Yim}{equal,mit}
\icmlauthor{Brian L. Trippe}{equal,columbia}
\icmlauthor{Valentin De Bortoli}{equal,cnrs}
\icmlauthor{Emile Mathieu}{equal,cambridge}
\icmlauthor{Arnaud Doucet}{oxford}  
\icmlauthor{Regina Barzilay}{mit}
\icmlauthor{Tommi Jaakkola}{mit}
\end{icmlauthorlist}

\icmlaffiliation{mit}{Department of Electrical Engineering and Computer Science, Massachusetts Institute of Technology, Massachusetts, USA}
\icmlaffiliation{columbia}{Department of Statistics, Columbia University, New York, USA}
\icmlaffiliation{cnrs}{Center for Sciences of Data, French National Centre for Scientific Research, Paris, France}
\icmlaffiliation{cambridge}{Department of Engineering, University of Cambridge, Cambridge, United Kingdom}
\icmlaffiliation{oxford}{Department of Statistics, University of Oxford, Oxford, United Kingdom}

\icmlcorrespondingauthor{Jason Yim}{jyim@csail.mit.edu}

\icmlkeywords{Machine learning, Generative models, Protein design, Geometric deep learning, Diffusion models, ICML}

\vskip 0.3in
]



\printAffiliationsAndNotice{\icmlEqualContribution} 

\begin{abstract}
The design of novel protein structures remains a challenge in protein engineering for applications across biomedicine and chemistry.
In this line of work, a diffusion model over rigid bodies in 3D (referred to as \emph{frames}) has shown success in generating novel, functional protein backbones that have not been observed in nature.
However, there exists no principled methodological framework for diffusion on $\SE(3)$, the space of orientation preserving rigid motions in $\rset^3$, that operates on frames and confers the group invariance.
We address these shortcomings by developing theoretical foundations of $\SE(3)$ invariant diffusion models on multiple frames followed by a novel framework, \framediff{}, for learning the $\SE(3)$ equivariant score over multiple frames.
We apply \framediff{} on monomer backbone generation and find it can generate designable monomers up to 500 amino acids without relying on a pretrained protein structure prediction network that has been integral to previous methods.
We find our samples are capable of generalizing beyond any known protein structure.
Code: \url{https://github.com/jasonkyuyim/se3_diffusion}
\end{abstract}

\section{Introduction}
\label{sec:introduction}
\begin{figure*}[!ht]
\begin{center}
\centerline{\includegraphics[width=\textwidth]{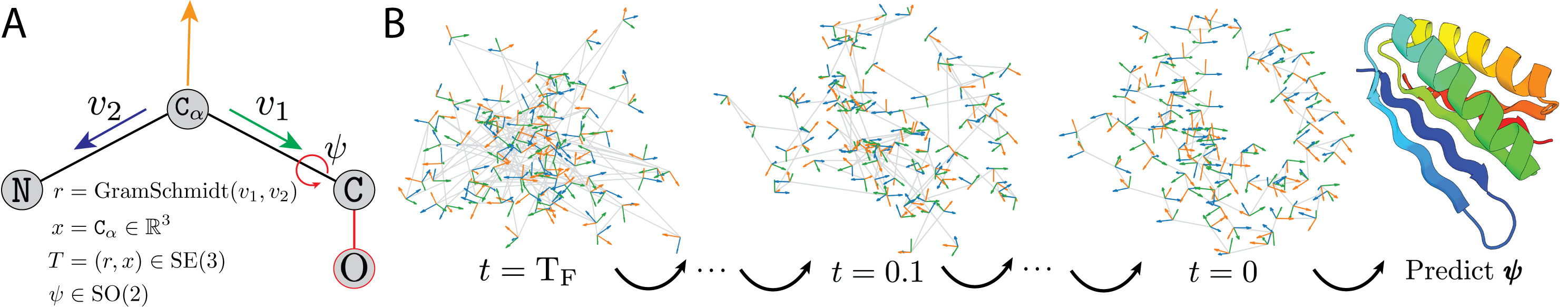}}
\caption{
Method overview. \textbf{(A)} Backbone parameterization with frames. Each residue along the protein chain shares the same structure of backbone atoms due to the fixed bonds between each atom. 
Performing the GramSchmidt operation on vectors $v_1, v_2$ results in rotation matrix $r$ that parameterizes the $\nitrogen - \calpha - \carbon$ placements with respect to the frame translation, $x$, set to the $\calpha$ coordinates.
An additional torsion angle, $\psi$, is required to determine the placement of the oxygen atom, $\oxygen$.
\textbf{(B)} Inference is performed by sampling $\numres$ frames initialized from the reference distribution over rotations and translations.
Then a time-reversed $\SE(3)$ diffusion is run from $t=\Tfinal$ to $t=0$ at which point the $\psi$ angle is predicted. The final frames and $\psi$ angles are used to construct the protein backbone atoms.
}
\label{fig:method_overview}
\end{center}
\vskip -0.4in
\end{figure*}
The ability to engineer novel proteins holds promise in developing
bio-therapeutics towards global health challenges such as SARS-COV-2
\citep{arunachalam2021adjuvanting} and cancer \citep{QUIJANORUBIO2020119}.
Unfortunately, efforts to engineer proteins have required substantial domain
knowledge and laborious experimental testing.
To this end, protein engineering has benefited from advancements in deep learning by automating knowledge
acquisition from data and improving efficiency in designing proteins
\citep{ding2022protein}.

Generating a novel protein satisfying specified structural or functional properties is the task of \denovo{} protein design \citep{huang2016coming}.
In this work, we focus on generating protein backbones.
A protein backbone consists of $\numres$ residues, each with four heavy atoms rigidly connected via covalent bonds, \bbatoms.
Computationally designing novel backbones is technically challenging due to the coupling of structure and sequence: 
atoms that comprise protein structure must adhere to physical and chemical constraints
while being ``designable" in the sense that there exists a sequence of amino acids which folds to that structure.
We approach this problem with diffusion generative modeling which has shown promise in recent work (see \cref{sec:related_work}).

A main technical challenge is to combine expressive geometric deep learning methods that operate on protein structures with diffusion generative modeling.
Because the $\nitrogen-\calpha - \carbon$ atoms for each residue may be described accurately as a frame (\cref{fig:method_overview}A), many successful computational methods for both protein structure prediction \citep{Jumper2021HighlyAP} and design \citep{watson2022rfdiffusion} represent backbone structures by an element of the Lie group $\SE(3)^N.$
Moreover, since the biochemical function of proteins is imparted by the relative geometries of the atoms (and so is invariant to rigid transformations) these methods typically utilize $\SE(3)$ equivariant neural networks.\footnote{ $\SE(3)^N$ is the manifold of $N$ frames while $\SE(3)$ equivariance refers to the equivariance on global rotations and translations.}
While \citet{debortoli2022Riemannian,huang2022Riemannian} have extended diffusion modeling to Riemannian manifolds (such as $\SE(3)$), these works do not readily provide tractable training procedures or accommodate inclusion of geometric invariances.

Modeling $\SE(3)^N$ poses theoretical challenges 
and current deep learning methods
have outpaced theoretical foundations.
\citet{watson2022rfdiffusion} demonstrated a diffusion model (RFdiffusion) to generate novel protein-binders with high, experimental-verified affinities,
but relied on a heuristic denoising loss and 
required pretraining on protein structure prediction.
Our goal is to bridge this theory-practice gap and develop a principled method without pretraining.

The contribution of this work is on the theory and methodology of $\SE(3)$ diffusion models with applications to protein backbone generation.
First, we construct a diffusion process on $\SE(3)^N$. In \cref{sec:diffusion_process}, we characterize the distribution of the Brownian motion on compact Lie groups (with a focus on $\SO(3)$) in a form amenable for denoising score matching (\dsm) training and define a forward process on $\SE(3)^\numres$ that allows for separation of translations and rotations.
We show that an $\SE(3)$ invariant process on $\SE(3)^\numres$ can only be made translation invariant by keeping the diffusion process centered at the origin since no $\rset^3$ invariant probability measure exists.
Second, we implement our theory as a $\SE(3)$ invariant diffusion model on $\SE(3)^\numres$ for protein backbones.
We refer to our method as \framediff{} and describe it in \cref{sec:diffusion_model}.
Empirically, we find through experiments in \cref{sec:experiment} that \framediff{} can generate designable, diverse, and novel protein monomers up to length 500.
Compared to other methods, \framediff{} achieves \emph{in-silico} designability success rates that are second only to RFdiffusion, a pretrained model with 4-fold more parameters.
Our contributions will enable further advancements in $\SE(3)$ diffusion
methodology that underlies RFdiffusion and \framediff{} for proteins as well as
other domains such as robotics where $\SE(3)$ and other Lie groups are used.

\section{Preliminaries and Notation}
\label{sec:notation}

\paragraph{Backbone parameterization.}
\label{sec:param_proteins}
We adopt the backbone frame parameterization used in AlphaFold2 (AF2)
\citep{Jumper2021HighlyAP}.  Here, an $\numres$ residue backbone is
parameterized by a collection of $\numres$ orientation preserving rigid transformations, or
\emph{frames}, that map from fixed coordinates
$\nitrogen^\star,\carbon_\alpha^\star, \carbon^\star, \oxygen^* \in \R^3$
centered at $\carbon_{\alpha}^\star= (0, 0, 0)$ (\Cref{fig:method_overview}A).
Each fixed coordinate assumes chemically idealized bond angles and lengths
measured experimentally \citep{engh2012structure}.  For each residue indexed by
$\res$, the backbone main atom coordinates are given by
\begin{align}
\label{eq:backboneatoms}
[\nitrogen_n, \carbon_n, (\calpha)_n] &= T_n \cdot [\nitrogen^\star, \carbon^\star, \carbon_{\alpha}^\star],
\end{align}
where $T_n$ is a member of the special Euclidean group $\SE(3)$, the set of orientation preserving
rigid transformations in Euclidean space.  Each $T_n$ may be decomposed into two
components $T_n=(r_n, x_n)$ where $r_n \in \SO(3)$ is a $3\times 3$ rotation
matrix and $x_n \in \rset^3$ represents a translation; for a coordinate
$v\in\R^3,$\ $T_n\cdot v= r_n v + x_n$ denotes the action of $T_n$ on
$v.$ Together, we collectively denote all $\numres$ frames as
$\bfT=[T_1, \dots, T_\numres]\in\SE(3)^\numres.$ With an additional torsion
angle $\psi$, we may construct the backbone oxygen by rotating $\oxygen^\star$
around the bond between $\calpha$ and $\carbon$.  \cref{sec:ideal_coordinates}
provides additional details on this mapping and idealized coordinates. 

\paragraph{Diffusion modeling on manifolds.}\label{sec:riemannian_diffusion_background}
To capture a distribution over backbones in $\SE(3)^\numres$ we build on the
Riemannian score based generative modeling approach of
\citet{debortoli2022Riemannian}.  We briefly review this approach. The goal of
Riemannian score based generative modeling is to sample from a distribution
$\dt{\bfX}{0}\sim p_0$ supported on a Riemannian manifold $\M$ by reversing a
stochastic process that transforms data into noise.  One first constructs an
$\M$-valued \emph{forward process} $(\dt{\bfX}{t})_{t\geq 0}$ that evolves from
$p_0$ towards an invariant density\footnote{density w.r.t. the volume form on
  $\M$.} $\piinv(x) \propto \rme^{-U(x)}$ following
\begin{equation}
  \label{eq:sde_manifold}
  \textstyle{
    \rmd \dt{\bfX}{t} = {- \tfrac{1}{2}\nabla U(\dt{\bfX}{t})} \rmd t + \rmd \dt{\bfB_\M}{t} , \qquad  \dt{\bfX}{0} \sim p_0,
    }
\end{equation}
where $\dt{\bfB_\M}{t}$ is the Brownian motion on $\M.$ The time-reversal of
this process is given by the following proposition.
\begin{proposition}[Time-reversal, \citet{debortoli2022Riemannian}]\label{prop:manifold_reversal}
  Let $\Tfinal > 0$ and $\dt{\overleftarrow{\bfX}}{t}$ given by $\dt{\overleftarrow{\bfX}}{0} \stackrel{d}{=} \dt{\bfX}{\Tfinal}$ and 
  \begin{equation}
    \label{eq:time_reversal_manifold}
    \textstyle{
      \rmd \dt{\overleftarrow{\bfX}}{t} = \{\tfrac{1}{2} \nabla U(\dt{\overleftarrow{\bfX}}{t})+ \nabla \log p_{\Tfinal-t}(\dt{\overleftarrow{\bfX}}{t})\} \rmd t + \rmd \dt{\bfB_\M}{t},
      }
    \end{equation}
    where $p_t$ is the density of $\dt{\bfX}{t}$.  Then under mild assumptions
    on $\M$ and $p_0$ we have that $\dt{\overleftarrow{\bfX}}{t} \stackrel{d}{=} \dt{\bfX}{\Tfinal-t}$.
\end{proposition}
Diffusion modeling in Euclidean space is a special case of \Cref{prop:manifold_reversal}. 
However, generative modeling using this reversal beyond the Euclidean setting requires additional mathematical machinery, which we now review.

\textbf{Riemannian gradients and Brownian motions.}
In the above, $\nabla U(x)$ and $\nabla \log p_t(x)$ are \emph{Riemannian
  gradients} taking values in $\mathrm{Tan}_{x} \M,$ the tangent space of $\M$
at $x$,
and depend implicitly on
the choice of an inner product on $\mathrm{Tan}_{x} \M$, denoted by
$\langle\cdot, \cdot\rangle_\M.$
Similarly, the Brownian motion relies on
$\langle \cdot, \cdot\rangle_\M$ through the Laplace--Beltrami operator,
$\Delta_\M,$ which dictates its
density through the Fokker-Planck equation in the absence of drift; if $\pi_t$ is the density of the $\dt{\bfB_\M}{t}$ then ${\partial_t } \pi_t = \frac{1}{2}\Delta_{\M} \pi_t$.
We refer the reader to \citet{lee2013smooth} and \citet{hsu2002stochastic} for background on differential geometry and
stochastic analysis on manifolds.

\textbf{Denoising score matching.}
The quantity $\nabla \log p_t $ is called
the Stein score and is unavailable in practice.  It is approximated with a score
network $s_\theta(t,\cdot)$ trained by minimizing a denoising score
matching (\dsm) loss
\begin{equation}
\label{eq:denoising_sm}
\textstyle{
\mathcal{L}(\theta) = \E [ \lambda_t \normLigne{ \nabla \log p_{t|0}(\dt{\bfX}{t}|\dt{\bfX}{0}) -
  s_\theta(t, \dt{\bfX}{t}) }^2],
}
\end{equation}
where $p_{t|0}$ is the density of $\dt{\bfX}{t}$ given $\dt{\bfX}{0}$,
$\lambda_t >0$ a weight, and the expectation is taken over
$t \sim \mathcal{U}([0, \Tfinal])$ and $(\dt{\bfX}{0}, \dt{\bfX}{t})$.  For an
arbitrarily flexible network, the minimizer
$\theta^\star = \mathrm{argmin}_{\theta}\mathcal{L}(\theta)$ satisfies
$s_{\theta^\star}(t, \cdot)=\nabla \log p_t$.

\textbf{Lie groups}
are Riemannian manifolds with an additional group structure,
i.e.\ there exists an operator $*: G \times G \rightarrow G$ such that $(G, *)$
is a group and $*$ as well as its inverse are smooth.
We define the left action as $L_g(h) = g * h$ for any $g,h \in G$ and its differential is denoted by $\rmd L_{g}(h): \mathrm{Tan}_g G \rightarrow \mathrm{Tan}_{g * h} G$.
$\SO(3)$, $\SE(3)$ and $\rset^3$ are all Lie groups.
For any group $G$, we denote $\mathfrak{g}$ its Lie algebra.
We refer to \citet{sola2018micro} for an introduction to Lie groups.

\textbf{Additional notation.}
Superscripts with parentheses are reserved for time, i.e.\ $x^{(t)}$.
Uppercase is used to denotes random variables, e.g.\ $X \sim p$, and lower case
is used for deterministic variables.  Bold denotes concatenated versions of
variables, e.g.\ $\mathbf{x} = (x_1, \dots, x_N)$ or processes
$(\dt{\bfX}{t})_{t \in \ccint{0,\Tfinal}}$.


\section{Diffusion models on $\SE(3)$}
\label{sec:method}

\label{sec:diffusion_process}

Parameterizing flexible distributions over protein backbones, leveraging the Riemannian diffusion method
of \Cref{sec:riemannian_diffusion_background} to $\SE(3)^\numres$,
requires several ingredients.
First, in \Cref{sec:forw-diff-se3} we develop a forward diffusion process on $\SE(3)$, then \Cref{sec:dsm_SE3} derives DSM training on compact Lie groups, using $\SO(3)$ as the motivating example.
At this point, a diffusion model on $\SE(3)^\numres$ is defined.
Next, because incorporating invariances can improve data efficiency and generalization \citep[e.g.][]{elesedy2021provably} we desire $\SE(3)$ invariance where the $\SE(3)^\numres$ data distribution is invariant to global rotations and translations.
\cref{sec:inv_processes} will show this is not possible without centering the process at the origin and having a $\SO(3)$-equivariant neural network.

\subsection{Forward diffusion on $\SE(3)$}\label{sec:forw-diff-se3}
%
In contrast to Euclidean space and compact manifolds, no canonical
forward diffusion on $\SE(3)^\numres$ exists, and we must define one.  This entails (a)
choosing an inner product on $\SE(3)$ to define a Brownian motion and (b)
choosing a reference measure for the forward diffusion.

We begin with the inner product, which we derive from the canonical inner products for $\SO(3)$ and $\R^3$ which we recall below--see \citet{1992Riemannian}.
For $u, v \in \so(3)$ and $x, y \in \R^3$
\begin{equation}
\textstyle{
\langle u, v \rangle_{\SO(3)}
= \mathrm{Tr}(u v^\top)/2  \mathrm{ \ and \  }
\langle x , y \rangle_{\R^3} = \sum_{i=1}^3 x_i y_i,}
\end{equation}
In the next proposition,
we show that, under an appropriate choice of inner product, $\SE(3)$ can be identified
with $\SO(3) \times \R^3$ from a \emph{Riemannian} point of view, thereby providing a Laplace-Beltrami operator and a well-defined Brownian
motion.

\begin{proposition}[Metric on $\SE(3)$]
  \label{prop:brownian}
  For any $T \in \SE(3)$ and
  $(a, x), (a^\prime, x^\prime) \in \mathrm{Tan}_T\SE(3)$ we define
  $\langle (a, x), (a^\prime, x^\prime)\rangle_{\SE(3)}= \langle a,
  a^\prime\rangle_{\SO(3)} + \langle x, x^\prime \rangle_{\R^3}.$ We have:
  \begin{enumerate}[label=(\alph*)]
  \item for any $f \in \rmc^\infty(\SE(3))$ and $T=(r,x) \in \SE(3)$,
    $\nabla_T f(T) = [\nabla_r f(r,x), \nabla_x f(r,x)]$,
  \item for any $f \in \rmc^\infty(\SE(3))$ and $T=(r,x) \in \SE(3)$, 
    $\Delta_{\SE(3)}f(T) = \Delta_{\SO(3)}f(r,x) + \Delta_{\rset^3}f(r,x)$,
  \item for any $t>0$,
  $\dt{\bfB_{\SE(3)}}t = [\dt{\bfB_{\SO(3)}}{t}, \dt{\bfB_{\R^3}}{t}]$ with
  independent $\dt{\bfB_{\SO(3)}}{t}$ and $\dt{\bfB_{\R^3}}{t}.$
  \end{enumerate}
\end{proposition}
Other choices of metric for $\SE(3)$ are possible, leading to different definitions of the exponential and Brownian motion.
Our choice has the advantage of simplicity and allows to treat $\SO(3)$ and $\rset^3$ forward processes independently (conditionally on $\dt{\bfT}{0}$).
%
For the invariant density of $T=(r, x)$, we choose 
$\piinv^{\SE(3)}(T) \propto\mathcal{U}^{\SO(3)}(r)~\mathcal{N}(x;0, \Id_3)$.
The associated forward process
$(\dt{\bfT}{t})_{t \geq 0} = (\dt{\bfR}{t},\dt{\bfX}{t})_{t \geq 0}$ is given
according to \eqref{eq:sde_manifold} and \cref{prop:brownian} by
\begin{equation}
  \label{eq:forward_SE3}
\rmd \dt{\bfT}{t} = [0, -\tfrac{1}{2}\dt{\bfX}{t}]\rmd t + [\rmd \dt{\bfB_{\SO(3)}}{t}, \rmd \dt{\bfB_{\R^3}}{t}] .
\end{equation}


\subsection{Denoising score matching on $\SE(3)$}\label{sec:dsm_SE3}

As a consequence of \Cref{prop:brownian} and the independence of the rotational
and translational components of the forward process, we have
$\nabla_{\dt{\bfT}{t}} \log
p_{t|0}(\dt{\bfT}{t}|\dt{\bfT}{0})=[\nabla_{\dt{\bfR}{t}} \log
p_{t|0}(\dt{\bfR}{t}|\dt{\bfR}{0}), \nabla_{\dt{\bfX}{t}} \log
    p_{t|0}(\dt{\bfX}{t}|\dt{\bfX}{0})]$ and we can compute these quantities
\emph{independently} over the rotation and translation components.

\textbf{Denoising score matching on $\SO(3)$.}
The forward process $(\dt{\bfR}{t})_{t \geq 0}$ is simply the Brownian motion on
$\SO(3)$,
and $p_{t|0}$ is defined by the heat
kernel, see \citet{hsu2002stochastic}.  We obtain $p_{t|0}$ analytically as a series as a special case of
the decomposition of the heat kernel for compact Lie groups.
\begin{proposition}[Brownian motion on compact Lie groups]
  \label{prop:lie_group_transition}
  Assume that $\M$ is a compact Lie group, where for any $\ell \in \nset$
  $\chi_\ell$ is the character associated with the
  irreducible unitary representation 
  of dimension $d_\ell$. Then
  $\chi_\ell: \ \M \to \rset$ is an eigenvector of $\Delta$ and there exists $\lambda_\ell\geq 0 $
  such that $\Delta \chi_\ell = - \lambda_\ell \chi_\ell$. In addition, we have
  for any $t >0 $ and $\dt{x}{0}, \dt{x}{t} \in \M$, $\textstyle{p_{t|0}(\dt{x}{t}|\dt{x}{0})
      = \sum_{\ell\in \nset} d_\ell e^{-\lambda_\ell t/2}
      \chi_\ell((\dt{x}{0})^{-1} \dt{x}{t}).}$
    
\end{proposition}


%

Combining \Cref{prop:lie_group_transition} and the explicit expression of
irreducible characters for $\SO(3)$
provides an explicit expression for the density transition kernel $\dt{\bfB_{\SO(3)}}{t}.$
In \Cref{sec:heat-kernel-compact}, we showcase another application of our method
by computing the heat kernel on $\SU(2)$.

\begin{proposition}[Brownian motion on $\SO(3)$]\label{prop:brownian_on_SO3}
  For any $t >0$ and $\dt{r}{0}, \dt{r}{t} \in \SO(3)$ we have that
  $p_{t|0}(\dt{r}{t}|\dt{r}{0}) = \IGSO_3(\dt{r}{t} ; \dt{r}{0}, t)$ given by
  $\IGSO_3(\dt{r}{t} ; \dt{r}{0}, t) = f(\omega(r^{(0)\top} \dt{r}{t}),
  t)$, 
  where $\omega(r)$ is the rotation angle in radians for any
  $r \in \SO(3)$---its length in the axis--angle representation\footnote{See
    \Cref{sec:param-so3} for  details about the parameterization of $\SO(3)$.}--- and
\begin{equation} \label{eq:so3_heat_kernel}
f(\omega, t) 
= \textstyle{\sum_{\ell\in \nset} (2 \ell + 1) \rme^{-\ell(\ell+1) t/2} 
  \tfrac{\sin((\ell+1/2)\omega)}{\sin(\omega/2)}.}
 \end{equation}
\end{proposition}
\Cref{prop:brownian_on_SO3} agrees with previous proposed expressions of the law of the Brownian motion \citep{nikolayev1970normal,leach2022denoising} up to a two-fold deceleration of time.
This deceleration is crucial to the correct application of \Cref{prop:manifold_reversal} (see \Cref{sec:igso3_time_scaling_discussion_and_related_work} for details).

Accurate values of the Brownian density \eqref{eq:so3_heat_kernel} can easily be obtained by truncating the series.
Also, although exact sampling is not available, accurate
samples can be obtained by numerically inverting the
cdf~\citep{leach2022denoising}.  Moreover, this density allows computation of
the conditional score required by the $\mathrm{dsm}$ loss.
\begin{proposition}[Score on $\SO(3)$]
  \label{prop:deno-score-match}
  For $t>0$, $\dt{r}{0}, \dt{r}{t} \in \SO(3)$, we have 
\begin{equation}
  \label{eq:conditional_score}
\nabla \log  p_{t|0}(\dt{r}{t}\mid \dt{r}{0})=
\tfrac{\dt{r}{t}}{
\omega^{(t)}
} \log\{r^{(0, t)}\} 
 \frac{\partial_\omega f(\omega^{(t)}, t)}{f(\omega^{(t)}, t)} ,
\end{equation}
with $r^{(0, t)}=r^{(0)\top}\dt{r}{t}$, 
$\omega^{(t)} = \omega(r^{(0, t)})$ and $\log$ the inverse of the
exponential on $\SO(3),$ i.e.\ the matrix logarithm.
\end{proposition}
%

\textbf{Denoising score matching on $\R^3$.}
%
 The process $(\dt{\bfX}{t})_{t \geq 0}$ is an
Ornstein--Uhlenbeck process, see \eqref{eq:forward_SE3}, (also called
VP-SDE \cite{song2020score}) and converges geometrically to
$\mathcal{N}(0,\Id)$.  In addition, 
$p_{t|0}(\dt{x}{t}|\dt{x}{0})=\mathcal{N}(\dt{x}{t};
\rme^{-t/2}\dt{x}{0}, (1-\rme^{-t})\Id_3)$ and the corresponding conditional
score can be computed explicitly as
\begin{equation}
\nabla\log p_{t|0}(\dt{x}{t}|\dt{x}{0})
= (1-\rme^{-t})^{-1}(\rme^{-t/2}\dt{x}{0}-\dt{x}{t}).
\end{equation}

\subsection{$\SE(3)$ invariance through centered $\SE(3)^\numres$} \label{sec:inv_processes}
In this subsection, we show how one can construct a diffusion process over $\SE(3)^\numres$ that is invariant to global translations and rotations.
Formally, we want to design a measure $\mu$ on $\SE(3)^N$ such that for any $T_0 \in \SE(3)$, 
and measurable
$\msa \subset \SE(3)^\numres,\ \ \mu(\msa) = \mu(\{T_0\cdot \bfT \ , \ \bfT\in
\msa\})$, where for any $\bfT=(T_1, \cdots, T_N)$,
$T_0 \cdot \bfT = (T_0 T_1, \dots, T_0 T_N)$. Unfortunately, there exists no
probability measure on $\SE(3)^N$ which is $\SE(3)$ invariant since there exists
no probability measure on $\rset^{3N}$ which is $\rset^3$ invariant. As a
result, no output of a $\SE(3)^N$-valued diffusion model can be $\SE(3)$ invariant.
However, we will show $\SE(3)$ invariance is achieved by keeping the diffusion process always centered at the origin.

\textbf{From $\SE(3)$ to $\SO(3)$ invariance.}
We show that we can construct an invariant \emph{measure} on $\SE(3)^N$ by keeping the center of mass fixed to zero, i.e.\ $\sum_{n=1}^\numres x_n = 0$.
Formally, this defines a subgroup of $\SE(3)^\numres$ denoted $\SE(3)^\numres_0$ with elements $[(r_1, x_1), \dots, (r_\numres, x_\numres)]$, which we refer to as \emph{centered} $\SE(3)$.
Note that $\SE(3)^\numres_0$ is still a Lie group and $\SO(3)$ is a subgroup of
$\SE(3)^\numres_0$.
    
\begin{proposition}[Disintegration of measures on $\SE(3)^N$]
  \label{sec:from-se3-so3}
  Under mild assumptions\footnote{See \Cref{sec:conn-betw-so_3rs} for a precise
    statement.}, for every $\SE(3)$-invariant measure $\mu$ on $\SE(3)^N$, there
  exist $\eta$ an $\SO(3)$-invariant probability measure on $\SE(3)^N_0$ and
  $\bar{\mu}$ proportional to the Lebesgue measure on $\rset^3$ such that
  \begin{align}
    &\textstyle{\rmd \mu([(r_1,x_1), ..., (r_N,x_N)]) = \rmd \bar{\mu}(\tfrac{1}{N}\sum_{i=1}^N x_i)} \\&  \textstyle{\times \rmd \eta([(r_1, x_1-\tfrac{1}{N}\sum_{i=1}^N x_i), ... , (r_N, x_N-\tfrac{1}{N}\sum_{i=1}^N x_i)])  .}
  \end{align}
\end{proposition}

The previous proposition is based on the \emph{disintegration of measures}
\cite{pollard2002user}. The converse is also true. In practice this means that
in order to define a $\SE(3)$-invariant measure on $\SE(3)^\numres$ one needs
only to define an $\SO(3)$-invariant measure on $\SE(3)^\numres_0$. This is the
goal of the next paragraph.

\textbf{Diffusion models on $\SE(3)^N_0$.}
A simple modification of the forward process \eqref{eq:forward_SE3} yields a
stochastic process on $\SE(3)_0^N$.  Indeed consider $(\dt{\bfT}{t})_{t \geq 0}$
on $\SE(3)^N$ given by 
\begin{equation}
  \label{eq:forward_se30}
  \rmd \bfT^{(t)} = [0, -\tfrac{\mathrm{1}}{2}\mathrm{P}\bfX^{(t)}] \rmd t + [ \rmd \bfB^{(t)}_{\SO(3)^N}, \mathrm{P} \rmd \bfB^{(t)}_{\rset^{3N}}] ,
\end{equation}
where $\mathrm{P} \in \rset^{3 N \times 3N}$ is the projection matrix removing
the center of mass $\tfrac{1}{\numres}\sum_{n=1}^\numres x_n$.  Then
$(\dt{\bfT}{t})_{t\geq 0} = (\dt{\bfR}{t}, \dt{\bfX}{t})_{t \geq 0}$ is a
stochastic process on $\SE(3)^\numres_0$ with invariant measure
$\mathrm{P}_\# (\mathcal{N}(0, \Id)^{\otimes N}) \otimes
\mathcal{U}(\SO(3))^{\otimes \numres}$\footnote{$\rmP_{\#}$ is the pushforward by $\rmP$.}.
We note that such `center of mass free' systems have been proposed for continuous normalizing flows and discrete time diffusion models \citep{kohler2020Equivariant,xu2022geodiff}.
An application of 
\Cref{prop:manifold_reversal,prop:brownian} shows
that the backward process
$(\dt{\overleftarrow{\bfT}}{t})_{t \in \ccint{0,\Tfinal}} = ([\dt{\overleftarrow{\bfR}}{t}, \dt{\overleftarrow{\bfX}}{t}])_{t \in
  \ccint{0,\Tfinal}}$ is given by
\begin{align}
  \label{eq:backward_se30}
  &\rmd \overleftarrow{\bfR}^{(t)} =  \nabla_{r} \log p_{\Tfinal-t}(\overleftarrow{\bfT}^{(t)}) \rmd t + \rmd \bfB^{(t)}_{\SO(3)^N}, \\ 
  & \rmd \overleftarrow{\bfX}^{(t)} = \mathrm{P}\{\tfrac{\mathrm{1}}{2}\overleftarrow{\bfX}^{(t)} + \nabla_{x} \log p_{\Tfinal-t}(\overleftarrow{\bfT}^{(t)})\} \rmd t + \mathrm{P} \rmd \bfB^{(t)}_{\rset^{3N}}.
\end{align}
As in \Cref{sec:dsm_SE3}, we have
$p_{t|0}((\dt{\bfr}{t},\dt{\bfx}{t})|(\dt{\bfr}{0},\dt{\bfx}{0})) =
p_{t|0}(\dt{\bfr}{t}|\dt{\bfr}{0})p_{t|0}(\dt{\bfx}{t}|\dt{\bfx}{0})$,
where these densities additionally factorizes along each of the residues.
In \Cref{sec:training}, we use the forward process \eqref{eq:forward_se30} for
training and the backward process \eqref{eq:backward_se30} for sampling in
\Cref{sec:sampling}.

\textbf{Invariance and equivariance  on Lie groups.}
Finally, we want the output of the backward process, i.e. the distribution of
$(\bfR^{(t)}, \bfX^{(t)})$ given by \eqref{eq:backward_se30} to be
$\SO(3)$-invariant so that the associated measure on $\SE(3)^\numres$ given by
\Cref{sec:from-se3-so3} is $\SE(3)$-invariant. To do so we use the following result.

\begin{proposition}[$G$-invariance and SDEs]
\label{prop:invariance_lie}
Let $G$ be a Lie group and $H$ a subgroup of $G$.
If (a) $\dt{\bfX}{0}\sim p_0$ for an $H$ invariant distribution $p_0$ and 
(b) 
$\rmd \bfX^{(t)} = b(t, \bfX^{(t)}) \rmd t + \Sigma^{1/2} \rmd \bfB^{(t)}$
for bounded, $H$-equivariant coefficients $b$ and $\Sigma$ satisfying 
$b\circ L_h = \rmd L_h (b)$ and
 $\Sigma \rmd L_h(\cdot)  = \rmd L_h(\Sigma\cdot ),$ 
and where $\bfB^{(t)}$ is a Brownian motion associated
  with a left-invariant metric.
Then for every $t\ge 0$ 
\begin{enumerate}[label=(\alph*)]
\item{the distribution $p_t$ of $\dt{\bfX}{t}$ is $H$-invariant, and}
\item{its score $\nabla_{\dt{\bfX}{t}} \log p_t(\dt{\bfX}{t})$ is $H$-equivariant.}
\end{enumerate}
\end{proposition}

The proof can be extended to non-bounded coefficients under
appropriate assumption on the growth of $b$.
As a consequence of \Cref{prop:invariance_lie} we obtain the announced invariance.

\begin{corollary}\label{prop:invariance_se30}
Suppose $\{ \dt{\bfT}{0}\}_{t\ge 0}$ has $\SO(3)$ invariant initial distribution $p_0$ and evolves according to \Cref{eq:forward_se30}.
Then for every $t \in (0, \Tfinal),\  \nabla \log p_{\Tfinal-t}(\overleftarrow{\bfT}^{(t)})$ is $\SO(3)$ equivariant, 
and the distribution of $(\overleftarrow{\bfR}^{(t)}, \overleftarrow{\bfX}^{(t)})$ implied by \Cref{eq:backward_se30} is $\SO(3)$-invariant.
\end{corollary}

The significance \Cref{prop:invariance_se30} is two-fold.  
First, because the true score $\nabla \log p_{\Tfinal-t}(\overleftarrow{\bfT}^{(t)})$ is $\SO(3)$-equivariant, the corollary shows that incorporating an $\SO(3)$-equivariance constraint into neural network approximations of the score, $[s_\theta^r, s_\theta^x],$ does not limit the ability of the model to describe any $\SO(3)$ invariant target.
Second, it shows that any such approximation $\overleftarrow{\bfT}^{(t)}$ will be $\SO(3)$ invariant.  

Equation \eqref{prop:invariance_se30} is still true if
$[\nabla_{r} \log p_{t}, \nabla_{x} \log p_{t}]$ is replaced with
$[s_\theta^r, s_\theta^x]$ with $s_\theta^r$ and
$s_\theta^x$ $\SO(3)$-equivariant neural networks, see
\Cref{sec:framepred}.









\section{Protein backbone diffusion model}
\label{sec:diffusion_model}
\begin{figure*}[!ht]
\begin{center}
\centerline{\includegraphics[width=\textwidth]{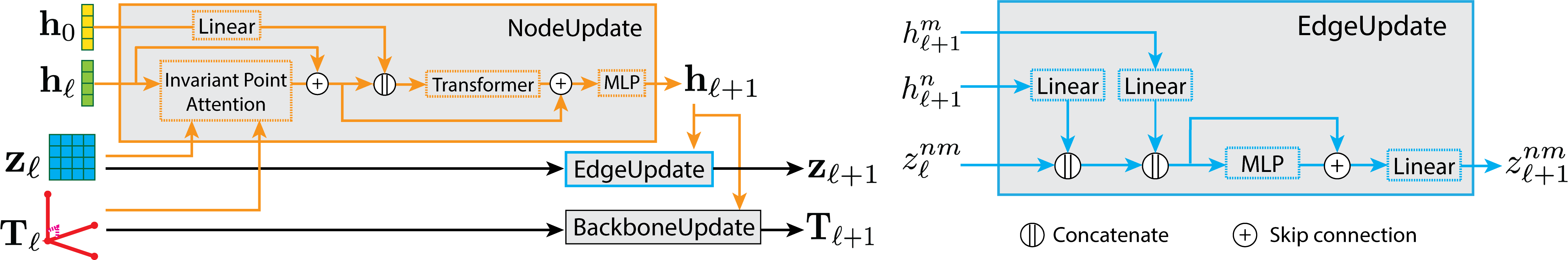}}
\caption{Single layer of \framediff. Each layer takes in the current node embedding $\bfh_\ell$, edge embedding $\bfz_\ell$, frames $\bfT_\ell$, and initial node embedding $\bfh_0$.
Rectangles indicate trainable neural networks. 
Node embeddings are first updated using IPA with a skip connection. 
Before Transformer, the initial node embeddings and post-IPA embeddings are concatenated.
After transformer, we include a skip connection with post-IPA embeddings.
The updated node embeddings $\mathbf{h}_{\ell+1}$ are then used to update edge embeddings $\mathbf{z}_{\ell+1}$ as well as predict frame updates $\bfT_{\ell+1}$.
See \Cref{sec:architecture} for in-depth architecture details.
}
\label{fig:framediff_block}
\end{center}
\vskip -0.3in
\end{figure*}

We now describe \framediff{}, a diffusion model for sampling protein backbones
by modeling frames based on the centered $\SE(3)^\numres$ stochastic process in
\Cref{sec:diffusion_process}.
In \cref{sec:framepred}, we describe our neural network to  learn the score using frame and torsion predictions. 
\cref{sec:loss} presents a multi-objective loss involving score matching and auxiliary protein structure losses.
Additional details for training and sampling are postponed to \cref{sec:training,sec:sampling}. 

\subsection{\framepred: score and torsion prediction}
\label{sec:framepred}

In this section, we provide an overview of our score and torsion prediction
network; technical details are given in \Cref{sec:architecture}.
Our neural network to learn the score is based on the structure module of AlphaFold2 (AF2)
\citep{Jumper2021HighlyAP}, which has previously be adopted for diffusion by \citet{anand2022protein}.
Namely, it performs iterative updates to the frames
over a series of $L$ layers using a combination of \emph{spatial} and
\emph{sequence} based attention modules.  Let
$\mathbf{h}_\ell = [h_\ell^1, \dots, h_\ell^\numres] \in
\mathbb{R}^{\numres\times D_h}$ be the node embeddings of the $\ell$-th layer
where $h_\ell^\res$ is the embedding for residue $\res$.
$\mathbf{z}_\ell \in \mathbb{R}^{\numres \times \numres \times D_z}$ are edge
embeddings with $z_\ell^{\res m}$ being the embedding of the edge between
residues $\res$ and $m$.  

\Cref{fig:framediff_block} shows one single layer of our neural network.
Spatial attention is performed with Invariant Point Attention (IPA) from AF2
which can attend to closer residues in coordinate space while a Transformer
\citep{vaswani2017transformer} allows for capturing interactions along the chain
structure.  We found including the Transformer greatly improved training and
sample quality.
As a result, the computational complexity of \framediff{} is quadratic in backbone length.
Unlike AF2, we do not use $\mathrm{StopGradient}$ between
rotation updates. The updates are $\SE(3)$-invariant since IPA is
$\SE(3)$-invariant.  We utilize fully connected graph structure where each
residue attends to every other node.  Updates to the node embeddings are
propagated to the edges in $\mathrm{EdgeUpdate}$ where a standard message
passing edge update is performed.  $\mathrm{BackboneUpdate}$ is taken from AF2
(Algorithm 23), where a linear layer is used to predict translation and rotation
updates to each frame.
Feature initialization follows \citet{trippe2022diffusion} where node embeddings are initialized with residue indices and timestep while edge embeddings additionally get relative sequence distances.
Edge embeddings are additionally initialized through self-conditioning \citep{chen2022analog} with a binned pairwise distance matrix between the model's $\carbon_\alpha$ predictions.
All coordinates are represented in nanometers.

Our model also outputs a prediction of the $\pmb{\psi}$ angle for each residue, which positions the backbone oxygen atom with respect to the predicted frame.
Putting it all together, our neural network with weights $\theta$ predicts the denoised frame and torsion angle,
\begin{align}
    (\dt{\mathbf{\pred{T}}}{0},  \pmb{\pred{\psi}}) = \mathrm{FramePred}(\mathbf{T}^{(t)}, t; \theta), \ \
    \dt{\mathbf{\pred{T}}}{0} = (\dt{\mathbf{\pred{R}}}{0}, \dt{\mathbf{\pred{X}}}{0}).
\end{align}

\textbf{Score parameterization.}
We relate the \framediff{} prediction to a score prediction via
$\nabla_{\dt{\bfT}{t}} \log p_{t|0}(\dt{\bfT}{t} \mid \dt{\mathbf{\pred{T}}}{0})
= \{ (s_\theta^\rmr(t, \dt{\bfT}{t})_n, s_\theta^\rmx(t,
\dt{\bfT}{t})_n)\}_{n=1}^N$
where the predicted score is computed separately for the rotation and translation of each residue,
$s_\theta^\rmr(t, \bfT^{(t)})_n = \nabla_{\dt{\bfR_n}{t}}  \log p_{t|0}(\dt{\bfR_n}{t} | \dt{\pred{\bfR}_{n}}{0})$ and $s_\theta^\rmx(t, \bfT^{(t)})_n= \nabla_{\dt{\bfX_n}{t}}\log p_{t|0}(\dt{\bfX_n}{t} | \dt{\pred{\bfX}_{n}}{0}).$

\subsection{Training losses}
\label{sec:loss}

Learning the translation and rotation score amounts to minimizing the \dsm{} loss given in \eqref{eq:denoising_sm}.
Following \citet{song2020score}, we choose the weighting schedule for the
rotation component as
$\lambda^{\mathrm{r}}_t\!=\! 1/\E[\|\nabla \log p_{t|0}(\dt{\bfR_n}{t} | \dt{\bfR}{0})\|^2_{\SO(3)}]$;
with this choice, the expected loss of the trivial prediction $\dt{\hat R}{0}=\dt{R}{t}$ 
is equal to $1$ for every $t.$

For translations, we use $\lambda_t^\rmx=(1-\rme^{-t})/\rme^{-t/2}$ so \eqref{eq:denoising_sm} simplifies as
\begin{align}
    \mathcal{L}_{\mathrm{dsm}}^\rmx 
    &= \textstyle{\tfrac{1}{N}\sum_{n=1}^N\|X^{(0)}_n - \pred{X}^{(0)}_{ n}\|^2} \label{eq:trans_score_loss}.
\end{align}
We find this choice is beneficial to avoid loss instabilities near low $t$ (see \citet{karras2022elucidating} for more discussion) where atomic accuracy is crucial for sample quality.
There is also the physical interpretation of directly predicting the $\calpha$ coordinates. 
Our $\SE(3)$ \dsm{} loss is $\mathcal{L}_{\mathrm{dsm}} = \mathcal{L}_{\mathrm{dsm}}^\rmr + \mathcal{L}_{\mathrm{dsm}}^\rmx.$

\textbf{Auxiliary losses.} 
In early experiments, we found that \framediff{} with $\mathcal{L}_{\mathrm{dsm}}$ generated backbones with plausible coarse-grained topologies, but unrealistic fine-grained characteristics, such as chain breaks or steric clashes. 
To discourage these physical violations, we use two additional losses to learn torsion angle $\psi$ and directly penalize atomic errors in the last steps of generation.
Let $\Omega = \{\nitrogen, \carbon, \calpha, \oxygen\}$ be the collection of backbone atoms.
The first loss is a direct MSE on the backbone (bb) positions,
\begin{align}
    \label{eq:bb_loss}
    \textstyle{\mathcal{L}_{\mathrm{bb}} = \tfrac{1}{4N}\sum_{n=1}^N \sum_{a \in \Omega} \|a_n^{(0)} - \pred{a}_{n}^{(0)}\|^2.}
\end{align}
Next, define $d_{ab}^{nm} = \|a^{(0)}_n - b^{(0)}_m\|$ as the true atomic distance between atoms $a,b\in \Omega$ for residue $n$ and $m$.
The predicted pairwise atomic distance is $\hat{d}_{ab}^{nm} = \|\hat{a}_n^{(0)} - \hat{b}_m^{(0)}\|$.
Similar in spirit to the distogram loss in AF2, the second loss is a local neighborhood loss on pairwise atomic distances,
\begin{align}
    &\textstyle{\mathcal{L}_{2\mathrm{D}} = \frac{1}{Z}\sum_{n,m=1}^N\sum_{a,b\in\Omega} \mathds{1}\{d_{ab}^{nm} < 0.6\} \|d_{ab}^{nm} - \hat{d}_{ab}^{nm}\|^2,}\\
    &\textstyle{Z = (\sum_{n,m=1}^N \sum_{a,b\in\Omega} \mathds{1}\{d_{ab}^{nm} < 0.6\}) - N.}
    \label{eq:bond_loss}
\end{align}
where $\mathds{1}\{d_{ab}^{nm} < 0.6\}$ is a indicator variable to only penalize atoms that within 0.6nm (i.e.\ $6$\AA).
We apply auxiliary losses only when $t$ is sampled near 0 ($t < \Tfinal/4$ in our experiments) during which the fine-grained characteristics emerge.
The full training loss can be written,
\begin{align}
    \label{eq:train_loss}
    \textstyle{
    \mathcal{L} = \mathcal{L}_{\text{dsm}} + w \cdot\mathds{1}\{t < \frac{\Tfinal}{4}\} (\mathcal{L}_{\mathrm{bb}} + \mathcal{L}_{2\mathrm{D}})},
\end{align}
where $w>0$ is a weight on these additional losses.
We find a including a high weight ($w=0.25$ in our experiments) leads to improved sample quality with fewer steric clashes and chain breaks.
Training follows standard diffusion training over the empirical data distribution $p_0$.
A full algorithm (\Cref{alg:train}) is provided in the appendix.

\begin{algorithm}
    \caption{\framediff{} sampling of protein backbones}
    \label{alg:sampling}
    \begin{algorithmic}[1]
        \REQUIRE $\theta, \numres, \Tfinal, N_{\mathrm{steps}}, \zeta, \epsilon$ 
        \STATE $\gamma=(1-\epsilon)/ N_\mathrm{steps}$
        \STATE \# Sample from invariant density
        \STATE $\dt{\bfT}{\Tfinal} \sim P_{\#}\piinv^{\SE(3)^N}$
        \FOR{$t=\Tfinal, \Tfinal-\gamma, \Tfinal-2\gamma,\dots, \epsilon$}
          \STATE $\dt{\hat{\bfT}}{0}, \_ = \mathrm{FramePred}(\dt{\bfT}{t}, t; \theta)$
          \STATE $\{(s_{\theta,n}^\rmr, s_{\theta,n}^\rmx)\}_{n=1}^N = \nabla_{\dt{\bfT}{t}} \log p_{t|0}(\dt{\bfT}{t} \mid \dt{\hat \bfT}{0})$
          \FOR{$(\dt{R_n}{t}, \dt{X_n}{t}) = \dt{T_1}{t}, \dots, \dt{T_N}{t} $}
            \STATE \# Translation tangent Gaussian
            \STATE $Z_n^{\rmx} \sim \mathcal{N}(0,  \Id_3)$
            \STATE $W_n^{\rmx} = P \gamma [\tfrac{1}{2} X_n^{(t)} + s^{\rmx}_{\theta,n}] + \zeta\sqrt{\gamma}  Z_n^{\rmx}$ 
            \STATE \# Remove center of mass
            \STATE $W_n^{\rmx} = P W_n^{\rmx}$
            \STATE \# Rotation tangent Gaussian
            \STATE $Z_n^{\rmr} \sim \mathcal{TN}_{\dt{R_n}{t}}(0, \Id)$
            \STATE \# Euler--Maruyama step on tangent space 
            \STATE $W_n^{\rmr} = \gamma \mathbf{s}^{\rmr}_{\theta,n} + \zeta \sqrt{\gamma}  Z_n^{\rmr}$ 
            \STATE $\dt{T_n}{t-\gamma} = \exp_{\dt{T_n}{t}} \left\{  (W_n^\rmr, W_n^\rmx) \right\}$
          \ENDFOR
        \ENDFOR
        \STATE \textbf{Return:} $\mathrm{FramePred}(\dt{\bfT}{\epsilon}, \epsilon; \theta)$
    \end{algorithmic}
\end{algorithm}
\vspace{-5pt}


\textbf{Centering of training examples.}
Each training example $\dt{\bfX}{t}$, is centered at zero in accordance with \Cref{eq:forward_se30}.
From a practical perspective, this centering leads to lower variance loss estimates than without centering. 
In particular, variability in the center of mass of $\dt{\bfX}{t}$ would lead to
corresponding variability in \framediff{}'s frame predictions as a result of the architecture's $\SE(3)$ equivariance.
By centering training examples, we eliminate this variability and thereby reduce the variance of  $\mathcal{L}_{\mathrm{dsm}}^\rmx$ and of gradient estimates.

\subsection{Sampling}
\Cref{alg:sampling} provides our sampling procedure.
Following \citet{debortoli2022Riemannian}, we use an
Euler--Maruyama discretization of \Cref{eq:backward_se30} with $\numsteps$ steps implemented as a geodesic random walk.
Each step involves samples $Z_n^\rmx$ and $Z_n^\rmr$ from Gaussian distributions defined in the tangent spaces of $\dt{X_n}{t}$ and $\dt{R_n}{t}$, respectively.
For translations, this is simply the usual Gaussian distribution on $\R^3,$ $Z_n^\rmx\sim \mathcal{N}(0,\Id_3).$
For rotations, we sample the coefficients of orthonormal basis vectors of the Lie algebra $\so(3)$ and rotate them into the tangent space to generate 
$Z_n^\rmr\sim\mathcal{TN}_{\dt{R_n}{t}}(0, \Id)$ as $Z_n^\rmr=\dt{R_n}{t}\sum_{i=1}^3\delta_i \mathbf{e}_i,$
where $\delta_i \overset{iid}{\sim} \mathcal{N}(0, 1)$ and $\mathbf{e}_1,\mathbf{e}_2,\mathbf{e}_3$ are orthonormal basis vectors (see \Cref{sec:double-covering-so3} for details).

Because we found that the backbones commonly destabilized in the final steps of sampling, we truncate sampling trajectories early, at a time $\epsilon>0.$ 
Following \citet{watson2022rfdiffusion}, 
we explore generating from the reverse process with
noise downscaled by a factor $\zeta\in \ccint{0,1}.$
For simplicity of exposition, we so far have assumed that the forward diffusion involves a Brownian motion without a diffusion coefficient;
in practice we set $\Tfinal=1$ and consider different diffusion coefficients for the rotation and translation (see \Cref{sec:diffusion_schedule}). 

\section{Experiments}
\label{sec:experiment}

\begin{figure*}[!ht]
\begin{center}
\centerline{\includegraphics[width=\textwidth]{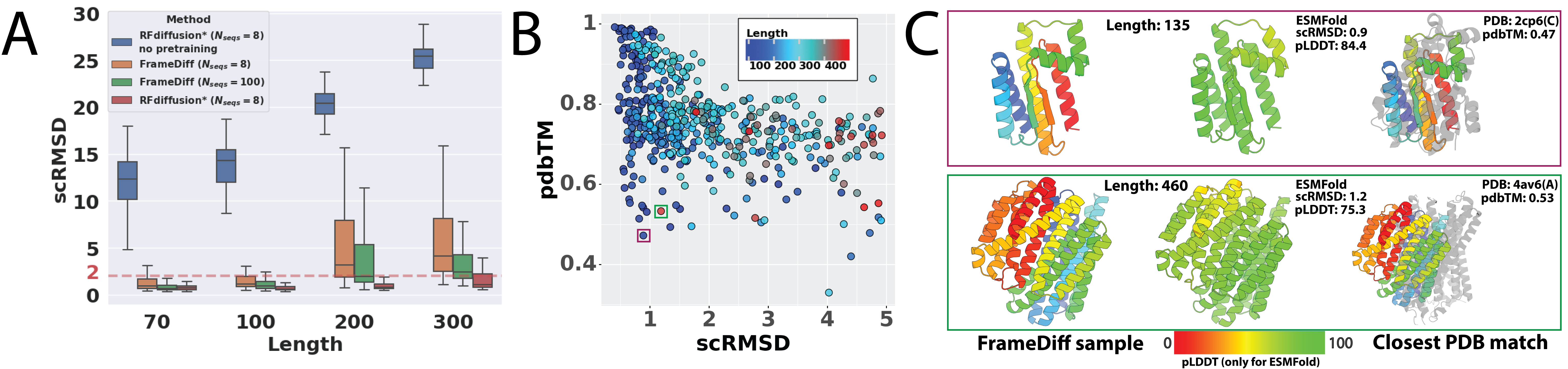}}
\vskip -0.15in
\caption{
Designability, diversity, and novelty of \framediff{} generated backbones with $\zeta=0.1$, $N_\mathrm{steps}=500$, $N_\mathrm{seq}=100$.
\textbf{(A)} \scrmsd{} based on 100 backbone samples of each length 70, 100, 200, 300 for $N_\mathrm{seq}=8,100$ plotted in the same manner as done in RFdiffusion.
\textbf{(B)} Scatter plot of Designability (\scrmsd) vs. novelty (\pdbtm) across lengths.
\textbf{(C)} Selected samples from panel (B) of novel and highly designable samples. Left: sampled backbones from \framediff.
Middle: best ESMFold predictions with high confidence (pLDDT)
Right: samples aligned with their closest PDB chain.
}
\label{fig:monomer_results}
\end{center}
\vskip -0.4in
\end{figure*}

We evaluate \framediff{} on monomer backbone generation.  We trained
\framediff{} with $L=4$ layers on a filtered set of 20312 backbones taken from
the Protein Data Bank (PDB) \citep{berman2000protein}.  Our model comprises 17.4
million parameters and was trained for one week on two A100 Nvidia GPUs.  See
\cref{sec:training} for data and training details.

We analyzed our samples in terms of designability (if a matching sequence can be
found), diversity, and novelty.  Comparison to prior protein backbone diffusion
models is challenging due to differences in training and evaluation among them.
We compared ourselves with published results from two promising protein
backbone diffusion models for protein design: Chroma
\citep{ingraham2022rfdiffusion} and RFdiffusion \citep{watson2022rfdiffusion}.
We include comparison in \Cref{sec:foldingdiff} to FoldingDiff \citep{wu2022protein} which has publicly available code.
We refer to \Cref{sec:related_work} for details on these and other diffusion methods.

\subsection{Monomeric protein generation and evaluation}
\label{sec:monomer_sampling}

We assess \framediff's performance in unconditional generation of
\emph{monomeric} protein backbones.
In this section, we detail our inference and evaluation procedure.

\textbf{Designability.} A generated backbone is meaningful only if there exists an amino acid sequence which folds to that structure.
We follow \citet{trippe2022diffusion} and assess backbone designability with \emph{self-consistency} evaluation:
a fixed-backbone sequence design algorithm proposes sequences,
these sequences are input to a structure prediction algorithm,
and self-consistency is assessed as the best agreement between the sampled and predicted backbones (see \cref{fig:designability}).
In this work, we use \proteinmpnn{} at temperature 0.1 to generate $\numseqs$ sequences for
\esmfold{} \citep{lin2022evolutionary} to predict structures.
We quantify self-consistency through both TM-score (\sctm, higher is better)
and
$\calpha$-RMSD (\scrmsd, lower is better).
Chroma reports using $\sctm > 0.5$ as the designable criterion.
However, it was shown \scrmsd{}$<2$\AA{} provides a more stringent filter,
particularly for long (e.g.\ 600 amino acid) backbones on which $0.75$ \sctm{} can be attained for very structurally different backbones \citep{watson2022rfdiffusion}.

\textbf{Diversity.}
We quantify the diversity of backbones sampled by \framediff{} 
 through the number of distinct structural clusters.
In particular, for a collection of backbone samples we use MaxCluster \citep{herbertmaxcluster}
to hierarchically cluster backbones with a 0.5 TM-score threshold.
 We report diversity as the proportion of unique clusters: (number of clusters) /
(number of samples).

\textbf{Novelty.} 
We assess the ability of \framediff{} to generalize beyond the training set and produce novel backbones by comparing the similarity to known structures in the PDB.
We use FoldSeek \citep{van2022foldseek} to search for similar structures and report the highest TM-scores of samples to any chain in PDB, which we refer to as \pdbtm.


\subsection{Results}
\label{sec:unconditional_results}

We analyze \framediff{} monomer samples on designability, diversity, and novelty.
On designability, we briefly compare \framediff's samples with  backbone generation diffusion models Chroma and RFdiffusion.
However, we note that the training and evaluation set-ups are significantly different across \framediff, Chroma, and RFdiffusion. 

Using \sctm{}$>0.5$ as the designable criterion, Chroma reported designability of 55\% with 100 designed sequences ($\numseqs=100$).
Lengths are between 100 and 500 and sampled proportionally ``1/length''.
However, this heavily biases performance towards shorter lengths and leads to additional length variability across evaluations.
Instead, we sample 10 backbones at every length [100, 105, \dots, 495, 500] in intervals of 5 (810 total samples) such that lengths are fixed and distributed uniformly.

\cref{table:designability} reports \framediff{} metrics as we vary different sampling parameters.
We notice a stark improvement in designability by changing the noise scale $\noisescale=0.5$ at the cost of lower diversity.
Increasing $\numseqs$ also improves designability but at a significant compute cost.
The reported results use $\numsteps=500$; however decreasing to $\numsteps=100$ with a low noise scale still resulted in designable backbones. 
With $\numsteps=100$, generation of a 100 amino acid backbone takes 4.4 seconds on an A100 GPU;
compared to RFdiffusion, this is more than an order of magnitude speed-up.\footnote{\citet{watson2022rfdiffusion} report 150 seconds (34-fold slower) for 100 amino acid backbones on an A4000 GPU.}
\begin{table}[!ht]
\vskip -0.2in
\caption{\framediff{} sample metrics.}
\vskip 3pt
\begin{center}
\label{table:designability}
\begin{small}
\begin{sc}
\begin{tabular}{lccccc}
\toprule
Noise scale $\noisescale$ &  1.0 & 0.5 & 0.1 & 0.1 & 0.1 \\
$N_{\text{steps}}$ &  500 & 500 & 500 & 500 & 100  \\
$N_{\text{seq}}$ & 8 & 8 & 8 & 100 & 8  \\
\midrule
$>0.5$ \sctm{} ($\uparrow$) & 49\% & 74\% & 75\% & 84\% & 74\%  \\
 $<2$\AA{} \scrmsd{} ($\uparrow$) & 11\% & 23\% & 28\% & 40\% & 24\% \\
Diversity ($\uparrow$) & 0.75 & 0.56 & 0.53 & 0.54  & 0.55  \\
\bottomrule
\end{tabular}
\end{sc}
\end{small}
\end{center}
\vskip -0.15in
\end{table}
Using $\zeta=1.0$, $\numsteps=500$, $\numseqs=8$, we perform ablations on self-conditioning, auxiliary losses, and form of the $\SO(3)$ loss -- either the DSM form developed in our work or the squared Frobenius norm loss ($\mathcal{L}_F$, equal to $\|
\dt{\pred{R}}{0}- \dt{R}{0}\|_F^2$) used in prior works \citep{watson2022rfdiffusion,luo2022antigen}.
Our results are in \Cref{table:ablations} where we see the best model incorporates all components.
We leave hyperparameter searches to future work.

\begin{table}[!ht]
\vskip -0.2in
\caption{\framediff{} ablations.}
\begin{center}
\label{table:ablations}
\vskip 5pt
\begin{sc}
\begin{tabular}{c|ccccc}
\toprule
$>0.5$\sctm{} ($\uparrow$) & \makecell{Self \\ cond.} & $\mathcal{L}_{\mathrm{2D}}$ & $\mathcal{L}_{\mathrm{bb}}$ & $\mathcal{L}_{\mathrm{dsm}}$ & $\mathcal{L}_F$\\
\midrule
 49\% & \checkmark & \checkmark & \checkmark & \checkmark & \\
 39\% & \checkmark & \checkmark & \checkmark &  & \checkmark \\
 42\% &  & \checkmark & \checkmark & \checkmark & \\
 22\% &  & & \checkmark & \checkmark & \\
 16\% &  & & & \checkmark & \\
 0\% & \checkmark & \checkmark & \checkmark &  & \\
\bottomrule
\end{tabular}
\end{sc}
\end{center}
\vskip -0.1in
\end{table}

In \cref{fig:monomer_results}A, we evaluate \scrmsd{} across four lengths. 
\framediff{} is able to generate designable samples without pretraining; by contrast, RFdiffusion demonstrated the capacity to generate designable sequences only when initialized with pre-trained weights.
More training data (i.e. training on complexes) and neural network parameters could help close the gap to RFdiffusion's reported performance. 
Finally, RFdiffusion uses an all-to-all pairwise TM-align to measure diversity of its samples with clustering at 0.6 TM-score threshold.
We perform an equivalent diversity evaluation using maxcluster with 0.6 TM-score threshold in \cref{table:rfdiffusion_diversity} where we find a high degree of diversity ($>$0.5) that is comparable with RFdiffusion.
\cref{sec:additional_results} shows more results and visualizations.

We next investigated the similarity of each sample to known structures in PDB.
In \cref{fig:monomer_results}B, we plot the novelty (\pdbtm) as a function of designability (\scrmsd).
As expected, designability decreases with longer lengths.
Samples with low \scrmsd{} tend to have high similarity with the PDB.
Our interest is in the lower left hand quadrant where \scrmsd{} $< 2.0$ and \pdbtm{} $< 0.6$.
\cref{fig:monomer_results}C illustrates two examples of \framediff{} samples that are designable and novel.
We additionally find ESMFold to be highly confident, predicted LDDT (pLDDT) $> 0.7$, for these samples.

Our experiments indicate \framediff{} is capable of learning complex distributions over protein monomer backbone that are designable, diverse, and in some cases novel compared to known protein structures.
When used with decreased noise-scale, $75\%$ of samples across a range of lengths were designable by \sctm{}${>}0.5;$
by contrast, all prior works reporting this metric not involving pretrained networks (see \cref{sec:related_work}) have reported below 55\% designability.
However, due to differences in training and evaluation across these methods and ours, we refrain making state-of-the-art claims.


\section{Related work}
\label{sec:related_work}
\textbf{Diffusion models on proteins.}
Past works have developed diffusion models over different
representations of protein structures without pretraining \citep{wu2022protein,trippe2022diffusion,anand2022protein,qiao2022dynamic}.
Out of these methods, Chroma \citep{ingraham2022rfdiffusion} reported the highest designability metric by diffusing over backbone atoms with a non-isotropic diffusion based on statistically determined covariance constraints.
Compared to these works, we develop a principled $\SE(3)$ diffusion framework
over protein backbones that demonstrates improved sample quality over methods
that do not use $\SE(3)$ diffusion.
Most similar to our work is RFdiffusion \citep{watson2022rfdiffusion} which
formulated the same forward diffusion process over $\SE(3)^\numres$, but with
squared Frobenius norm rotation loss and reverse step that deviates from theory.
We discuss the nuance between the rotation losses in \Cref{sec:loss_comparison}.
While not outperforming RFdiffusion, FrameDiff enjoys several benefits such as being principled, having 1/4 the number of neural network weights, and not requiring expensive pretraining on protein structure prediction.

\textbf{Diffusion models on manifolds.}
A general framework for continuous diffusion models on manifolds was first introduced in
\citet{debortoli2022Riemannian} extending the work of \citet{song2020score} to
Riemannian manifolds. 
Concurrently, \citet{huang2022Riemannian} introduced a similar framework extending the maximum
likelihood approach of \citet{huang2021variational}. 
Some manifolds have been considered in the setting of diffusion models for specific applications.
In particular, \citet{jing2022torsional} consider the product of tori for molecular conformer
generation, \citet{corso2022diffdock} on the product
space $\rset^3 \times \SO(3) \times \SO(2)^m$ for protein docking applications
and \citet{leach2022denoising} on $\SO(3)$ for rotational alignment.
Finally, we highlight the work of
\citet{urain2022se} who introduce $\SE(3)$-diffusion models for robotics
applications. One major theoretical and methodological difference with the
present work is that we develop a principled diffusion
model on this Lie group ensuring that at optimality we recover the exact
backward process.


\section{Discussion}
\label{sec:discussion}
Protein backbone generation is a fundamental task in \emph{de novo} protein design.
Motivated by the success of rigid-body frame representation of proteins, we developed an $\SE(3)$-invariant diffusion models on $\SE(3)^N$ for protein modelling.
We laid the theoretical foundations of this method, and introduced \framediff, a instance of this framework, equipped with an $\SE(3)$-equivariant score network which needs not to be pretrained.
We empirically demonstrated \framediff's ability to generate designable and diverse samples. 
Even with stringent filters, we find our samples can generalize beyond PDB,
although we note that claims of generating novel proteins requires experimental characterization.
Our results are competitive with those reported in Chroma and RFdiffusion.
However, differences in training and evaluation confound rigorous comparisons between the methods.

One important research direction is to extend \framediff{} to conditional generative modeling tasks, such as probabilistic sequence-to-structure prediction which to capture functional motion \citep{lane2023Protein} and probabilistic scaffolding design given a functional motif \citep{trippe2022diffusion} 
We hypothesize scaling \framediff{} to train on larger data and improving the optimization would deliver backbones with designability on par with RFdiffusion while maintaining \framediff's simplicity.
Finally, we highlight that the key aspects of our theoretical contributions---
the general form of Brownian motions that is amenable to \dsm{} along with sub-group invariance---are applicable to general Lie groups.
Of particular interest are $\SO(3)$ in robotics \citep{barfoot2011Pose} and $\SU(2)$ in Lattice QCD \citep{albergo2021introduction}.


\section*{Acknowledgements}
The authors thank Hannes St\"{a}rk, Gabriele Corso, Bowen Jing, David Juergens, Joseph Watson, Nathaniel Bennett, Luhuan Wu and David Baker for helpful discussions.

EM is supported by an EPSRC Prosperity Partnership EP/T005386/1 between Microsoft Research and the University of Cambridge.
JY is supported in part by an NSF-GRFP.
JY, RB, and TJ acknowledge support from NSF Expeditions grant (award 1918839: Collaborative Research: Understanding the World Through Code), Machine Learning for Pharmaceutical Discovery and Synthesis (MLPDS) consortium, the Abdul Latif Jameel
Clinic for Machine Learning in Health, the DTRA Discovery of Medical Countermeasures Against
New and Emerging (DOMANE) threats program, the DARPA Accelerated Molecular Discovery
program and the Sanofi Computational Antibody Design grant. AD acknowledges support from EPSRC grants EP/R034710/1 and EP/R018561/1.




\bibliography{content/references}
\bibliographystyle{icml2023}

\newpage

\onecolumn
\appendix
\appendixhead

\section{Organization of the supplementary}
\label{sec:organ-suppl}

In this supplementary, we first recall in \cref{sec:basics-repr-theory} some
important concepts on Lie groups and representation theory which are useful for
what follows.  In \cref{sec:representations-so3} we derive the irreducible
representations of $\SU(2)$ and then of
$\SO(3)$. 
Using these, we introduce in \cref{sec:metric-laplacian-se3} the canonical
(bi-invariant) metric on $\SO(3)$, and a left-invariant metric on $\SE(3)$ which
induces a Laplacian that factorises over $\SO(3)$ and $\R^3$. In particular, we
prove \Cref{prop:brownian}.  In \cref{sec:heat_kernel}, we compute the heat
kernel on compact Lie groups and in particular on $\SO(3)$, therefore proving
\Cref{prop:lie_group_transition} and \Cref{prop:brownian_on_SO3}.  In
\cref{sec:invariant_process}, we show that equivariant drift and diffusion
coefficients induces invariant processes and prove \Cref{prop:invariance_lie}.
In \cref{sec:conn-betw-so_3rs}, we show the equivalence $\SE(3)$-invariant
measures and $\SO(3)$-invariant measures with pinned center of mass, proving
\Cref{sec:from-se3-so3}.  Details about score computations on $\SO(3)$ using
Rodrigues' formula are given in \Cref{sec:rodrigues-formula}, including the
proof of \Cref{prop:deno-score-match}.
In \cref{sec:additional_methods}, we
include additional method details.  In \cref{sec:additional_experiments}, we
present additional experiment details.


\section{Lie group and representation theory toolbox}
\label{sec:basics-repr-theory}
In this section, we introduce some useful tools for the study of the heat kernel
on Lie groups using representation theory. We refer to
\cite{faraut2008analysis,hall2015lie,harris1991representation,knapp1996lie,folland2016course}
for more details on Lie groups and representation theory.

\subsection{Group representation}
\label{sec:group-lie-algebra}

Let $G$ be a group. A group representation $(\rho, V)$ is given by a vector
space $V$\footnote{We focus on real vector spaces in this presentation.} and a
homomorphism $\rho: \ G \to \mathrm{GL}(V)$. A representation $(\rho,V)$ is said
to be irreducible if for any subspace $W \subset V$ which is invariant by
$\rho$, i.e. $\rho(G)(W) \subset W$, then $W = \{0\}$ or $W = V$. The study of
irreducible group representations is at the heart of the analysis on groups.  In
particular, it is remarkable that if $G$ is compact every unitary representation
can be decomposed as a direct sum of irreducible finite dimensional unitary 
representations of $G$. This result is known as the Peter--Weyl theorem
\cite{weylh1927Vollstandigkeit}.

\subsection{Lie group and Lie algebra}
\label{sec:lie-groups}

We recall that a Lie group is a group which is also a differentiable manifold
for which the multiplication and inversion maps are smooth. Homomorphism of Lie
groups are homomorphisms of groups with an additional smoothness assumption.
The Lie algebra of a Lie group $G$ is defined as the tangent space of the Lie
group at the identity element $\mathbf{e}$ and is denoted $\mathfrak{g}$. A
vector field $X \in \mathfrak{X}(G)$ acts on a smooth function
$f \in \rmc^\infty(G)$ as $X(f) = \sum_{i=1}^d X_i \partial_i f$. Note that
$X(f) \in \rmc^\infty(G)$.  Given two vector fields $X,Y \in \mathfrak{X}(G)$,
the bracket between $X$ and $Y$ is given by $[X,Y] \in \mathfrak{X}(G)$ such
that for any $f \in \rmc^\infty(G)$, $[X,Y](f) = X(Y(f)) - Y(X(f))$. Note that
for any $X_0 \in \mathfrak{g}$ there exists $X \in \mathfrak{X}(G)$ such that
$X(\mathbf{e}) = X_0$. Hence, we define the Lie bracket between
$X_0, Y_0 \in \mathfrak{g}$ as $[X_0,Y_0] = [X,Y]$. Note that if
$G \subset \mathrm{GL}_n(\cset)$ then we have that for any
$X,Y \in \mathfrak{g}$, $[X,Y] = XY - YX$, where $XY$ is the matrix product
between $X$ and $Y$. 

For an arbitrary Lie group, the exponential mapping is defined as
$\exp: \ \mathfrak{g} \to G$ such that for any $X \in \mathfrak{g}$,
$\exp[X] = \gamma(1)$ where $\gamma: \ \rset \to G$ is an homomorphism such that
$\gamma'(0) = X$. Another useful exponential map in the space of matrices is the
exponential of matrix, given by $\exp[X] = \sum_{k \in \nset} X^k / k!$.  Note
that if the metric is bi-invariant (invariant w.r.t. the left and right actions)
then the associated exponential map coincides with the exponential of matrix,
see \citep[Chapter 3, Exercise 3]{1992Riemannian}.  If $G$ is compact then
given any left-invariant metric $\langle \cdot, \cdot \rangle_G$ we can consider
$\langle \cdot, \cdot \rangle_{\bar{G}}$ given for any $X, Y \in \mathfrak{G}$ by
\begin{equation}
  \textstyle{\langle X, Y \rangle_{\bar{G}} = \int_G \langle \rmd R_g X, \rmd R_g Y \rangle \rmd \mu(g) ,}
\end{equation}
where $\mu$ is the left-invariant Haar measure on $G$.  Then
$\langle \cdot, \cdot \rangle_{\bar{G}}$ is bi-invariant.  If $G$ is compact and
connected then $\exp$ is surjective, see
\citet[Exercises 2.9, 2.10]{hall2015lie}.

One of the most important aspect of Lie groups is that (at least in the
connected setting), they can be described entirely by their Lie algebra. More
precisely, for any homomorphism $\Phi: \ G \to H$, denoting
$\phi = \rmd \Phi(\mathbf{e}): \ \mathfrak{g} \to \mathfrak{h}$, we have
$\Phi \circ \exp = \exp \circ \phi$,
see \citet[p.105]{harris1991representation}.

\subsection{Lie algebra representations}
\label{sec:lie-algebra-repr}

A Lie algebra homomorphism $\phi: \ \mathfrak{g} \to \mathfrak{g}'$ between two
Lie algebras $\mathfrak{g}$ and $\mathfrak{g}'$ is defined as a linear map which
preserves Lie brackets, i.e. for any $X,Y \in \mathfrak{g}$,
$\phi([X,Y]) = [\phi(X),\phi(Y)]$. A Lie algebra representation of
$\mathfrak{g}$ is given by $(\rho, V)$ such that
$\rho: \ \mathfrak{g} \to \gl(V)$ is a Lie algebra homomorphism,
i.e. for any $X,Y \in \mathfrak{g}$,
$\rho([X,Y]) = \rho(X)\rho(Y) - \rho(Y)\rho(X)$. One way to construct Lie
algebra representations is through Lie group representations. Indeed, one can
verify that the differential at $\mathbf{e}$ of any Lie group representation is
a Lie algebra representation.

One important Lie algebra representation is given by the adjoint representation.
First, define $\Phi: \ G \to \Aut(G)$ such that for any $g, h \in G$,
$\Phi(g)(h) = g h g^{-1}$. Then, for any $g \in G$, denote
$\Ad(g) = \rmd \Phi(g)(\mathbf{e})$. Note that for any $g \in G$, we have that
$\Ad(g) \in \GL(\mathfrak{g})$. $\Ad: \ G \to \GL(\mathfrak{g})$ is a Lie
homomorphism and therefore a representation of $G$. Differentiating the adjoint
\emph{Lie group} representation we obtain a \emph{Lie algebra} representation
$\ad: \ \mathfrak{g} \to \gl(\mathfrak{g})$. It can be shown that for any
$X, Y \in \mathfrak{g}$, $\ad(X)(Y) = [X,Y]$. Note that we have
$\Ad \circ \exp = \exp \circ \ad$ \citep[Chapter 2, Proposition 2.24, Exercise
19]{hall2015lie}. Note that this equivalence between homomorphism defined on the
\emph{group} level and homomorphisms defined on the \emph{Lie algebra} level can
be extended in the simply connected setting, see \citep[Theorem
3.7]{hall2015lie}.

\section{Representations and characters of $\SO(3)$}
\label{sec:representations-so3}

In order to study the irreducible (unitary) representations of $\SO(3)$, we
first focus on the irreducible (unitary) representations of $\SU(2)$ in
\Cref{sec:representations-su2}. We describe the double-covering of $\SO(3)$ by
$\SU(2)$, which relates these two groups, in \Cref{sec:double-covering-so3}. We
discuss different $\SO(3)$ parameterizations in \Cref{sec:param-so3}. Finally,
we give the $\SO(3)$ irreducible unitary representations in
\Cref{sec:representations-so3-1}.

\subsection{Representations and characters of $\SU(2)$}
\label{sec:representations-su2}

In this section, we follow the presentation of 
\citet{faraut2008analysis} and provide a construction of the irreducible unitary
representations of $\SU(2)$ for completeness.  We refer to \citet{faraut2008analysis,hall2015lie} for an extensive study of this group. For every $m \in \nset$, with
$n \geq 1$ we consider the representation $(\pi_m, V_m)$ where $V_m$ is the
space of homogeneous polynomials of degree $m$ with two variables $X$, $Y$ and
complex coefficients. For any $P \in V_m$ and $g \in \SL(2, \cset)$, we define
$\pi_m(g)(P)(X,Y) = P(g(X,Y))$. For example, we have 
\begin{equation}
    g = \left( \begin{matrix}
    0 & -1 \\
    1 & 0
  \end{matrix} \right) \eqsp , \qquad \pi_4(g) (X^3Y - X^2Y^2) = -XY^3 - X^2Y^2 .
\end{equation}
We denote $\rho_m$ the differentiated
representation arising from $\pi_m: \ \sllie(2, \cset) \to \gl(V_m)$. A basis of
$\sllie(2,\cset)$ (as a complex Lie algebra) is given by
\begin{equation}
  H = \left( \begin{matrix}
    1 & 0 \\
    0 & -1
  \end{matrix} \right) \eqsp , \qquad
  E = \left(\begin{matrix}
    0 & 1 \\
    0 & 0
  \end{matrix}\right) \eqsp , \qquad 
  F = \left(\begin{matrix}
    0 & 0 \\
    1 & 0
  \end{matrix}\right) \eqsp .   
\end{equation}
We have the following Lie brackets
\begin{equation}
  \label{eq:lie_bracket_relationship}
  [H,E] = 2E \eqsp , \qquad [H,F] = -2F \eqsp , \qquad [E,F] = H \eqsp . 
\end{equation}
A basis of $V_m$ is given by $\{P_j\}_{j=0}^m$ with $P_j = X^j Y^{m-j}$ for
$j \in \{0, \dots, m\}$.  Using \citep[Theorem 3.7]{hall2015lie}, we have that
$\rho_m(M) = (\exp[\pi_m(tM)])'_{t=0}$, for $M \in \{E, F, H\}$ and therefore 
\begin{equation}
  \rho_m(H)(P) = X \partial_X P - Y \partial_Y P \eqsp , \qquad \rho_m(E)(P) = X \partial_Y P \eqsp , \qquad \rho_m(F)(P) =  Y \partial_X P \eqsp .
\end{equation}

\begin{proposition}{}{}
  For any $m \in \nset$ with $m \geq 1$, $\rho_m$ is an irreducible Lie algebra representation.
\end{proposition}

\begin{proof}
  For any $j \in \{0, \dots, m\}$ we have
  \begin{equation}
    \rho_m(H)(P_j) = (2j - m) P_j \eqsp , \qquad \rho_m(E)(P_j)  = (m-j) P_{j+1} \eqsp , \qquad \rho_m(F)(P_j) = j P_{j-1} \eqsp ,
  \end{equation}
  with $P_{-1} = P_{m+1} = 0$.  Now let $W \neq \{0\}$ be an invariant subspace
  of $V_m$ for $\rho_m$. We have that $\rho_m(H)$ restricted to $W$ admits an
  eigenvector and therefore, there exists $j_0 \in \{0, \dots, m\}$ such that
  $P_{j_0} \in W$. Indeed, let $P = \sum_{i=0}^m \alpha_i P_i$, with
  $(\alpha_i)_{i=0}^m \in \cset^{m+1}$, be such an eigenvector with eigenvalue
  $\lambda \in \cset$.  We have
  \begin{equation}
    \textstyle{ \sum_{i=0}^m (2i-m) \alpha_i P_i = \rho_m(H)(P) = \sum_{i=0}^m \lambda \alpha_i P_i. }
  \end{equation}
  Hence, $\sum_{i=0}^m (2i-m-\lambda)\alpha_i P_i =0$. This means that for any
  $i \in \{0, \dots, m\}$ except for one $i_0 \in \{0, \dots, m\}$,
  $\alpha_i = 0$. Hence, $P_{i_0} \in W$.  Upon applying $\rho_m(E)$ and
  $\rho_m(F)$ repeatedly we find that for any $j \in \{0, \dots, m\}$,
  $P_j \in W$ and therefore $W = V_m$, which concludes the proof.
\end{proof}

\begin{proposition}{}{}
  Let $(\rho, V)$ an irreducible Lie algebra representation, then there exist
  $m \in \nset$ with $m \geq 1$ and $A \in \GL(V, V_m)$ such that for any
  $\rho = A^{-1} \rho_m A$.
\end{proposition}

\begin{proof}
  Let $v$ be the eigenvector of $\rho(H)$ associated with eigenvalue $\lambda$
  with smallest real part. Using \eqref{eq:lie_bracket_relationship}, we have that
  \begin{equation}
    \rho(H)\rho(E)v = \rho(E)\rho(H)v + \rho([H,E])v = (\lambda+2) \rho(E)v \eqsp . 
  \end{equation}
  Similarly, we have
  \begin{equation}
    \label{eq:H_to_F}
    \rho(H)\rho(F)v = \rho(F)\rho(H)v + \rho([H,F])v = (\lambda-2) \rho(F)v \eqsp . 
  \end{equation}
  For any $k \in \nset$, denote $v_k = \rho(E)^k v$. Denote $m \in \nset$ such that for
  any $k > m$, $v_k=0$ and $v_m \neq 0$. We have that $\{v_0, \dots, v_m\}$ are
  linearly independent, since for each $k \in \{0, \dots, m\}$ we have that
  $v_k$ is an eigenvector of $\rho(H)$ with eigenvalue $\lambda + 2k$. Denote
  $W$ the subspace spanned by $\{v_0, \dots, v_m\}$. $\rho(H)(W) \subset W$ and
  $\rho(E)(W) \subset W$. Let us show that $\rho(F)(W) \subset W$. Assume that
  $\rho(F)v_k = \alpha_k v_{k-1}$ for some $k \in \{1, \dots, m-1\}$, then we have
  \begin{equation}
    \rho(F)(v_{k+1}) = \rho(F)\rho(E)v_k = \rho(E)\rho(F)v_k+ \rho([F,E])v_k = (\alpha_k - (\lambda + 2k)) v_{k}  \eqsp . 
  \end{equation}
  Hence, setting $\alpha_{k+1} = \alpha_k - (\lambda +2k)$, we get that that
  $\rho(F)(v_{k+1}) = \alpha_{k+1} v_k$. Let us show that
  $\rho(F)(v_1) = \alpha_1 v_0 = -\lambda v_0$. First, using \eqref{eq:H_to_F}
  we have that if $\rho(F)(v_0) \neq 0$ then $\rho(F)(v_0)$ is an eigenvector of
  $\rho(H)$ for the eigenvalue $\lambda -2$, which is absurd since $\lambda$ has
  minimal real part. Hence $\rho(F)(v_0) = 0$ and we have
  \begin{equation}
    \rho(F)(v_{1}) = \rho(E)\rho(F)v_0+ \rho([F,E])v_0 = -\lambda v_{0}  \eqsp . 
  \end{equation}
  Therefore, we have that $\rho(F)(W) \subset W$ and therefore $V = W$ since
  $\rho$ is irreducible. In addition, by recursion, we have that for any
  $k \in \{0, \dots, m\}$, $\alpha_k = -k(\lambda + k -1)$.  A basis of $V$ is
  given by $\{v_j\}_{j=0}^m$ and we have that for any $j \in \{0, \dots, m\}$
  \begin{equation}
    \rho(H)(v_j) = (\lambda + 2j) v_j \eqsp , \qquad \rho(E)(v_j) = v_{j+1} \eqsp , \qquad \rho(F)(v_j) = -j(\lambda +j - 1) v_{j-1} \eqsp , 
  \end{equation}
  with $v_{-1} = v_{m+1} = 0$.
  We have that
  \begin{equation}
    \textstyle{\trace(\rho(H)) = 0 = \sum_{j=0}^m \lambda + 2j = (m+1)(\lambda+m)} \eqsp .
  \end{equation}
  Hence $\lambda = -m$ and we have
  \begin{equation}
    \rho(H)(v_j) = (-m + 2j) v_j \eqsp , \qquad \rho(E)(v_j) = v_{j+1} \eqsp , \qquad \rho(F)(v_j) = j(m -j +1) v_{j-1} \eqsp , 
  \end{equation}
  Hence, letting $w_j = \lambda_j v_j$ with $\lambda_{j}/\lambda_{j+1} = m -j$ we have 
  \begin{equation}
    \rho(H)(w_j) = (-m + 2j) w_j \eqsp , \qquad \rho(E)(w_j) = (m-j) w_{j+1} \eqsp , \qquad \rho(F)(w_j) = j w_{j-1} \eqsp .
  \end{equation}
  We conclude upon defining $A w_j = P_j$ for any $j \in \{0, \dots, m\}$.
\end{proof}

\begin{proposition}
  \label{prop:irred_su2}
  Let $(\pi,V)$ be a irreducible representation of $\SU(2)$ then there exist
  $m \in \nset$ and $A \in \GL(V, V_m)$ such that $\pi = A^{-1} \pi_m A$.
\end{proposition}

\begin{proof}
  Let $\rho: \ \su(2) \to \gl(V)$ the Lie algebra representation associated with
  $\pi$. $\rho$ can be linearly extended to a Lie algebra representation of
  $\sllie(2,\cset)$ using that $\sllie(2,\cset) = \su(2) \oplus \su(2)$ (indeed
  each element $Z$ of $\sllie(2,\cset)$ can be written uniquely as
  $Z = X + \imag Y$ with $X,Y \in \su(2)$). The extension of $\rho$ is given by
  $\rho_{\mathrm{ext}}(Z) = \rho(X) + \imag \rho(Y)$. Let $W$ be an invariant
  subspace for $\rho_{\mathrm{ext}}$ then it is an invariant subspace for $\rho$
  and therefore for any $X \in \su(2)$, $\exp[\rho(X)](W) \subset W$. Using that
  $\SU(2)$ is connected we have that for any $U \in \SU(2)$ there exists
  $X \in \su(2)$ such that $U = \exp[X]$ and using that
  $\pi \circ \exp = \exp \circ \rho$, \citep[Theorem 3.7]{hall2015lie}, we get
  that $\pi(\SU(2))(W) \subset W$ and therefore $W = V$. Hence,
  $\rho_{\mathrm{ext}}$ is irreducible and there exist $m \in \nset$ and
  $A \in \GL(V, V_m)$ such that $\rho = A^{-1} \rho_m A$. We conclude by
  exponentiation, \citep[Theorem 3.7]{hall2015lie}.
\end{proof}


\subsection{Double-covering of $\SO(3)$}
\label{sec:double-covering-so3}

In order to derive the (unitary) irreducible representations of $\SO(3)$ we
first link $\SO(3)$ with $\SU(2)$ using the adjoint representation. First, let
us consider a basis of $\su(2)$, $(X_1, X_2, X_3)$ given by
\begin{equation}
  X_1 = \left(
    \begin{matrix}
      0 & \imag \\
      \imag & 0 
    \end{matrix}
  \right) \eqsp , \qquad
X_2 = \left(
    \begin{matrix}
      0 & -1 \\
      1 & 0 
    \end{matrix}
  \right) \eqsp , \qquad 
X_3 = \left(
    \begin{matrix}
      \imag & 0 \\
      0 & -\imag 
    \end{matrix}
  \right) \eqsp .
\end{equation}
A basis of
$\so(3)$ is given by
\begin{equation}
    Y_1 = \left(
      \begin{matrix}
        0 & 0 & 0 \\
        0 & 0 & -1 \\
        0 & 1 & 0 
    \end{matrix}
  \right) \eqsp , \qquad
    Y_2 = \left(
      \begin{matrix}
        0 & 0 & 1 \\
        0 & 0 & 0 \\
        -1 & 0 & 0 
    \end{matrix}
  \right) \eqsp , \qquad
      Y_3 = \left(
      \begin{matrix}
        0 & -1 & 0 \\
        1 & 0 & 0 \\
        0 & 0 & 0 
    \end{matrix}
  \right) \eqsp .
\end{equation}
Note that for any $i \in \{1, 2, 3\}$, $\ad(X_i) = 2 Y_i$, when represented in
the basis $(X_1, X_2, X_3)$ (recall that
$\ad: \mathfrak{g} \to \mathfrak{gl}(\mathfrak{g})$). Therefore, we have that
$\ad: \ \su(2) \to \so(3)$ is an isomorphism. Since $\SO(3)$ is compact and
connected, $\exp$ is surjective and therefore using that
$\Ad \circ \exp = \exp \circ \ad$, we get that $\Ad: \ \SU(2) \to \SO(3)$ is
surjective. In addition, we have that $\mathrm{Ker}(\Ad) = \{\pm
\mathbf{e}\}$. Hence $\SU(2)$ is a \emph{double-covering} of $\SO(3)$. 

\subsection{Parameterizations of $\SO(3)$}
\label{sec:param-so3}

Before concluding this section and describing the unitary representations of
$\SO(3)$, we describe different possible parameterizations of $\SO(3)$ and its
Lie algebra.

\paragraph{Axis-angle.}
Let $(a,b,c) \in \rset^3$ such that $a^2+b^2+c^2 = 1$, i.e.
$\bm{\omega} = (a,b,c) \in \mathbb{S}^2$ and $\theta \in \rset_+$, then any element of
$\so(3)$ is given by $Y = \theta K$, with $K = a Y_1 + b Y_2 + c Y_3$. Hence,
any element of $\SO(3)$ can be written as $\exp[\theta K]$. The parameterization
of $\SO(3)$ using $(\bm{\omega}, \theta)$ is called the \emph{axis-angle}
parameterization. Using that $K^3 = -K$ we have 
\begin{equation}
  \exp[\theta K] = \Id + \sin(\theta) K + (1 - \cos(\theta)) K^2 \eqsp . 
\end{equation}
This is called the \emph{Rodrigues' formula} and provides a concise way of
computing the exponential.  In addition, it should be noted that for any
$(a,b,c), v \in \rset^3$,
\begin{equation}
  \label{eq:relation_inter_rodrigue}
  (aY_1 + bY_2 + cY_3) v = a \times v \eqsp , \qquad (aY_1 + bY_2 + cY_3)^2 v = \langle a, v \rangle a - v \eqsp. 
\end{equation}
Combining this result \eqref{eq:relation_inter_rodrigue} we recover the
\emph{Rodrigues' rotation} formula, i.e. for any $\bm{v} \in \rset^3$ we have 
\begin{equation}
  \exp[\theta K] \bm{v} = \cos(\theta) \bm{v} + \sin(\theta) \bm{\omega} \times \bm{v} + (1 - \cos(\theta)) \langle \bm{\omega}, \bm{v} \rangle \bm{\omega} \eqsp .
\end{equation}
From this formula, it can be seen that $\exp[\theta K] \bm{v}$ is the rotation
of the vector $\bm{v}$ of angle $\theta$ around the axis $\bm{\omega}$.

\paragraph{Euler angles.}
For every $U \in \SU(2)$ there exists $(\psi, \theta, \varphi) \in \rset^3$ such
that $U = \exp[\psi X_3]\exp[\theta X_2] \exp[\varphi X_3]$. Therefore, using
that $\Ad$ is surjective and that $\Ad \circ \exp = \exp \circ \ad$ we have that
for any $R \in \SO(3)$ there exists $(\psi, \theta, \varphi) \in \rset^3$ such
that
\begin{align}
  R &= \exp[\psi Y_3]\exp[\theta Y_1] \exp[\varphi Y_3] \\
  &=
  \left(
  \begin{matrix}
    \cos(\psi) & -\sin(\psi) & 0 \\
    \sin(\psi) & \cos(\psi) & 0 \\
    0 & 0 & 1 
  \end{matrix}
\right)
\left(
  \begin{matrix}
    1 & 0 & 0 \\
    0 & \cos(\theta) & -\sin(\theta) \\
    0 & \sin(\theta) & \cos(\theta)
  \end{matrix}
\right)
\left(
  \begin{matrix}
    \cos(\varphi) & -\sin(\varphi) & 0 \\
    \sin(\varphi) & \cos(\varphi) & 0 \\
    0 & 0 & 1
  \end{matrix}
  \right) \eqsp . 
\end{align}
The three angles $\psi, \theta, \varphi$ are called the \emph{Euler angles}:
$\psi$ is called the \emph{precession angle}, $\theta$ the \emph{nutation angle}
and $\varphi$ the \emph{angle of proper rotation} (or spin).

\paragraph{Quaternions.}
Every element $U$ of $\SU(2)$ can be uniquely written as
\begin{equation}
  U = \left(
    \begin{matrix}
      \alpha & \beta \\
      -\bar{\beta} & \alpha
    \end{matrix}
  \right) \eqsp , 
\end{equation}
with $\alpha, \beta \in \cset^2$ and $\abs{\alpha}^2 + \abs{\beta}^2 = 1$. This
representation of $\SU(2)$ entails an isomorphism between $\SU(2)$ and the unit
sphere in $\rset^4$, which shows that $\SU(2)$ is simply connected. To draw the
link with quaternions, we introduce $\bm{i} = X_1$, $\bm{j} = -X_2$ and
$\bm{k} = X_3$. Note that $\bm{i}^2 = \bm{j}^2 = \bm{k}^2 = \bm{ijk} =
-1$. Using the exponential map and the properties of $\bm{i}, \bm{j}$ and
$\bm{k}$, we get that each element in $\SU(2)$ can be uniquely represented as
$q = a + b \bm{i} + c \bm{j} + d \bm{k}$ with
$\norm{q}^2 = a^2 + b^2 + c^2 + d^2 = 1$. Using the adjoint representation, we
get that $\Ad(q)$ is the rotation with axis $(b,c,d)$ and angle $\theta$ such
that $\tan(\theta/2) = \sqrt{b^2 + c^2 + d^2} / \abs{a}$ if $a \neq 0$ and
$\theta = \pi$ otherwise.


\subsection{Irreducible representations and characters of $\SO(3)$}
\label{sec:representations-so3-1}

We start by describing the irreducible characters of $\SU(2)$. Recall that
irreducible unitary representations of $\SU(2)$ are given in
\Cref{prop:irred_su2}.

\begin{proposition}{}{}
  \label{sec:repr-char-su2}
  Let $U \in \SU(2)$ such that $U = \exp[\theta X]$ with
  $X = a X_1 + b X_2 + c X_3$ and $a^2 + b^2 + c^2 = 1$, $\theta > 0$. Then for
  any $m \in \nset$ with $m \geq 1$ we have
  \begin{equation}
    \chi_m(U) = \sin((m+1)\theta)/\sin(\theta) \eqsp . 
  \end{equation}
\end{proposition}

\begin{proof}
  First note that $X^2 = -\Id$. Hence $\{-\imag \theta, \imag \theta\}$ are the
  eigenvalues of $\theta X$ and $\{\rme^{\imag \theta}, \rme^{-\imag \theta}\}$ are the
  eigenvalues of $\exp[\theta X]$. Hence, there exists $U_0 \in \SU(2)$ such
  that $U = U_0 U_\theta U_0^{-1}$ with $U_\theta$ diagonal with values
  $\{\rme^{\imag \theta}, \rme^{-\imag \theta}\}$. Hence, since $\chi_m$ is a
  trace class function we have that $\chi_m(U) = \chi_m(U_\theta)$. We have that
  $\rho_m(\imag \theta H)$ has eigenvalues $\{\imag \theta (2j -
  m)\}_{j=0}^m$. Therefore, we get that $\pi_m(U_\theta)$ has eigenvalues
  $\{\rme^{\imag \theta (2j - m)}\}_{j=0}^m$, using that
  $\exp \circ \rho_m = \pi_m \circ \exp$. We conclude upon summing the
  eigenvalues.
\end{proof}

\begin{proposition}{}{}
  Let $(\pi,V)$ be an irreducible representation of $\SO(3)$. Then there exist
  $m \in \nset$ with $m \geq 1$ and $A \in \GL(V, V_m)$ such that
  $\pi \circ \Ad = A^{-1} \pi_{2m} A$.  Respectively for any $m \in \nset$,
  there exists $\tilde{\pi}_m$ such that $\tilde{\pi}_m \circ \Ad = \pi_{2m}$.
\end{proposition}

\begin{proof}
  Let $\pi$ be an irreducible representation of $\SO(3)$. Then $\pi \circ \Ad$
  is an irreducible representation of $\SU(2)$ and therefore equivalent to
  $\pi_m$ for some $m \in \nset$ with $m \geq 1$. Since
  $\Ad(-\mathbf{e}) = \mathbf{e}$ we have that $m$ is even. Respectively, for
  any $m \in \nset$ with $m \ge 1$, since $\pi_{2m}(-\mathbf{e}) = \mathbf{e}$
  then $\pi_{2m}$ factorizes through $\Ad$, which concludes the proof.
\end{proof}

\begin{proposition}{}{}
  Let $R \in \SO(3)$ such that $R = \exp[\theta X]$ with
  $X = a Y_1 + b Y_2 + c Y_3$ and $a^2 + b^2 + c^2 = 1$, $\theta >0$ (i.e.\ we
  consider the axis-angle representation of $R$). Then for any $m \in \nset$
  with $m \geq 1$, we have
  \begin{equation}
    \chi_m(R) = \sin((m+1/2)\theta)/\sin(\theta/2) \eqsp . 
  \end{equation}
\end{proposition}

\begin{proof}
  Let $m \in \nset$ with $m \geq 1$. The associated representation with $\chi_m$
  is $\tilde{\pi}_m$ such that $\tilde{\pi}_m = \pi_{2m} \circ \Ad$. Let
  $X = a X_1 + b X_2 + c X_3$, we have that
  $\Ad(\exp[(\theta/2)X]) = \exp[\theta Y]$. We conclude using
  \Cref{sec:repr-char-su2}.
\end{proof}

We conclude this section by noting that $\SO(3)$ representations can also be
realized with spherical harmonics \cite{faraut2008analysis}. 


\section{Metrics and Laplacians}
\label{sec:metric-laplacian-se3}

In this section, we provide more details on the metrics and Laplacian on
$\SE(3)$. We start by introducing a canonical metric on $\SO(3)$ in
\Cref{sec:metr-lapl-so3}. Then, we move onto the parameterization of $\SE(3)$,
its Lie algebra and adjoint representations in \Cref{sec:param-se3-lie}. Once we
have introduced these tools we describe one metric in
\Cref{sec:choice-metr-lapl} which gives rises to the factorized formulation of
the Laplacian. Finally, we conclude with considerations on the unimodularity of
$\SE(3)$ in \Cref{sec:unim-cons}.


\subsection{Canonical metric on $\SO(3)$}
\label{sec:metr-lapl-so3}
We first describe a canonical metric on $\SO(3)$ obtained using the notion of
Killing form.  The construction of such a metric is valid for any compact Lie
group.

\paragraph{Adjoint representations.} First, we need to compute the adjoint
representation on $\SO(3)$. We recall that a basis of
$\so(3)$ is given by
\begin{equation}
    Y_1 = \left(
      \begin{matrix}
        0 & 0 & 0 \\
        0 & 0 & -1 \\
        0 & 1 & 0 
    \end{matrix}
  \right) \eqsp , \qquad
    Y_2 = \left(
      \begin{matrix}
        0 & 0 & 1 \\
        0 & 0 & 0 \\
        -1 & 0 & 0 
    \end{matrix}
  \right) \eqsp , \qquad
      Y_3 = \left(
      \begin{matrix}
        0 & -1 & 0 \\
        1 & 0 & 0 \\
        0 & 0 & 0 
      \end{matrix}      
    \right) \eqsp .
    \label{eq:basis_so3}
\end{equation}
We have that $[Y_1, Y_2] = Y_3$, $[Y_2, Y_3] = Y_1$ and $[Y_3, Y_1] = Y_2$.
We have the following result.

\begin{proposition}{}{}
  \label{sec:adjo-repr-so3}
$\ad = \Id$ and $\Ad = \Id$.
\end{proposition}

\begin{proof}
  Recalling that for any $i,j \in \{1,2,3\}$, $\ad(Y_i)(Y_j) = [Y_i, Y_j]$ we
  obtain the result using the Lie bracket relations. We conclude upon using that
  $\Ad \circ \exp = \exp \circ \ad$ and that $\exp$ is surjective since $\SO(3)$
  is compact and connected.
\end{proof}

\paragraph{Killing form.} We begin by recalling a few basics on the Killing
form.  The Killing form $B$ is a symmetric $2$-form on $\mathfrak{g}$ defined for
any $X,Y \in \mathfrak{g}$ by
\begin{equation}
  B(X,Y) = \trace(\ad(X) \circ \ad(Y))  . 
\end{equation}
One of the key property of the Killing form is that it is invariant under any
automorphims of the Lie algebra. In particular, using that for any $g \in G$,
$X,Y \in \mathfrak{g}$ and $g \in G$, $\Ad(g)[X,Y] = [\Ad(X), \Ad(Y)]$, we have
\begin{equation}
  \label{eq:adjoint_invariant}
  B(\Ad(g)(X), \Ad(g)(Y)) = B(X,Y)  . 
\end{equation}
The invariance under the adjoint representation is key to define metrics which
are bi-invariant (left and right invariant). Let $\bar{B}$ a positive symmetric
$2$-form on $\mathfrak{g}$, i.e. a scalar product. Then $\bar{B}$ defines a
\emph{left-invariant} metric $\langle \cdot, \cdot \rangle$ on $G$ by letting
for any $g \in G$ and $X,Y \in \mathrm{T}_g G$
\begin{equation}
  \langle X_g, Y_g\rangle_G = \bar{B}(\rmd L_g(\mathbf{e})^{-1} X_g, \rmd L_g(\mathbf{e})^{-1} Y_g)  ,
\end{equation}
where $L_g: \ G \to G$ is given for any $h \in G$ by $L_g(h) = gh$.

\begin{proposition}
  \label{prop:biinvariant}
  The metric $\langle \cdot, \cdot \rangle$ is right-invariant if and only if
  $\bar{B}$ is $\Ad(g)$-invariant for any $g \in G$.
\end{proposition}

\begin{proof}
  We have that $\langle \cdot, \cdot \rangle$ is right-invariant if for any
  $g, h \in G$ and $X_h, Y_h \in \mathrm{T}_hG$,
  \begin{equation}
   \langle \rmd R_g(h)(X_h), \rmd R_g(h)(Y_h)\rangle =  \langle X_h, Y_h \rangle  . 
 \end{equation}
 We have that for any
  $g, h \in G$ and $X_h, Y_h \in \mathrm{T}_hG$
  \begin{equation}
    \label{eq:scalar_product_inter}
    \langle \rmd R_g(h)(X_h), \rmd R_g(h)(Y_h)\rangle = \bar{B}(\rmd L_{hg}(\mathbf{e})^{-1} \rmd R_g(h)(X_h), \rmd L_{hg}(\mathbf{e})^{-1} \rmd R_g(h)(Y_h)) 
  \end{equation}
  In addition, using that for any $g_1,g_2 \in G$, $L_{g_1}$ and $R_{g_2}$
  commute, we have that for any $g, h \in G$ and $X_h, Y_h \in \mathrm{T}_hG$
  \begin{align}
    \rmd L_{hg}(\mathbf{e})^{-1} \rmd R_g(h) &= \rmd L_{g^{-1} h^{-1}}(hg) \rmd R_g(h) \\
                                             &= \rmd L_{g^{-1}}(g)  \rmd L_{h^{-1}}(hg) \rmd R_g(h) \\
                                             &= \rmd L_{g^{-1}}(g) \rmd R_g(e) \rmd L_{h^{-1}}(h) \\
    &= \Ad(g) \rmd L_{h^{-1}}(h) = \Ad(g) \rmd L_h(\mathbf{e})^{-1}  . 
  \end{align}
  Combining this result and \eqref{eq:scalar_product_inter} we get that for any
  $g, h \in G$ and $X_h, Y_h \in \mathrm{T}_hG$
  \begin{equation}
     \langle \rmd R_g(h)(X_h), \rmd R_g(h)(Y_h)\rangle = \bar{B}(\Ad(g) \rmd L_h(\mathbf{e})^{-1} X_h, \Ad(g) \rmd L_h(\mathbf{e})^{-1} Y_h)  . 
   \end{equation}
   In addition, we have for any $h \in G$ and $X_h, Y_h \in \mathrm{T}_hG$,
   $\langle X_h, Y_h \rangle = \bar{B}(\rmd L_h(\mathbf{e})^{-1} X_h, \rmd
   L_h(\mathbf{e})^{-1} Y_h)$. Therefore, we have that
   $\langle \cdot, \cdot \rangle$ is right-invariant if and only if
   for any
  $g, h \in G$ and $X_h, Y_h \in \mathrm{T}_hG$
  \begin{equation}
      \bar{B}(\rmd L_h(\mathbf{e})^{-1} X_h, \rmd
   L_h(\mathbf{e})^{-1} Y_h) = \bar{B}(\Ad(g) \rmd L_h(\mathbf{e})^{-1} X_h, \Ad(g) \rmd L_h(\mathbf{e})^{-1} Y_h)  . 
 \end{equation}
 Hence, we get that $\langle \cdot, \cdot \rangle$ is right-invariant if and only if
   for any $g \in G$ and $X,Y \in \mathfrak{g}$,
   \begin{equation}
     \bar{B}(X,Y) = \bar{B}(\Ad(g)(X), \Ad(g)(Y))  ,
   \end{equation}
   which concludes the proof.
\end{proof}

Combining this result and \eqref{eq:adjoint_invariant} we immediately get that
if the Killing form defines a scalar product then the associated left-invariant
metric is also right-invariant. In the case of $\SO(3)$ we have the following
explicit formula for the Killing form.

\begin{proposition}{}{}
  If $G = \SO(3)$ we have that $B(X,Y) = \trace(XY)$. In the basis
  $(Y_1, Y_2, Y_3)$ we have that $B = -2\Id$.
\end{proposition}

\begin{proof}
  The first result is a direct consequence of \Cref{sec:adjo-repr-so3}. The
  second result is a consequence of the fact that
  $\trace(Y_i Y_j) = -2\updelta_{i,j}$ for $i,j \in \{1,2,3\}$. 
\end{proof}

Hence by considering $-B/2$ we obtain that $\{Y_1, Y_2, Y_3\}$ is an orthonormal
basis on $\so(3)$. The associated metric is bi-invariant. We can define the
Laplace-Beltrami operator associated with $\so(3)$ and we have that for any
$f\in \rmc^\infty(\SO(3))$ and $g \in \SO(3)$
\begin{equation}
  \textstyle{\Delta f(g) = \sum_{i=1}^3 \tfrac{\rmd}{\rmd t^2}f(g \exp[t Y_i])|_{t=0}  .}
\end{equation}
Also, note that in that case the \emph{Riemannian} exponential mapping coincide
with the \emph{matrix} exponential map, \citep[Chapter 3, Exercise
3]{1992Riemannian}. 

\paragraph{Eigenvalues of the Laplacian.}
Similarly, one can define $\Delta$ on $\SU(2)$ using the Killing form. In this case we
have that $B(X,Y) = -\trace(XY)$ and we set the metric on $\SU(2)$ to be the one
associated with $-B/2$. We have that $\{X_i\}_{i=1}^3$ is an orthonormal basis
of $\su(2)$ for this metric and therefore for any $f\in \rmc^\infty(\SU(2))$ and
$g \in \SU(2)$
\begin{equation}
    \textstyle{\Delta f(g) = \sum_{i=1}^3 \tfrac{\rmd}{\rmd t^2}f(g \exp[t X_i])|_{t=0}  ,}
\end{equation}
see \citep[p.162]{faraut2008analysis} for a definition and basic properties. It
can be shown \citep[Proposition 8.2.1, Proposition 8.3.1]{faraut2008analysis}
that for any $m \in \nset$ with $m \geq 1$, we have
\begin{equation}
  \Delta \chi_m = -m (m+2) \chi_m  . 
\end{equation}
Using that $\Ad$ is surjective for any $g \in \SO(3)$ there exists
$g_0 \in \SU(2)$ such that $\Ad(g_0) = g$. In addition, for any
$i \in \{1,2,3\}$, $\ad(X_i) = 2 Y_i$. Using these results and the fact that
$\Ad \circ \exp = \exp \circ \ad$ we have that for any
$f\in \rmc^\infty(\SO(3))$ and $g \in \SO(3)$
\begin{align}
  \Delta f(g) &\textstyle{= \sum_{i=1}^3 \tfrac{\rmd}{\rmd t^2}f(g \exp[t Y_i])|_{t=0}} \\
              &\textstyle{= \sum_{i=1}^3 \tfrac{\rmd}{\rmd t^2}f(g \exp[\ad(t X_i /2)])|_{t=0}} \\
              &\textstyle{= \sum_{i=1}^3 \tfrac{\rmd}{\rmd t^2}f(g  \Ad (\exp[t X_i /2]))|_{t=0}} \\
  &\textstyle{= \sum_{i=1}^3 \tfrac{\rmd}{\rmd t^2}f(\Ad(g_0 \exp[t X_i /2]))|_{t=0}} = \Delta (f \circ \Ad)(g_0) /4  . \label{eq:delta_relation}
\end{align}

This result yields the following proposition.
\begin{proposition}{}{}
  For every $m \in \nset$, $\Delta \tilde{\chi}_m = -m (m+1) \tilde{\chi}_m$.
\end{proposition}

\begin{proof}
  Recall that for any $m \in \nset$, $\chi_{2m} = \tilde{\chi}_m \circ \Ad$.
  Therefore, using \eqref{eq:delta_relation}, we have that for any
  $g \in \SO(3)$ and $m \in \nset$
  \begin{equation}
    \Delta \tilde{\chi}_m(g) = \Delta \chi_{2m}(g_0)/4 = -m(m+1) \chi_{2m}(g_0) = -m(m+1) \tilde{\chi}_m(g)  . 
  \end{equation}
\end{proof}




\subsection{Parameterization of $\SE(3)$ and Lie algebra}
\label{sec:param-se3-lie}

\paragraph{Parameterization.}


The special Euclidean group on $\rset^3$, denoted $\SE(3)$, (also known as the
rigid body motion group, see \citep{murray1994mathematical}) is the group given
by all the affine isometries. We have
\begin{equation}  
  \SE(3) = \ensemble{\left(\begin{matrix}
                                  R & x \\
                                  0 & 1
                                  \end{matrix}\right)}{R \in \SO(3), x \in \rset^3}  . 
                            \end{equation}
                            As a consequence we have the following composition
                            rule for $(R,x), (R'x,') \in \SE(3)$
                            \begin{equation}
                              (R,x) * (R',x') = (RR', x + Rx')  . 
                            \end{equation}
                            Therefore as a group we have that
                            $\SE(3) = \SO(3) \rtimes \rset^3$. In particular,
                            the group structure of $\SE(3)$ is different from
                            the canonical product $\SO(3) \times \rset^3$. The
                            inverse of $(R,x)$ is given by
                            $(R,x)^{-1} = (R^{-1}, -R^{-1}x)$. $\SE(3)$ is also
                            a $6$-dimensional Lie group and its Lie algebra is
                            given by
                            \begin{equation}
                              \se(3) = \ensemble{
                                \left(\begin{matrix}
                                  X & x \\
                                  0 & 0
                                  \end{matrix}\right)}{X \in \so(3), x \in \rset^3}  . 
                            \end{equation}
                            A basis for $\se(3) = \so(3) \oplus \rset^3$ is
                            given by $\{Y_1, Y_2, Y_3, e_1, e_2, e_3\}$ where
                            $\{Y_1, Y_2, Y_3\}$ is a basis for $\so(3)$, see
                            \eqref{eq:basis_so3}.
                            
                            \paragraph{Adjoint representations.}
                            Let us now compute the adjoint representation of $\SE(3)$.
                            We have the following result.

                            \begin{proposition}
                              \label{eq:adjoint_se3}
                              We have that for any $g = (R,x) \in \SE(3)$ we have 
                              \begin{equation}
                                \Ad(g) = \left(
                                  \begin{matrix}
                                    R & 0 \\
                                    M & R
                                  \end{matrix} \right)  ,
                              \end{equation}
                              in the basis $\{Y_1, Y_2, Y_3, e_1, e_2, e_3\}$
                              with
                              $M = (-R Y_1 R^{-1} x | -R Y_2 R^{-1} x | -R Y_3
                              R^{-1} x)$.
                            \end{proposition}

                            \begin{proof}
                              Let $i \in \{1,2,3\}$. We have that
                              \begin{equation}
                                \Ad(g)(X_i) = \left(\begin{matrix}
                                  R & x \\
                                  0 & 1
                                                    \end{matrix}\right)
\left(\begin{matrix}
                                  X_i & 0 \\
                                  0 & 0
                                  \end{matrix}\right)                                                  
                                \left(\begin{matrix}
                                        R^{-1} & -R^{-1}x \\
                                        0 & 1
                                  \end{matrix}\right) = \left(
                                  \begin{matrix}
                                    R X_i R^{-1} & -R X_i R^{-1} x \\
                                    0 & 0 
                                  \end{matrix}
                                \right)  . 
                              \end{equation}
                              Similarly, for any $\xi \in \rset^3$ we have
                              \begin{equation}
                                \Ad(g)(\xi) = \left(\begin{matrix}
                                  R & x \\
                                  0 & 1
                                                    \end{matrix}\right)
\left(\begin{matrix}
                                  0 & \xi \\
                                  0 & 0
                                  \end{matrix}\right)                                                  
                                \left(\begin{matrix}
                                        R^{-1} & -R^{-1}x \\
                                        0 & 1
                                  \end{matrix}\right) = \left(
                                  \begin{matrix}
                                    0 & R \xi \\
                                    0 & 0 
                                  \end{matrix}
                                \right)  ,
                              \end{equation}
                              which concludes the proof upon using that
                              $\Ad = \Id$ on $\SO(3)$, see
                              \Cref{sec:adjo-repr-so3}.
                            \end{proof}

\subsection{Choice of metric and Laplacian derivation}
\label{sec:choice-metr-lapl}

\paragraph{A left invariant metric.}
It can be shown that the Killing form is not negative and therefore there is no
canonical metric on $\SE(3)$. In fact in this section, we show that there is no
bi-invariant metric on $\SE(3)$. However, one specific choice of left-invariant
metric on $\SE(3)$ leads to a metric (and Laplacian) that factorizes between
$\SO(3)$ and $\rset^3$. Roughly speaking, this implies that \emph{as a
  Riemannian manifold} $\SE(3)$ can be seen as $\SO(3) \times \rset^3$. The
following proposition can be found in see \citep[Proposition
A.5]{murray1994mathematical} and is a consequence of \Cref{eq:adjoint_se3} and
\Cref{prop:biinvariant}.

\begin{proposition}{}{}
  \label{sec:left-invar-metr}
  Let $\bar{B}$ be a symmetric $2$-form on $\se(3)$. Then
  $\bar{B}$ is $\Ad$ invariant if and only if there exist $\alpha, \beta >0$
  s.t.\
  \begin{equation}
    \bar{B} = \left(
      \begin{matrix}
        \alpha \Id & \beta \Id \\
        \beta \Id & 0
      \end{matrix}
    \right)  ,
  \end{equation}
  where $\bar{B}$ is expressed in the basis $\{Y_1, Y_2, Y_3, e_1, e_2, e_3\}$
  where $\{Y_1, Y_2, Y_3\}$ is a basis for $\so(3)$, see \eqref{eq:basis_so3}.
\end{proposition}

Note that in any case $\bar{B}$ is not positive definite and therefore, there
does not exist any bi-invariant metric on $\SE(3)$. However, one can define
pseudo metrics.  Letting $\beta = 1$ and $\alpha = 0$ one recover the
\emph{Klein form} which yields an hyperbolic metric on $\SE(3)$. If one lets
$\alpha = -4$ then we recover the \emph{Killing form}.

In this work, we consider the metric $\bar{B} = \Id$. According to
\Cref{sec:left-invar-metr} the associated metric on $\SE(3)$ is left-invariant
but not right-invariant. However, this metric has interesting properties which
we list below. We denote $\langle \cdot, \cdot \rangle_{\SE(3)}$ the metric
associated with $\bar{B}$, $\langle \cdot, \cdot \rangle_{\SO(3)}$ the one
associated with the Killing form in $\SO(3)$, see
\Cref{sec:metric-laplacian-se3} and $\langle \cdot, \cdot \rangle$ the Euclidean
inner product. 

\begin{proposition}[Metric on $\SE(3)$]
  \label{prop:brownian_app}
  For any $T \in \SE(3)$ and
  $(a, x), (a^\prime, x^\prime) \in \mathrm{Tan}_T\SE(3)$ we define
  $\langle (a, x), (a^\prime, x^\prime)\rangle_{\SE(3)}= \langle a,
  a^\prime\rangle_{\SO(3)} + \langle x, x^\prime \rangle_{\R^3}.$ We have:
  \begin{enumerate}[label=(\alph*)]
  \item for any $f \in \rmc^\infty(\SE(3))$ and $T=(r,x) \in \SE(3)$,
    $\nabla_T f(T) = [\nabla_r f(r,x), \nabla_x f(r,x)]$.
  \item for any $f \in \rmc^\infty(\SE(3))$ and $T=(r,x) \in \SE(3)$,
    $\Delta_{\SE(3)}f(T) = \Delta_{\SO(3)}f(r,x) + \Delta_{\rset^3}f(r,x)$. In
    addition, $T \mapsto \Delta_{\SE(3)}f(T)$ is $\SE(3)$-equivariant (for the
    left action).
  \item for any $t>0$,
  $\dt{\bfB_{\SE(3)}}t = [\dt{\bfB_{\SO(3)}}{t}, \dt{\bfB_{\R^3}}{t}]$ with
  independent $\dt{\bfB_{\SO(3)}}{t}$ and $\dt{\bfB_{\R^3}}{t}.$
\item For any $(R_0, x_0) \in \SE(3)$ and
  $(X, x) \in \mathrm{Tan}_{(R_0,x_0)}\SE(3)$ we have
  $\exp_{(R_0,x_0)}[X, x] = (R_0 \exp[R_0^{-1} X], x_0 + x)$.
  \end{enumerate}
\end{proposition}

\begin{proof}
  We have that $\{Y_1, Y_2, Y_3, e_1, e_2, e_3\}$ where $\{Y_1, Y_2, Y_3\}$ is a
  basis for $\so(3)$, see \eqref{eq:basis_so3}, is an orthonormal basis for
  $\se(3)$. By definition of the metric on $\SE(3)$, we also have that for any
  $(R,x) \in \SE(3)$, $\{RY_1, RY_2, RY_3, Re_1, Re_2, Re_3\}$ (note the action
  of $R$ on the $\rset^3$ components) is an orthonormal basis on
  $\mathrm{Tan}_{(R,x)}\SE(3)$. However, another orthonormal basis of $\se(3)$
  is given by $\{Y_1, Y_2, Y_3, R^{-1}e_1, R^{-1}e_2, R^{-1}e_3\}$ which implies
  that $\{RY_1, RY_2, RY_3, e_1, e_2, e_3\}$ is an orthonormal basis of
  $\mathrm{Tan}_{(R,x)}\SE(3)$.  We divide the rest of the proof into four
  parts.
  \begin{enumerate}[wide, labelindent=0pt, label=(\alph*)]  
  \item First, we show that for any $f \in \rmc^\infty(\SE(3))$ and
    $T=(r,x) \in \SE(3)$, $\nabla_T f(T) = [\nabla_r f(r,x), \nabla_x
    f(r,x)]$. Let $f \in \rmc^\infty(\SE(3))$ and $T=(r,x) \in \SE(3)$. Consider
    the smooth curve $\gamma: \ \ccint{-\vareps, \vareps} \to \SE(3)$ given for any
    $t \in \ccint{-\vareps, \vareps}$, by $\gamma(t) = (R \exp[t Y_1],x)$. We
    have that
    \begin{equation}
      \tfrac{\rmd }{\rmd t} f(\gamma(t))|_{t=0} = \tfrac{\rmd }{\rmd t} f(R \exp[t Y_1], x)|_{t=0} = \rmd f (R, x) (R Y_1) = (\nabla_r f (R,x))_1 ,
    \end{equation}
    since $\{RY_1, RY_2, RY_3\}$ is an orthonormal basis of
    $\mathrm{Tan}_{R} \SO(3)$. Similarly, we have that
    $\{RY_1, RY_2, RY_3, e_1, e_2, e_3\}$ is an orthonormal basis of
    $\mathrm{Tan}_T \SE(3)$. Consider
    the smooth curve $\gamma: \ \ccint{-\vareps, \vareps} \to \SE(3)$ given for any
    $t \in \ccint{-\vareps, \vareps}$, by $\gamma(t) = (R ,x+ t e_1)$. We
    have that
    \begin{equation}
      \tfrac{\rmd }{\rmd t} f(\gamma(t))|_{t=0} = \tfrac{\rmd }{\rmd t} f(R, x+te_1)|_{t=0} = \rmd f (R, x) (e_1) = (\nabla_x f (R,x))_1 ,
    \end{equation}
    which concludes the proof.
  \item By definition of the divergence, the previous point and using that
    $\{RY_1, RY_2, RY_3, e_1, e_2, e_3\}$ is an orthonormal basis of
    $\mathrm{Tan}_{(R,x)}\SE(3)$, we have
    \begin{equation}
      \textstyle{\Delta_{\SE(3)} f = \mathrm{div}(\nabla_T f) = \sum_{i=1}^3 \langle \nabla_{R Y_i} \nabla_r f, R Y_i \rangle_{\SO(3)} + \sum_{i=1}^3  \langle \nabla_{e_i} \nabla_r f, e_i \rangle_{\rset^3} = \Delta_{\SO(3)} f + \Delta_{\rset^3} f. }
    \end{equation}
    The equivariance property is a direct consequence of the definition of the Laplacian, see \Cref{sec:invar-diff-proc}.
  \item For any $t>0$,
    $\dt{\bfB_{\SE(3)}}t = [\dt{\bfB_{\SO(3)}}{t},
    \dt{\bfB_{\R^3}}{t}]$. According to the previous point, we have that for any
    $f \in \rmc^\infty(\SE(3))$.
    \begin{equation}
      \textstyle{f(\dt{\bfB_{\SE(3)}}t) - f(\dt{\bfB_{\SE(3)}}0) - (1/2) \int_0^t \Delta_{\SE(3)} f(\dt{\bfB_{\SE(3)}}s) \rmd s ,}
    \end{equation}
    which is a local martingale (with respect to the filtration associated with
    $(\dt{\bfB_{\SO(3)}}{t})_{t \geq 0}$ and
    $(\dt{\bfB_{\R^3}}{t})_{t \geq 0}$). Using \citep[Proposition
    3.2.1]{hsu2002stochastic}, we have that $(\dt{\bfB_{\SE(3)}}t)_{t \geq 0}$
    is a Brownian motion on $\SE(3)$.
  \item Let $\gamma : \ \ccint{-\vareps, \vareps} \to \SE(3)$ a smooth curve and consider
    \begin{equation}
      \textstyle{E(\gamma) = \int_{-\vareps}^\vareps \| \gamma'(t) \|^2_{\SE(3)} \rmd t = \int_{-\vareps}^\vareps \| \gamma'_r(t) \|^2_{\SO(3)} \rmd t + \int_{-\vareps}^\vareps \| \gamma'_x(t) \|^2_{\rset^3} \rmd t} ,
    \end{equation}
    where $\gamma = [\gamma_r, \gamma_x]$. $\gamma$ is a geodesics between
    $\gamma(-\vareps)$ and $\gamma(\vareps)$ if it minimizes $E(\gamma)$, see
    \citep[Section 9.2]{1992Riemannian}. Therefore, $\gamma_r$ is the
    geodesics on $\SO(3)$ between $\gamma_r(-\vareps)$ and $\gamma_r(\vareps)$
    and $\gamma_x$ is the geodesics on $\rset^3$ between $\gamma_x(-\vareps)$ and
    $\gamma_x(\vareps)$, which concludes the proof.
  \end{enumerate}
\end{proof}


This proves \Cref{prop:brownian}.  In particular, note that the exponential
mapping on $\SE(3)$ does not coincide with the matrix exponential mapping
contrary to the compact Lie group setting like $\SO(3)$.




\subsection{Haar measure on $\SE(3)$}
\label{sec:unim-cons}

We conclude this section with some measure theoretical consideration on
$\SE(3)$. Let $G$ be a locally compact Hausdorff topological group. The Borel
algebra $\mathcal{B}(G)$ is the $\sigma$-algebra generated by the open subsets
of $G$. A left-invariant Haar measure is a measure $\mu$ on the Borel subsets of
$G$ such that:
\begin{enumerate}[label=(\alph*)]
\item For any $g \in G$ and $\msa \in \mathcal{B}(G)$, $\mu(g \msa) = \mu(\msa)$.
\item For any $\msk$ compact, $\mu(\msk) < +\infty$.
\item For any $\msa \in \mathcal{B}(G)$, $\mu(\msa) = \inf \ensembleLigne{\mu(\msu)}{\msa \subset \msu, \ \text{$\msu$ open}}$.
\item For any $\msu$ open, $\mu(\msu) = \sup \ensembleLigne{\mu(\msk)}{\msk \subset \msu, \ \text{$\msk$ compact}}$.
\end{enumerate}
Similarly, we define right-invariant Haar measures. Haar's theorem asserts that
left-invariant and right-invariant Haar measures are unique up to a positive
multiplicative scalar. A group $G$ for which the left and right-invariant Haar
measures coincide is called a \emph{unimodular} group. It can be shown that the
product measure between $\mu_{\SO(3)}$ (the Haar measure on $\SO(3)$) and the
Lebesgue measure on $\rset^3$ is a left and right invariant measure on
$\SE(3)$. This measure can be realized as the volume form associated with the
metrics described in the previous section.


\section{Heat kernel on Lie groups: theory and practice}
\label{sec:heat_kernel}

We start this section with a result on the heat kernel on $\SO(3)$ in
\Cref{sec:heat-kernel-compact}.  Then, we present practical considerations in
\Cref{sec:igso3_time_scaling_discussion_and_related_work} and
\Cref{sec:pytorch_SO3_example}.

\subsection{Heat kernel on compact Lie groups}
\label{sec:heat-kernel-compact}

On a compact Lie group we have the following result, see \citet[Section
2.5.1]{ebert2011wavelets} for instance.

\begin{proposition}[Brownian motion on compact Lie groups]
  \label{prop:lie_group_transition_app}
  Assume that $\M$ is a compact Lie group, where for any $\ell \in \nset$
  $\chi_\ell$ is the character associated with the
  irreducible unitary representation 
  of dimension $d_\ell$. Then
  $\chi_\ell: \ \M \to \rset$ is an eigenvector of $\Delta$ and there exists $\lambda_\ell\geq 0 $
  such that $\Delta \chi_\ell = - \lambda_\ell \chi_\ell$. In addition, we have
  for any $t >0 $ and $\dt{x}{0}, \dt{x}{t} \in \M$
  \begin{equation} \label{eq:heat_kernel_lie_group}
    \textstyle{p_{t|0}(\dt{x}{t}|\dt{x}{0})
      = \sum_{\ell\in \nset} d_\ell e^{-\lambda_\ell t/2}
      \chi_\ell((\dt{x}{0})^{-1} \dt{x}{t}).}
  \end{equation}
\end{proposition}

It is important to note here that we have implicitly chosen a Brownian motion
and therefore a metric to define the Laplace-Beltrami operator. The metric
chosen here is the canonical invariant metric given by the Killing form which is
bi-invariant in the compact case. 

  In the special case of $\SO(3)$ it turns out that the characters can be
  computed as shown in \cref{sec:representations-so3-1}.

\begin{proposition}[Brownian motion on $\SO(3)$]\label{prop:brownian_on_SO3_app}
  For any $t >0$ and $\dt{r}{0}, \dt{r}{t} \in \SO(3)$ we have that
  $p_{t|0}(\dt{r}{t}|\dt{r}{0}) = \IGSO_3(\dt{r}{t} ; \dt{r}{0}, t)$ given 
  by 
  \begin{equation}
\label{eq:igso3_app}
\IGSO_3(\dt{r}{t} ; \dt{r}{0}, t) = f(\omega(r^{(0)\top} \dt{r}{t}), t) ,
  \end{equation}
  where $\omega(r)$ is the rotation angle in radians for any
  $r \in \SO(3)$---its length in the axis--angle representation\footnote{See
    \Cref{sec:param-so3} for  details about the parameterization of $\SO(3)$.}--- and
\begin{equation}
f(\omega, t) 
= \textstyle{\sum_{\ell\in \nset} (2 \ell + 1) \rme^{-\ell(\ell+1) t/2} 
  \tfrac{\sin((\ell+1/2)\omega)}{\sin(\omega/2)}.}
 \end{equation}
\end{proposition}

We can also give a similar result on $\SU(2)$ using the same tools, see
\citet{fegan1983fundamental}.

\begin{proposition}[Brownian motion on $\SU(2)$]\label{prop:brownian_on_SU2_app}
  For any $t >0$ and $\dt{r}{0}, \dt{r}{t} \in \SU(3)$ we have that
  $p_{t|0}(\dt{r}{t}|\dt{r}{0}) = \IGSU_2(\dt{r}{t} ; \dt{r}{0}, t)$ given 
  by 
  \begin{equation}
\label{eq:igsu2_app}
\IGSU_2(\dt{r}{t} ; \dt{r}{0}, t) = f(\omega(r^{(0)\top} \dt{r}{t}), t) ,
  \end{equation}
  where $\omega(r)$ is the rotation angle in radians for any
  $r \in \SU(2)$---its length in the axis--angle representation--- and
\begin{equation}
f(\omega, t) 
= \textstyle{\sum_{\ell\in \nset, \ell \geq 1} \ell^2 \rme^{-(\ell^2-1) t/8} 
  \tfrac{\sin(\ell \omega)}{\sin(\omega)}.}
 \end{equation}
\end{proposition}

\subsection{Sampling and evaluating density of Brownian motion on $\SO(3)$}\label{sec:igso3_evaluation_sampling}
In practice, we obtain a tractable and accurate approximation of the Brownian motion density by truncating the series \eqref{eq:igso3_app} with $N=2000$ terms as
\begin{equation} \label{eq:igso3_approx_app}
{p}_{t|0}(\dt{r}{t}|\dt{r}{0}) \approx \tilde{p}_{t|0}(\dt{r}{t}|\dt{r}{0}) \triangleq \sum_{\ell=0}^{N-1} (2 \ell + 1) \rme^{-\ell(\ell+1) t/2} 
  \tfrac{\sin((\ell+1/2)\omega)}{\sin(\omega/2)}.
\end{equation}
We similarly approximate the conditional score $\nabla_{\dt{r}{t}} \log  p_{t|0}(\dt{r}{t}\mid \dt{r}{0})=
\tfrac{\dt{r}{t}}{
\omega^{(t)}
} \log\{r^{(0, t)}\} 
\frac{\partial_\omega f(\omega^{(t)}, t)}{f(\omega^{(t)}, t)}$ from \cref{prop:deno-score-match} by truncating the partial derivative $\partial_\omega f(\omega^{(t)}, t)$ term.

Following \citet{leach2022denoising}, samples are obtained via inverse transform sampling, where the cdf is numerically approximated through trapezoidal integration of the truncated density \eqref{eq:igso3_approx_app} .

\subsection{Diffusion modeling on $\SO(3)$, and the scaling of time in the $\IGSO_3$ density of the Brownian motion}\label{sec:igso3_time_scaling_discussion_and_related_work}

It is worth mentioning as well that the choice of inner product on $\so(3)$ influences the speed of the Brownian motion.
In particular, in the present work we have chosen to define 
$\langle u, v \rangle_{\so(3)}
= \mathrm{Tr}(u v^\top)/2 $ because this 
is the metric for which the canonical basis vectors of $\so(3)$ 
(\Cref{sec:double-covering-so3}) are orthonormal.
However, had we instead chosen $\langle u, v\rangle_{\so(3)}=\mathrm{Tr}(uv^\top)$ the Brownian motion would again have a different speed, and the normalization in the conditional score in \Cref{prop:deno-score-match} would also be different.

Additionally, another source of error is the confusion between the heat kernel
$(q_t)_{t \geq 0}$ satisfying $\partial_t q_t = \Delta q_t$ and the density of
the Brownian motion $(p_t)_{t \geq 0}$ satisfying
$\partial_t p_t = \tfrac{1}{2}\Delta p_t$. The origin of this factor $1/2$ can
be traced back to the Fokker-Planck equation which describes the evolution of
the density of the Brownian motion.

Other recent works have attempted a generative modeling on rotations through an iterative denoting paradigm akin to diffusion modeling in applications to protein modeling \citep{anand2022protein,luo2022antigen}, as well as robotics \citep{urain2022se}.
However, the  associated ``forward noising'' mechanisms in these works are not defined with respect to an underlying diffusion and do not have a well defined time-reversal.
We hope that our thorough identification of the law of the $\dt{\mathbf{B}_{\SO(3)}}{t}$, its score, and its time reversal provides stable ground for further work on generative modeling on $\SO(3)$ across a variety of application areas.

\subsection{Pytorch implementation of $\IGSO_3$, and simulation of forward and reverse process on a toy example}\label{sec:pytorch_SO3_example}

The goal of this section is to provide a minimal example of a forward and
reverse process on $\SO(3)$.
In particular, we pay attention to the
definition of the exponential, the sampling of a normal with zero mean and
identity covariance matrix in the tangent space, and the sampling from
$\IGSO(3)$.

In the example that follows, we consider as a target $p_0$ a discrete measure on $\SO(3)$
$$
\textstyle p_0(dR) = N^{-1}\sum_{n=1}^N\delta_{R_n}(dR),
$$
where $\delta_{R_n}$ denotes a Dirac mass on $R_n$ and the atoms locations $R_n$ are chosen randomly by sampling from the uniform distribution on $\SO(3).$

The intermediate densities are defined via the transition kernel of the Brownian motion as 
$$
\textstyle p_t(dR) \int_{R_0} p_{t|0}(dR|\dt{R}{0})p_{0}(dR_0),
$$
and the Stein score of these densities $\nabla_{R} \log p_t(dR)$ is computed using automatic differentiation.

When the forward and reverse processes are simulated using a geodesic random walk as implemented in Listing 4, their marginal distributions closely agree for each time $t.$

\begin{lstlisting}[language=Python, caption=Primitives for moving between parameterizations of SO(3)]
import numpy as np
import torch
from scipy.spatial.transform import Rotation
import scipy.linalg

 # Orthonormal basis of SO(3) with shape [3, 3, 3]
basis = torch.tensor([
    [[0.,0.,0.],[0.,0.,-1.],[0.,1.,0.]],
    [[0.,0.,1.],[0.,0.,0.],[-1.,0.,0.]],
    [[0.,-1.,0.],[1.,0.,0.],[0.,0.,0.]]])
    
# hat map from vector space R^3 to Lie algebra so(3)
def hat(v): return torch.einsum('...i,ijk->...jk', v, basis) 

# Logarithmic map from SO(3) to R^3 (i.e. rotation vector)
def Log(R): return torch.tensor(Rotation.from_matrix(R.numpy()).as_rotvec()) 

# logarithmic map from SO(3) to so(3), this is the matrix logarithm
def log(R): return hat(Log(R))

# Exponential map from so(3) to SO(3), this is the matrix exponential
def exp(A): return torch.linalg.matrix_exp(A)

# Exponential map from tangent space at R0 to SO(3)
def expmap(R0, tangent):
    skew_sym = torch.einsum('...ij,...ik->...jk', R0, tangent) 
    return torch.einsum('...ij,...jk->...ik', R0, exp(skew_sym))
    
# Return angle of rotation. SO(3) to R^+
def Omega(R): return torch.arccos((torch.diagonal(R, dim1=-2, dim2=-1).sum(axis=-1)-1)/2)
\end{lstlisting}

\begin{lstlisting}[language=Python, caption=Primitives for simulating and reversing the Brownian motion.]
# Power series expansion in the IGSO3 density.
def f_igso3(omega, t, L=500):
    ls = torch.arange(L)[None]  # of shape [1, L]
    return ((2*ls + 1) * torch.exp(-ls*(ls+1)*t/2) *
             torch.sin(omega[:, None]*(ls+1/2)) / torch.sin(omega[:, None]/2)).sum(dim=-1)

# IGSO3(Rt; I_3, t), density with respect to the volume form on SO(3) 
def igso3_density(Rt, t, L=500): return f_igso3(Omega(Rt), t, L)

# Normal sample in tangent space at R0
def tangent_gaussian(R0):
    return torch.einsum('...ij,...jk->...ik', R0, hat(torch.randn(R0.shape[0], 3)))


# Riemannian gradient of f at R
def riemannian_gradient(f, R):
    coefficients = torch.zeros(list(R.shape[:-2])+[3], requires_grad=True)
    R_delta = expmap(R, torch.einsum('...ij,...jk->...ik', R, hat(coefficients)))
    grad_coefficients = torch.autograd.grad(f(R_delta).sum(), coefficients)[0]
    return torch.einsum('...ij,...jk->...ik', R, hat(grad_coefficients))

# Simluation procedure for forward and reverse
def geodesic_random_walk(p_initial, drift, ts):
    Rts = {ts[0]:p_initial()}
    for i in range(1, len(ts)):
        dt = ts[i] - ts[i-1] # negative for reverse process
        Rts[ts[i]] = expmap(Rts[ts[i-1]],
            drift(Rts[ts[i-1]], ts[i-1]) * dt + 
            tangent_gaussian(Rts[ts[i-1]]) * np.sqrt(abs(dt)))
    return Rts
\end{lstlisting}

\paragraph{Scaling rules. } As noted in
\Cref{sec:igso3_time_scaling_discussion_and_related_work}, the choice of inner
product impacts the scalings of several objects in the implementation in Listing
2. Let $\langle \cdot, \cdot \rangle$ be an inner product on $G$ and denote
$\langle \cdot, \cdot \rangle_{\alpha}$ the inner product given by
$\langle \cdot, \cdot \rangle_{\alpha} = \alpha \langle \cdot, \cdot \rangle$. We
consider a test function $f \in \rmc^\infty(G)$ and $X \in \mathfrak{X}(G)$ a
vector field.
\begin{enumerate}[label=(\alph*)]
\item If $\nabla f$ is the gradient of $f$ w.r.t.\ $\langle \cdot, \cdot \rangle$, then  $\nabla f / \alpha$ is the gradient of $f$ w.r.t.\ $\langle \cdot, \cdot \rangle_\alpha$.
\item If $\mathrm{div}(X)$ is the divergence of $X$ w.r.t.\ $\langle \cdot, \cdot \rangle$, then  $\mathrm{div}(X)$ is the gradient of $X$ w.r.t.\ $\langle \cdot, \cdot \rangle_\alpha$.
\item If $\Delta f$ is the Laplace-Beltrami of $f$ w.r.t.\ $\langle \cdot, \cdot \rangle$, then  $\Delta f / \alpha$ is the Laplace-Beltrami of $f$ w.r.t.\ $\langle \cdot, \cdot \rangle_\alpha$.
\item If $\{ X_i \}_{i=1}^d$ is an orthonormal basis of $\mathrm{Tan}_gG$ at $g \in G$ w.r.t\ $\langle \cdot, \cdot \rangle$. then $\{ X_i /\sqrt{\alpha} \}_{i=1}^d$ is an orthonormal basis of $\mathrm{Tan}_gG$ at $g \in G$ w.r.t\ $\langle \cdot, \cdot \rangle_\alpha$.
\item If $Z$ is a Gaussian random variable with zero mean and identity
  covariance in $\mathrm{Tan}_g G$ at $g \in G$ w.r.t.\
  $\langle \cdot, \cdot \rangle$, then $Z/\sqrt{\alpha}$ is a Gaussian random variable with
  zero mean and identity covariance in $\mathrm{Tan}_g G$ at $g \in G$ w.r.t.\
  $\langle \cdot, \cdot \rangle_\alpha$.
\item If $\exp$ is the exponential mapping w.r.t.\
  $\langle \cdot, \cdot \rangle$, then $\exp$ is the exponential mapping w.r.t.\
  $\langle \cdot, \cdot \rangle_\alpha$.
\end{enumerate}

\begin{lstlisting}[language=Python, caption={Instantiation of invariant density, discrete target measure, and its Stein score.}]
# Sample N times from U(SO(3)) by inverting CDF of uniform distribution of angle 
def p_inv(N, M=1000):
    omega_grid = np.linspace(0, np.pi, M)
    cdf = np.cumsum(np.pi**-1 * (1-np.cos(omega_grid)), 0)/(M/np.pi)
    omegas = np.interp(np.random.rand(N), cdf, omega_grid)
    axes = np.random.randn(N, 3)
    axes = omegas[:, None]* axes/np.linalg.norm(axes, axis=-1, keepdims=True)
    return exp(hat(torch.tensor(axes)))

# Define discrete target measure on SO(3), and it's score for t>0
N_atoms = 3
mu_ks = p_inv(N_atoms) # Atoms defining target measure

# Sample p_0 ~ (1/N_atoms)\sum_k Dirac_{mu_k}
def p_0(N): return mu_ks[torch.randint(mu_ks.shape[0], size=[N])] 

# Density of discrete target noised for time t
def p_t(Rt, t): return sum([
        igso3_density(torch.einsum('ji,...jk->...ik', mu_k, Rt), t)
        for mu_k in mu_ks])/N_atoms

# Stein score, grad_Rt log p_t(Rt)
def score_t(Rt, t): return riemannian_gradient(lambda R_: torch.log(p_t(R_, t)), Rt)
\end{lstlisting}

\begin{lstlisting}[language=Python, caption={Simulation of forward and reverse processes.}]
### Set parameters of simulation
N = 5000 # Number of samples
T = 4. # Final time 
ts = np.linspace(0, T, 200) # Discretization of [0, T]

# Simulate forward process
forward_samples = geodesic_random_walk(
    p_initial=lambda: p_0(N), drift=lambda Rt, t: 0., ts=ts)

# Simulate reverse process
reverse_samples = geodesic_random_walk(
    p_initial=lambda: p_inv(N), drift=lambda Rt, t: -score_t(Rt, t), ts=ts[::-1])
\end{lstlisting}



\section{Invariant diffusion processes}
\label{sec:invariant_process}

In this section, we prove \Cref{prop:invariance_lie}.  Let $G$ be a Lie group
and $H$ a subgroup acting on $G$. We define the left shift operator
$L_h(g) = hg$. Note that since, we are on a Lie group, this function is
differentiable and we have for any $g \in G$, $h \in H$, 
$\rmd L_h(g): \mathrm{Tan}_g G \to \mathrm{Tan}_{hg} G$.
\begin{definition}{}{}
  A function $f: G \to \rset$ is said to be \emph{$H$-invariant} if for any
  $g \in G$ and $h \in H$, $f(L_g(h)) = f(h)$. We note $g.f = f$. A section
  $F \in \Gamma(\TG)$ is said to be \emph{$H$-equivariant} if for any $h \in H$ and
  $g \in G$, $F(L_h(g)) = \rmd L_h(g) F(g)$.
  An operator
  $A: \ \rmc^\infty(G, \rset) \to \rmc^\infty(G, \rset)$ is \emph{$H$-invariant}
  if for any $h \in H$ and $f \in \rmc^\infty(G, \rset)$, $A (h.f) = f$. An operator
  $A: \ \rmc^\infty(G, \rset) \to \rmc^\infty(G, \rset)$ is \emph{$H$-equivariant}
  if for any $h \in H$ and $f \in \rmc^\infty(G, \rset)$, $A (h.f) = h.(Af)$.
\end{definition}

\begin{proposition}{}{}
  Let $G$ be a Lie group and $H$ a subgroup of $G$. Let $\bfX$ associated with
  $\rmd \bfX^{(t)} = b(t, \bfX^{(t)}) \rmd t + \Sigma^{1/2} \rmd \bfB^{(t)}$,
  with bounded coefficients, where $\bfB^{(t)}$ is a Brownian motion associated
  with a left-invariant metric. Assume that the distribution of $\bfX^{(0)}$ is
  $H$-invariant and that for any $t \geq 0$ and $h \in H$,
  $\Sigma (\rmd L_h.\nabla p_t) = \rmd L_h.(\Sigma \nabla p_t)$ and
  $b\circ L_h = \rmd L_h . b$\footnote{$b$ is said to be \emph{equivariant} with
    respect to action of $H$.} then the distribution of $\bfX^{(t)}$ is
  $H$-invariant for any $t \geq 0$.
\end{proposition}

\begin{proof}
  Denote $p_t$ the density of the distribution of $\bfX_t$ w.r.t. the Haar
  measure. Since the Haar measure is $H$-invariant by definition, we only need
  to show that $p_t$ is $H$-invariant. To do so, we show that $p_t \circ L_h$
  satisfy the same Fokker-Planck equation as $p_t$. Indeed, in that case we have
  that $(\bfX_t)_{t \geq 0}$ and $(h.\bfX_t)_{t \geq 0}$ both satisfy the same
  martingale problems and therefore are both weak solution to the SDE
  $\rmd \bfX^{(t)} = b(t, \bfX^{(t)}) \rmd t + \Sigma^{1/2} \rmd
  \bfB^{(t)}$. Since the coefficients are continuous and bounded we have
  uniqueness in the solution, see \citep[Chapter IV, Theorem
  3.3]{ikeda2014stochastic} and the distribution of $h.\bfX_t$ is the same as
  the one of $\bfX_t$ for all $h \in H$, which concludes the proof. Using
  \Cref{sec:invar-diff-proc}, we have for any $t \in \ccint{0,\Tfinal}$ and $g \in G$
  \begin{align}
    \partial_t (h.p_t)(g) &= -\mathrm{div}(b p_t)(L_h(g)) + \tfrac{1}{2} \Delta_\Sigma p_t (L_h(g)) \\
                          &= -\mathrm{div}(b p_t)(L_h(g)) + \tfrac{1}{2} h. (\Delta_\Sigma p_t) (g) \\
                          &= -\mathrm{div}(b p_t)(L_h(g)) + \tfrac{1}{2} \Delta_\Sigma (h.p_t) (g) \\
                          &= -\mathrm{div}(b)(L_h(g)) h.p_t(g) - \langle b(L_h(g)), \nabla p_t(L_h(g)) \rangle  + \tfrac{1}{2} \Delta_\Sigma (h.p_t) (g)  , \label{eq:int_FK}
  \end{align}
  We have that for any $t \in \ccint{0,\Tfinal}$ and $g \in G$
  \begin{equation}
    \rmd (h.p_t)(g) =  \rmd p_t(L_h(g)) \rmd L_h(g)  .
  \end{equation}
  Hence, for any $t \in \ccint{0,\Tfinal}$ and $g \in G$ and $u \in \mathrm{T}_{g} G$ we have
  \begin{equation}
    \langle \nabla (h.p_t)(g), u \rangle = \langle \nabla p_t(L_h(g)), \rmd L_h(g) u \rangle  . 
  \end{equation}
  Hence, using this result and that $b$ is $H$-equivariant we have for any $t \in \ccint{0,\Tfinal}$ and $g \in G$
  \begin{equation}
    \langle b(L_h(g)), \nabla p_t(L_h(g)) \rangle = \langle \rmd L_h(g) b(g), \nabla p_t(L_h(g)) \rangle = \langle b(g), \nabla (h.p_t)(g) \rangle  . \label{eq:scalar_product}
  \end{equation}
  Finally, using \Cref{lemma:invariance_div}, we have that
  $\dive(b)(L_h(g)) = \dive(b)(g)$ for any $g \in G$. Therefore, we get that for
  any $t \in \ccint{0,\Tfinal}$ and $g \in G$
  \begin{align}
    \partial_t (h.p_t)(g) &= -\mathrm{div}(b)(L_h(g)) h.p_t(g) - \langle b(L_h(g)), \nabla p_t(L_h(g)) \rangle  + \tfrac{1}{2} \Delta_\Sigma (h.p_t) (g) \\
                          &= -\mathrm{div}(b)(g) h.p_t(g) - \langle b(g), \nabla (h.p_t)(g) \rangle  + \tfrac{1}{2} \Delta_\Sigma (h.p_t) (g) \\
    &= -\mathrm{div}(b h.p_t)(g) + \tfrac{1}{2} \Delta_\Sigma (h.p_t)(g)  . 
  \end{align}
  Hence $h \cdot p_t$ satisfies the same Fokker-Planck equation as $p_t$, which concludes the proof.
\end{proof}

\begin{lemma}{}{}
  Assume that $X \in \Gamma(\TG)$ is $H$-equivariant. Then for any
  $Y \in \Gamma(\TG)$ which is $H$-equivariant $\nabla_Y X$ is $H$-equivariant.
\end{lemma}

\begin{proof}{}{}
  Let $g \in G$. We have
  $\nabla_Y X (g) = ( \rmd L_{g \gamma(t)^{-1}}(\gamma(t)) X(\gamma(t)))'(0)$,
  with $\gamma(t)$ a smooth curve such that $\gamma'(0)= Y(g)$ and
  $\gamma(0)=g$. Note that $\gamma_h(t) = L_h(\gamma(t))$ is a smooth curve such
  that $\gamma_h'(t) = Y(hg)$. As a consequence, using the equivariance of $X$, we have 
  \begin{align}
    \nabla_Y X (L_h g) &= ( \rmd L_{hg \gamma(t)^{-1}h^{-1}}(L_h(\gamma(t))) X(L_h(\gamma(t))))'(0) \\
                       &= ( \rmd L_{hg \gamma(t)^{-1}h^{-1}}(L_h(\gamma(t))) \rmd L_{h}(\gamma(t)) X(\gamma(t)))'(0) \\
                       &= (\rmd L_{hg\gamma(t)^{-1}}(\gamma(t)) X(\gamma(t)))'(0) \\
                         &= \rmd L_{h}(g) (\rmd L_{g\gamma(t)^{-1}}(\gamma(t)) X(\gamma(t)))'(0) = \rmd L_h(g) \nabla_Y X(g) ,
  \end{align}
  which concludes the proof.
\end{proof}

Using this result we have the following lemma.

\begin{lemma}{}{}
  \label{lemma:invariance_div}
  Assume that $X \in \Gamma(\TG)$ is $H$-equivariant. Then $\dive(X)$ is $H$-invariant.
\end{lemma}

We provide two proofs of this theorem.

\begin{proof}{}{}
  For the first proof, let $\{e_i\}_{i=1}^d$ be an orthonormal frame of
  $\TG$, then we have that
  \begin{equation}
    \textstyle{\dive(X) = \sum_{i=1}^d \langle \nabla_{e_i} X, e_i \rangle  . }
  \end{equation}
  Therefore, using that $\{e_i\}_{i=1}^d$ is orthonormal and that the
  $\rmd L_h(g)$ is an isometry, we have for any $g \in G$ and $h \in H$
  \begin{align}
    \dive(X)(hg) &= \textstyle{\sum_{i=1}^d \langle \nabla_{e_i} X(hg) , e_i(hg) \rangle} \\
                 &= \textstyle{\sum_{i=1}^d \langle \rmd L_h(g) \nabla_{e_i} X(g) , \rmd L_h(g) e_i(g) \rangle} \\
    &= \textstyle{\sum_{i=1}^d \langle \nabla_{e_i} X(g) , e_i(g) \rangle = \dive(X)(g)}  ,
  \end{align}
  which concludes the proof.
\end{proof}

For the second proof, we use the divergence theorem and don't rely on the fact
that the covariant derivative preserve the equivariance.

\begin{proof}{}{}
  For any test function $f \in \rmc^\infty_c(G, \rset)$ we have
  \begin{align}
    \textstyle{\int_{G} f(g) \dive(X)(hg) \rmd \mu(g)} = \textstyle{\int_{G} f(h^{-1}g) \dive(X)(g) \rmd \mu(h)}  . \label{eq:integration}
  \end{align}
  Second we have that $\rmd (f \circ L_{h^{-1}})(g) = \rmd f (h^{-1}g) \rmd L_{h^{-1}}(g)$. In particular, for any $u \in \mathrm{T}_{g}G$ we have
  \begin{equation}
    \langle \nabla  (f \circ L_{h^{-1}})(g), u \rangle = \rmd (f \circ L_{h^{-1}})(g)(u) = \rmd f (h^{-1}g) \rmd L_{h^{-1}}(g)(u) = \langle \nabla f(h^{-1}g), \rmd L_{h^{-1}}(g) u \rangle  . 
  \end{equation}
  Combining this result, \eqref{eq:integration} and the divergence theorem.
  \begin{align}
    \textstyle{\int_{G} f(g) \dive(X)(hg) \rmd \mu(g)} &\textstyle{= - \int_G \langle \nabla (f \circ L_{h^{-1}})(g), X(g) \rangle \rmd \mu(g)} \\
                                                       &=\textstyle{ -\int_{G} \langle \nabla f(h^{-1}g), \rmd L_{h^{-1}}(g) X(g) \rangle \rmd \mu(g)} \\
                                                       &= \textstyle{-\int_{G} \langle \nabla f(h^{-1}g), X(h^{-1}g) \rangle \rmd \mu(g)} \\ 
                                                         &= \textstyle{-\int_{G} \langle \nabla f(g), X(g) \rangle \rmd \mu(g) = \textstyle{\int_G f(g) \dive(X)(g) \rmd \mu(g)}  . }
  \end{align}
  Hence, we have that for any test function $f \in \rmc^\infty_c(G, \rset)$,
  $\int_G f(g)(\dive(X)(hg) - \dive(X)(g)) \rmd \mu(g) = 0$ and therefore
  $\dive(X)$ is $H$-invariant.
\end{proof}

\begin{lemma}
  \label{sec:invar-diff-proc}
  Let $f \in \rmc^\infty(G)$ such that for any $h \in H$,
  $\rmd L_h (\Sigma \nabla f) = \Sigma (\rmd L_h \nabla f)$. Then, we have that
  for any $h \in H$, $h.\Delta_\Sigma(f) = \Delta_\Sigma(h.f)$, where
  $\Delta_\Sigma(f) = \mathrm{div}(\Sigma \nabla f)$.
\end{lemma}

Note that in the case where $\Sigma = \Id$ we recover that $\Delta$ is equivariant.

\begin{proof}
  For any test function $u, v \in \rmc^\infty_c(G, \rset)$ we have
  \begin{align}
    \textstyle{\int_{G} u(g) \dive(\Sigma \nabla v)(hg) \rmd \mu(g)} = \textstyle{\int_{G} u(h^{-1}g) \dive(\Sigma \nabla)(g) \rmd \mu(g)}  , \label{eq:integration}
  \end{align}
  where $\mu$ is the (left-invariant) Haar measure on $G$.  Second we have that
  $\rmd (u \circ L_{h^{-1}})(g) = \rmd u (h^{-1}g) \rmd L_{h^{-1}}(g)$. In
  particular, for any $\xi \in \mathrm{T}_{g}G$ we have
  \begin{equation}
    \langle \nabla  (u \circ L_{h^{-1}})(g), \xi \rangle = \rmd (u \circ L_{h^{-1}})(g)(\xi) = \rmd u (h^{-1}g) \rmd L_{h^{-1}}(g)(\xi) = \langle \nabla u(h^{-1}g), \rmd L_{h^{-1}}(g) \xi \rangle  . 
  \end{equation}
  Combining this result, \eqref{eq:integration} and the divergence theorem.
  \begin{align}
    \textstyle{\int_{G} u(g) \dive(\Sigma \nabla v)(hg) \rmd \mu(g)} &\textstyle{= - \int_G \langle \nabla (u \circ L_{h^{-1}})(g), \Sigma \nabla v(g) \rangle \rmd \mu(g)} \\
                                                                     &=\textstyle{ -\int_{G} \langle \nabla u(h^{-1}g), \rmd L_{h^{-1}}(g) \Sigma \nabla v(g) \rangle \rmd \mu(g)} \\
&=\textstyle{ -\int_{G} \langle \nabla u(h^{-1}h), \Sigma \rmd L_{h^{-1}}(g)  \nabla v(g) \rangle \rmd \mu(g)} \\    
                                                       &= \textstyle{-\int_{G} \langle \nabla u(h^{-1}g), \Sigma \nabla (h.v)(h^{-1}g) \rangle \rmd \mu(g)} \\ 
                                                         &= \textstyle{-\int_{G} \langle \nabla u(g), \Sigma \nabla (h.v)(g) \rangle \rmd \mu(g) = \textstyle{\int_G u(g) \dive(\Sigma \nabla (h.v))(g) \rmd \mu(g)}  . }
  \end{align}
  Hence, we have that for any test function $u \in \rmc^\infty_c(G, \rset)$,
  $\int_G u(g)(\dive(\Sigma \nabla v)(hg) - \dive(\Sigma \nabla (h.v))(g)) \rmd
  \mu(g)=0$ and therefore $h.\Delta_\Sigma(v) = \Delta_\Sigma(h.v)$.
\end{proof}


\section{Connection between $\SO(3)$-invariant pinned probability measures and $\SE(3)$-invariant measures}
\label{sec:conn-betw-so_3rs}

In this section, we prove \Cref{sec:from-se3-so3}. We first present a result on
the disintegration of measures, see \citep[p.117]{pollard2002user}. We specify
this result
\begin{proposition}
  Let $\mu$ be a measure on $\SE(3)^N$ which can be written as a countable sum
  of finite measures, each with compact support. Then, there exist a kernel
  $\rmK: \ \rset^3 \times \mcb{\SE(3)^N} \to \rset_+$ such that
  $(\mu \otimes \rmK) = F_{\#} \mu$ with
  $F([T_1, \dots, T_n]) = ([T_1, \dots, T_n], \tfrac{1}{N}\sum_{i=1}^N x_i)$.
\end{proposition}

In what follows, we denote $M([T_1, \dots, T_n]) = \tfrac{1}{N} \sum_{i=1}^N x_i$.
We are now ready to state the following proposition.

\begin{proposition}[Disintegration of measures on $\SE(3)^N$]
  \label{sec:from-se3-so3_app}
  Let $\mu$ be a measure on $\SE(3)^N$ which can be written as a countable sum
  of finite measures, each with compact support. Assume that for any
  $f \in \rmc^\infty_c(\SE(3)^N)$,
  $x \mapsto \int_{\SE(3)^N} f([T_1, \dots, T_N]) \rmd \rmK(x, [T_1, \dots,
  T_N])$ is continuous and for any $x \in \rset^3$,
  $\rmK(x, \SE(3)^N) <+\infty$. Then, there exist $\eta$ an $\SO(3)$-invariant
  probability measure on $\SE(3)^N_0$ and $\bar{\mu}$ proportional to the
  Lebesgue measure on $\rset^3$ such that
  \begin{align}
    &\rmd \mu([(r_1,x_1), \dots, (r_N,x_N)]) \\& \qquad \qquad = \rmd \eta([(r_1, x_1-\bar{x}), \dots, (r_N, x_N-\bar{x})]) \rmd \bar{\mu}(\bar{x}) .
  \end{align}
\end{proposition}

\begin{proof}
  First, we have that $M_\# \mu$ is translation invariant since $\mu$ is
  $\SE(3)$-invariant. Since $f_{\#} \mu$ is a translation invariant measure on
  $\rset^3$, we have that $\mu$ is proportional to the Lebesgue measure, without
  of loss of generality we assume that it is equal to the Lebesgue measure in
  what follows. For any $x_0 \in \rset^3$, $f \in \rmc_c^\infty(\SE(3)^N)$ and
  $g \in \rmc_c^\infty(\rset^3)$ we have
  \begin{align}
    &\textstyle{\int_{\SE(3)^N} f([T_1, \dots, T_N]) g(M([T_1, \dots, T_N])) \rmd \mu([T_1, \dots, T_N])} = \textstyle{\int_{\rset^3} g(\bar{x}) \int_{\SE(3)^N} f([T_1, \dots, T_N]) \rmK(\bar{x}, \rmd [T_1, \dots, T_N]) \rmd \bar{x}} \\
    & \qquad \qquad \qquad  \textstyle{=  \int_{\rset^3} g(\bar{x}+x_0) \int_{\SE(3)^N} f([(R_1, x_1), \dots, (R_N, x_N)]) \rmK(\bar{x}+x_0, \rmd [T_1, \dots, T_N]) \rmd \bar{x} }\\
    & \qquad \qquad \qquad  \textstyle{=  \int_{\rset^3} g(\bar{x}+x_0) \int_{\SE(3)^N} f([(R_1, x_1+x_0), \dots, (R_N, x_N+x_0)]) \rmK(\bar{x}, \rmd [T_1, \dots, T_N]) \rmd \bar{x} } ,
  \end{align}
  where the first equality is obtained using the translation invariance of the
  Lebesgue measure and the second is obtained using the $\SE(3)$ invariance of
  $\mu$. Therefore, we obtained that for almost any $\bar{x} \in \rset^3$,
  $f \in \rmc_c^\infty(\SE(3)^N)$
  \begin{equation}
    \textstyle{\int_{\SE(3)^N} f([T_1, \dots, T_N]) \rmK(\bar{x}+x_0, \rmd [T_1, \dots, T_N]) = \int_{\SE(3)^N} f([T_1, \dots, T_N]) (t_{x_0})_{\#}\rmK(\bar{x}, \rmd [T_1, \dots, T_N]),}
  \end{equation}
  where $t_{x_0}([T_1, \dots, T_n]) = [(R_1, x_1+x_0), \dots, (R_N,
  x_N+x_0)]$. Since, for any $f \in \rmc_c^\infty$,
  $x_0 \mapsto \int_{\SE(3)^N} f([T_1, \dots, T_N]) \rmK(\bar{x}+x_0, \rmd [T_1,
  \dots, T_N])$ is continuous, we have that for any $\bar{x} \in \rset^3$,
  $f \in \rmc_c^\infty(\SE(3)^N)$
  \begin{equation}
    \textstyle{\int_{\SE(3)^N} f([T_1, \dots, T_N]) \rmK(\bar{x}+x_0, \rmd [T_1, \dots, T_N]) = \int_{\SE(3)^N} f([T_1, \dots, T_N]) (t_{x_0})_{\#}\rmK(\bar{x}, \rmd [T_1, \dots, T_N]),}
  \end{equation}
  Therefore, we get that for any $x_0 \in \rset^3$,
  $\rmK(x_0, \cdot) = (t_{x_0})_{\#} \rmK(0, \cdot)$. By definition, we have
  that $\rmK(0, \cdot)((\SE(3)_0^N)^\complementary) = 0$, i.e.\ $\rmK(0, \cdot)$
  is supported on $\SE(3)^N_0$. In what follows, we denote
  $\eta = \rmK(0, \cdot)$. We have that for any $f \in \rmc_c^\infty(\SE(3)^N)$
  \begin{equation}
    \textstyle{\int_{\SE(3)^N} f([T_1, \dots, T_N]) \rmd \mu([T_1, \dots, T_N]) = \int_{\rset^3} \int_{\SE(3)_0^N}f([T_1, \dots, T_N]) \rmd \eta([(r_1, x_1-\bar{x}), \dots, (r_N, x_N-\bar{x})]) \rmd \bar{x} }.
  \end{equation}
  For any $f \in \rmc_c^\infty(\SE(3)^N)$
  \begin{align}
    &\textstyle{\int_{\SE(3)^N} f([T_1, \dots, T_N]) \rmd \mu([T_1, \dots, T_N])} = \textstyle{\int_{\rset^3} \int_{\SE(3)_0^N}f([T_1, \dots, T_N]) \rmd \eta([(r_1, x_1-\bar{x}), \dots, (r_N, x_N-\bar{x})]) \rmd \bar{x}} \\
    &\qquad \qquad = \textstyle{\int_{\rset^3} \int_{\SE(3)_0^N}f([(r_0r_1,r_0x_1), \dots, (r_0r_N,r_0x_N)]) \rmd \eta([(r_1, x_1-\bar{x}), \dots, (r_N, x_N-\bar{x})]) \rmd \bar{x}} \\
    &\qquad \qquad = \textstyle{\int_{\rset^3} \int_{\SE(3)_0^N}f([T_1, \dots, T_N]) (r_0)_{\#} \rmd \eta([(r_1, x_1-\bar{x}), \dots, (r_N, x_N-\bar{x})]) \rmd \bar{x}}.
  \end{align}
  Therefore, $\eta$ is $\SO(3)$-invariant which concludes the proof.
\end{proof}

We also have the following proposition.

\begin{proposition}[Construction of invariant measures]
  \label{sec:from-se3-so3_app_old}
  Let $\eta$ be an $\SO(3)$-invariant probability measure on $\SE(3)^N_0$, 
  $\bar{\mu}$ the Lebesgue measure on $\rset^3$. Then
  \begin{equation}
    \rmd \eta([(r_1, x_1-\bar{x}), \dots, (r_N, x_N-\bar{x})]) \rmd \bar{\mu}(\bar{x}) ,
  \end{equation}
  is $\SE(3)$-invariant on $\SE(3)^N$.
  \end{proposition}

\section{Rodrigues' formula and differentiation}
\label{sec:rodrigues-formula}

In this section, we prove \Cref{prop:deno-score-match}.  We recall that the Lie
algebra $\so(3)$ can be described with $\omega \in \S^2$ and $\theta \in \rset$
by
\begin{equation}
  Y = \theta Y_\omega, \qquad Y_\omega = \omega_1 Y_1 + \omega_2 Y_2 + \omega Y_3 .
\end{equation}
This is the \emph{axis-angle} representation of the Lie algebra. Note that
$\| Y_\omega \|^2 = 2$, since $\omega \in \S^2$ and
$\trace(Y_i Y_j^\top) = 2 \updelta_{i,j}$. In addition, we have that
$Y_\omega^3 = -Y_\omega$ and therefore we recover Rodrigues' formula
\begin{equation}
  \exp[ \theta Y_\omega] = \Id + \sin( \theta) Y_\omega + (1 - \cos(\theta)) Y_\omega^2 .
\end{equation}
Denote $\varphi: \ooint{0, \uppi} \times \S^2 \to \SO(3)$ with $\S^2$ identified with $\ensembleLigne{a_1Y_2+ a_2Y_2 + a_3Y_3}{(a_1,a_2,a_3)\in \S^2}$ and 
\begin{equation}
  \varphi(\theta, Y_\omega) = \Id + \sin(\theta) Y_\omega + (1-\cos(\theta)) Y_\omega^2 .
\end{equation}
Note that $\varphi$ is injective, we denote $\mathrm{Im}(\varphi)$ its image,  and is inverse is given by
\begin{align}
  &\varphi^{-1}(R)_1 = \theta = \cos^{-1}((\trace(R)-1)/2) , \\
  &\varphi^{-1}(R)_2 =Y_\omega =
  ((R_{32}-R_{23})Y_1 + (R_{13} - R_{31}) Y_2 + (R_{21} - R_{12}) Y_3)/ (2\sin(\theta)) . 
\end{align}
We have the following proposition.

\begin{proposition}
  For any $R \in \mathrm{Im}(\varphi)$, we have
  \begin{equation}
    \nabla \varphi^{-1}(R)_1 = R \exp^{-1}(R) / \exp^{-1}(R)_1.
  \end{equation}
\end{proposition}

\begin{proof}
  First, note that $(RY_1, RY_2, RY_3)$ is an
  orthonormal basis for $\mathrm{Tan}_R\SO(3)$.  Consider $R_t = R \exp[t Y_1]$. We have that $\varphi_1^{-1}(R_t)'(0) = (\nabla \varphi^{-1}(R))_1$.
  Let $\vareps > 0$ such that for any $t \in \ccint{-\vareps, \vareps}$,
  $R_t \in \mathrm{Im}(\varphi)$. We have that for any $t \in \ccint{-\vareps, \vareps}$
  \begin{equation}
    \varphi^{-1}(R_t)_1' = -(1-((\trace(R)-1)/2)^2)^{1/2} \trace(R Y_1)/2 = -\trace(R Y_1)/(2\sin(\theta)).
  \end{equation}
  Using that $\trace(R Y_1) = -R_{32} +R_{23}$ we get that
  \begin{equation}
    \varphi^{-1}(R_t)_1' =  (R_{32} - R_{23})/(2\sin(\theta)) ,
  \end{equation}
  Hence, we have
    \begin{equation}
    \nabla \varphi^{-1}(R)_1 = R \varphi^{-1}(R)_2 = R \varphi^{-1}(R)_1 \varphi^{-1}(R)_2 / \varphi^{-1}(R)_1.
  \end{equation}
    Note that identifying $\rset^3$ and $\rset_+ \times \S_{\so(3)}$ we
  have the identification
  \begin{equation}
    \varphi^{-1}(R) = \varphi^{-1}(R)_1 \varphi^{-1}(R)_2,
  \end{equation}
  which concludes the proof.
\end{proof}


Finally, we have the following proposition
\begin{proposition}
\label{prop:diff_angle}
    For almost any $R,R' \in \SO(3)$ we have
  \begin{equation}
    \nabla \varphi^{-1}(R'^\top R)_1 = R \exp^{-1}(R'^\top R) / \exp^{-1}(R'^\top R)_1.
  \end{equation}
\end{proposition}

\begin{proof}
  Let $H_1 = R Y_1$, $f(R) = \varphi^{-1}(R'^\top R)$ and
  $g(R) = \varphi^{-1}(R)$ defined for almost all $R \in \SO(3)$. We have that
  $f = g \circ L_{R'^\top}$. Therefore, we have that for almost any $R \in \SO(3)$
  \begin{align}
    \rmd f(R)(H_1) &= \rmd g(R'^\top R) (\rmd L_{R'^\top}(R)(H_1)) \\
                   &=  \rmd g(R'^\top R) (R'^\top H_1) = \langle \nabla g(R'^\top R), R'^\top H_1 \rangle = \langle R' \nabla g(R'^\top R), H_1 \rangle ,
  \end{align}
  which concludes the proof.
\end{proof}

The proof of \Cref{prop:deno-score-match} is a direct consequence of \Cref{prop:diff_angle}.


\section{Additional method details}
\label{sec:additional_methods}

\subsection{Frame to coordinates}
\label{sec:ideal_coordinates}


We continue from \cref{sec:param_proteins} in describing backbone atom parameterization in terms of frames.
As discussed, $\nitrogen^\star,\carbon_\alpha^\star, \carbon^\star, \oxygen^*$ are idealized atom coordinates that assumes chemically idealized bond angles and lengths. 
AF2 derived these coordinates from \citet{engh2012structure}.
However, these values differ slightly per amino acid type.
Since we do not model sequence, we take the idealized values of Alanine which are,
\begin{align}
    \nitrogen^\star &= (-0.525, 1.363, 0.0) \\
    \carbon_\alpha^\star &= (0.0, 0.0, 0.0) \\
    \carbon^\star &= (1.526, 0.0, 0.0) \\
    \oxygen^* &= (0.627, 1.062, 0.0)
\end{align}
note the idealized values are taken with respect to $\carbon_\alpha^\star$ as the origin. Using a central frame $T_n$, we may manipulate idealized coordinates to construct backbone atoms for residue $n$ via \cref{eq:backboneatoms}.

The backbone oxygen requires rotating a idealized oxygen around the $\carbon - \calpha$ bond.
\begin{equation}
\label{eq:ideal_oxygen}
\oxygen_n = T_n \cdot T^\star_{\mathrm{psi}}(\psi_n) \cdot \oxygen^\star \nonumber.
\end{equation}
where $\psi_n \in \mathrm{SO}(2)$ denotes a backbone torsion angle of residue $n$ and
$T^\star_{\mathrm{psi}}(\psi_\res) = (R_x(\psi_\res), x_\mathrm{psi})$ is a Euclidean transformation from the
central frame $T_n$ to a new frame $T_\res \cdot T_\mathrm{psi}^*$ centered at $\carbon$ and rotated around the x-axis by $\psi_\res$. Recall $\psi_n$ is a tuple of of two values specifying a point along the unit circle, $\psi_n = [\psi_{n,1}, \psi_{n,2}]$ where $(\psi_{n,1})^2 + (\psi_{n,2})^2 = 1$.
\begin{equation}
\label{eq:oxygen}
\begin{aligned}
    R_x(\psi) &= \begin{pmatrix}
        1 & 0 & 0 \\
        0 & \psi_{\res,1} & -\psi_{\res,2} \\
        0 & \psi_{\res,2} & \psi_{\res,1}
    \end{pmatrix} \\
    x_\mathrm{psi} &= (1.526, 0.0, 0.0)
\end{aligned}
\end{equation}
The mapping from frames to idealized coordinates, frame2atom, is achieved with \cref{eq:backboneatoms,eq:oxygen}:
\begin{equation}
\label{eq:frame2atom}
[\nitrogen_n, \carbon_n, (\calpha)_n, \oxygen_n] = \text{frame2atom}(T_n, \psi_n).
\end{equation}

We next describe constructing frames from coordinates. Each residue's frames are obtained as described in \cref{fig:method_overview}A and the rigidFrom3Point algorithm in AF2,
\begin{equation}
\begin{aligned}
    &v_1 = \carbon_\res - (\calpha)_\res, \qquad v_2 = \nitrogen_\res - (\calpha)_\res \\
    &e_1 = v_1 / \|v_2\|, \qquad u_2 = v_2 - e_1(e_1^Tv_2) \\
    &e_2 = u_2 / \|u_2\| \\
    &e_3 = e_1 \times e_2 \\
    &R_\res = \mathrm{concat}(e_1, e_2, e_3)\\
    &x_\res = (\calpha)_\res \\
    &T_\res = (R_\res, x_\res)
\end{aligned}
\end{equation}
where the first four lines follow from Gram-Schmidt. 
The operation of going from coordinates to frames is called atom2frame,
\begin{equation}
    \label{eq:atom2frame}
    T_n = \text{atom2frame}(\nitrogen_n, \carbon_n, (\calpha)_n).
\end{equation}
atom2frame will be used in \cref{sec:architecture} when constructing initial frames from the data.
We do not need to construct $\psi_n$ since in \cref{sec:training} our losses are directly on the coordinates themselves. However, it can easily be solved by solving a least squares with \cref{eq:ideal_oxygen}.

\subsection{\framepred{} architecture}
\label{sec:architecture}

Here we provide mathematical detail of \framepred{} presented in \cref{sec:framepred}.
To recap, $\bfh_\ell = [h_\ell^1, \dots, h_\ell^\numres] \in
\mathbb{R}^{\numres\times D_h}$ are the node embeddings of the $\ell$-th layer
where $h_\ell^\res$ is the embedding for residue $\res$;
$\bfz_\ell \in \mathbb{R}^{\numres \times \numres \times D_z}$ are edge
embeddings with $z_\ell^{\res m}$ being the embedding of the edge between
residues $\res$ and $m$.  The frames of every residue at the $\ell$-th layer is
denoted $\bfT_\ell \in \SE(3)^\numres$.
Unless stated otherwise, all instances of Multi-Layer Perceptrons (MLP) use 3 Linear layers with biases, ReLU activation, and LayerNorm \citep{ba2016layer} after the final layer.
\emph{In this section, superscripts without parentheses are used to refer to residue indices, superscript numbers within parentheses refer to time step, subscripts refer to variable names.}

\paragraph{Feature initialization.} Following \citet{trippe2022diffusion}, node embeddings are initialized with residue indices and timestep while edge embeddings additionally get relative sequence distances. Initial embeddings at layer 0 for residues $\res, m$ are obtained with an MLP and sinusoidal embeddings $\phi(\cdot)$ \citep{vaswani2017transformer} over the features.

We additionally include self-conditioning of predicted $\calpha$ displacements.
Let $\hat{x}_{sc}$ be the $\calpha$ coordinates (in \AA) predicted during self-conditioning.
50\% of the time we set $\hat{x}_{sc}=0$. 
The binned displacement of two $\alpha$ is given as,
\begin{equation}
\textstyle{
    \mathrm{disp}_{sc}^{nm} = \sum_{i=1}^{N_\mathrm{bins}} \mathds{1}\left\{|\hat{x}_{sc}^n - \hat{x}_{sc}^m| < \nu_i\right\}
    }
\end{equation}
where $\nu_1,\dots, \nu_{N_\mathrm{bins}} = \mathrm{linspace}(0, 20)$ are equally spaced bins between 0 and 20 angstroms.
In our experiments we set $N_\mathrm{bins}=22$.
The initial embeddings can be expressed as
\begin{align}
    &&h_0^\res = \mathrm{MLP}([\phi(n), \phi(t)]) & && h_0^\res\in \mathbb{R}^{D_h}\\
    &&z_0^{\res m} = \mathrm{MLP}([\phi(\res), \phi(m), \phi(m - \res), \phi(t), \phi(\mathrm{disp}_{sc}^{nm})]) & && z_0^{\res m}\in \mathbb{R}^{D_z}
\end{align}
where $D_h, D_z$ are node and edge embedding dimensions.

To construct the initial frames, $\calpha$ coordinates are first zero-centered and all backbone coordinates ($\nitrogen, \carbon, \calpha, \oxygen$) are scaled to nanometers as done in AF2 by multiplying coordinates by 1/10.
We then construct initial frames for each residue $n$ with \cref{eq:atom2frame},
\begin{equation}
    T^{(0),n} = (R^{(0),n}, x^{(0),n}) = \mathrm{atom2frame}(\nitrogen^n, \carbon^n, \calpha^n)
\end{equation}
During training, initial frames are then sampled $\bfT^{(t)}_0 \sim p_{t|0}(\cdot|\bfT^{(0),n})$.
We now write out the neural network described in \cref{fig:framediff_block}. 
Starting at layer $\ell=0$, we iteratively update node embeddings, edge embeddings, and frames.

\paragraph{Node update.} Invariant Point Attention (IPA) was introduced in \citep{Jumper2021HighlyAP}. We apply it without modifications. No weight sharing is performed across layers. Transformer is used without modification from \citep{vaswani2017transformer}. Hyperparameters for Transformer and IPA are given in \cref{sec:hyperparameters}.
\begin{align}
    &&\bfh_{\mathrm{ipa}} = \mathrm{LayerNorm}(\mathrm{IPA}(\bfh_\ell, \bfz_\ell, \bfT_\ell) + \bfh_\ell) & &&\bfh_{\mathrm{ipa}} \in \mathbb{R}^{\numres, D_h}\\
    &&\bfh_{\mathrm{skip}} = \mathrm{Linear}(\bfh_0) & && \bfh_{\mathrm{skip}} \in \mathbb{R}^{\numres,D_{\mathrm{skip}}} \\
    &&\bfh_{\mathrm{in}} = \mathrm{concat}(\bfh_{\mathrm{ipa}}, \bfh_{\mathrm{skip}})  & && \bfh_{\mathrm{in}} \in \mathbb{R}^{\numres,(D_{\mathrm{skip}}+D_h)}\\
    &&\bfh_{\mathrm{trans}} = \mathrm{Transformer}(\bfh_{\mathrm{in}})  & && \bfh_{\mathrm{trans}} \in \mathbb{R}^{\numres,(D_{\mathrm{skip}}+D_h)}\\
    &&\bfh_{\mathrm{out}} = \mathrm{Linear}(\bfh_{\mathrm{trans}}) + \bfh_\ell & && \bfh_{\mathrm{out}} \in \mathbb{R}^{\numres,D_h} \\
    &&\bfh_{\ell+1} = \mathbf{MLP}(\bfh_{\mathrm{out}}) & && \bfh_\ell+1 \in \mathbb{R}^{\numres,D_h}
\end{align}

\paragraph{Edge update.} Each edge is updated with a MLP over the current edge and source and target node embeddings.
\begin{align}
    &&\bfh_{\mathrm{down}} = \mathrm{Linear}(\bfh_{\ell+1}) & && \bfh_{\mathrm{down}} \in \mathbb{R}^{\numres,D_h/2} \\
    &&z_{\mathrm{in}}^{nm} = \mathrm{concat}(h_\mathrm{down}^n, h_\mathrm{down}^m, z_{\ell}^{nm}) & && z_{\mathrm{in}}^{nm} \in \mathbb{R}^{\numres,(D_h+D_z)} \\
    &&\bfz_{\ell+1} = \mathrm{LayerNorm}(\mathrm{MLP}(\bfz_{\mathrm{in}})) & && \bfz_{\ell+1} \in \mathbb{R}^{\numres,\numres,D_z}
\end{align}
In the first line, node embeddings are first projected down to half the dimension.

\paragraph{Backbone update.} Our frame updates follow the BackboneUpdate algorithm in AF2.
We write the algorithm here with our notation,
\begin{align}
    b^n, c^n, d^n, x_{\mathrm{update}}^n &= \mathrm{Linear}(h_\ell) \\
    (a^n, b^n, c^n, d^n) &= (1, b^n, c^n, d^n) / \sqrt{1 + b^n + c^n + d^n}  \label{eq:quat_update} \\
    R^n_{\mathrm{update}} &= \begin{psmallmatrix}
        (a^n)^2 + (b^n)^2 - (c^n)^2 - (d^n)^2 & 2b^nc^n - 2a^nd^n & 2b^nd^n + 2a^nc^n \\
        2b^nc^n + 2a^nd^n & (a^n)^2 - (b^n)^2 + (c^n)^2 - (d^n)^2 & 2c^nd^n - 2a^nb^n \\
        2b^nd^n - 2a^nc^n & 2c^nd^n - 2a^nb^n & (a^n)^2 - (b^n)^2 - (c^n)^2 + (d^n)^2
    \end{psmallmatrix} \label{eq:rot_update} \\
    T^n_{\mathrm{update}} &= (R^n_{\mathrm{update}}, x_{\mathrm{update}}^n) \nonumber \\
    T^n_{\ell+1} &= T^n_{\ell} \cdot T^n_{\mathrm{update}}. \nonumber
\end{align}
where $b^n, c^n, d^n \in \mathbb{R}$, $x_{\mathrm{update}}^n \in \mathbb{R}^3$.
\cref{eq:quat_update} constructs a normalized quaternion which is then converted into a valid rotation matrix in \cref{eq:rot_update}.

\paragraph{Frame and score prediction.} After $L$ layers, we take the final frame as the predicted frame, $\bfT_L = \hat{\bfT}^{(0)} = (\hat{\bfR}^{(0)}, \hat{\bfx}^{(0)})$. From this we construct the score for residue $\res$ (denoted with a subscript) as,
\begin{align}
    s_\theta^\rmx(t, \bfT^{(t)})_n &= \nabla_{\dt{x_n}{t}}\log p_{t|0}(\dt{x_n}{t} | \dt{\pred{x}_{n}}{0}) \\
    &= -\frac{x^{(t)}_n -  e^{-\frac{1}{2}\beta(t)} x^{(0)}_n}{1 - e^{\beta(t)}} \\
    s_\theta^\rmr(t, \bfT^{(t)})_n &= \nabla_{\dt{R_n}{t}}  \log p_{t|0}(\dt{R_n}{t} | \dt{\pred{R}_{n}}{0}) \\
    &= \tfrac{\dt{R}{t}_n}{
\omega(\dt{\pred{R}_{n}}{0}))
} \log\{\hat{R}^{(0, t)}_n\} 
\partial_\omega f(\omega(\dt{\pred{R}_{n}}{0})), t) \label{eq:rotscore_pred}
\end{align}

\paragraph{TorsionPrediction.} Predicting torsion angle $\psi$ follows AF2.
\begin{align}
    &&\bfh_{\mathrm{psi}} = \mathrm{MLP}(\bfh_L) & && \bfh_{\mathrm{psi}} \in \mathbb{R}^{\numres,D_h} \\
    &&\pmb{\psi}_{\mathrm{unnormalized}} = \mathrm{Linear}(\bfh_{\mathrm{psi}} + \bfh_L) & && z_{\mathrm{in}}^{nm} \in \mathbb{R}^{\numres,2} \\
    &&\pmb{\hat{\psi}} = \pmb{\psi}_{\mathrm{unnormalized}} / \|\pmb{\psi}_{\mathrm{unnormalized}}\| & &&  \pmb{\hat{\psi}} \in \mathrm{SO}(2)^N
\end{align}


\subsection{Diffusion schedule and reduced noise sampling}\label{sec:diffusion_schedule}
For simplicity of exposition, the main text presents the forward diffusion process on $\SE(3)$ as evolving as
\begin{equation}
\rmd \dt{\bfT}{t} = [0, -\tfrac{1}{2}\dt{\bfX}{t}]\rmd t + [\rmd \dt{\bfB_{\SO(3)}}{t}, \rmd \dt{\bfB_{\R^3}}{t}],
\end{equation}
and reaching sufficiently close to the invariant distribution by some time $T>0.$
However, for the purpose of implementation it preferable
to consider the diffusion as approaching the invariant distribution by $t=1$ (taking $T=1$),
and to decouple the rates of diffusion of $\dt{\bfX}{t}$ and $\dt{\bfX}{t}.$
To accomplish this, we introduce drift and diffusion coefficients, $f(\cdot)$ and $g(\cdot)$ respectively,
which we define separately for the rotations and translations.
For the translations we write
$$
\rmd \dt{\bfX}{s} = f_x(s)\dt{\bfX}{s}\rmd t +  g_x(s)\rmd \dt{\bfB_{\R^3}}{s}
$$
where $f_x(s)=-\frac{1}{2}\beta(s)$
and $g_x(s)=\sqrt{\beta(s)},$ for some schedule $\beta(\cdot).$
We choose 
$$
\beta(s) = \beta_\mathrm{min} + t(\beta_\mathrm{max} - \beta_\mathrm{min}),
$$
which is the linear beta schedule introduced by \citet{ho2020denoising} adapted to the SDE setting \citet{song2020score}.

This may be seen as a time-rescaled OU process.
Letting $G_x(s) = \int_0^s g_x(t)^2= t\cdot \beta_{\mathrm{min}} + \frac{1}{2}t^2(\beta_{\mathrm{max}} -  \beta_{\mathrm{min}}),$
we have that $p_{s|0}(\dt{\bfX}{s}|\dt{\bfX}{0})=\mathcal{N}(\dt{\bfX}{s};\exp^{-G(s)}\dt{\bfX}{s}, 1- \exp^{-G_x(s)}\Id_3).$

Similarly, for rotations we have
$$
\rmd \dt{\bfR}{s} =  g_r(s)\rmd \dt{\bfB_{\SO(3)}}{s}.
$$
We relate $g_r(s)$ to a time rescaling $\sigma^2_r(s)=\int_0^s g(t)^2dt$ so that we may write
$$
p_{s|0}(\dt{\bfR}{s}| \dt{\bfR}{0})=\IGSO_3(\dt{\bfR}{s}; \dt{\bfR}{0}, \sigma^2_r(s)^2).
$$
We found it easier to choose $g_r(s)$ implicitly through the choice of the time rescaling $\sigma_r(s)$. 
In particular, we first defined $\sigma_r(s) = \log(s \cdot \exp\{\sigma_{\mathrm{max}}\} + (1-s)\exp\{\sigma_\mathrm{min}\})$
This schedule is depicted in \Cref{fig:var_schedules} Right.
This choice of $\sigma_r$ coincides with the diffusion coefficient
$g_r(s)=\sqrt{\frac{d}{ds}\sigma^2(s)}.$
We choose $\sigma^2_\mathrm{min}=0.01$ and $\sigma^2_\mathrm{max}=2.25$.

The forward processes above imply a time reversals $\dt{\overleftarrow{\bfT}}{s} \overset{d}{=} \dt{\bfT}{1-s}$ is given by 

\begin{equation}\label{eqn:time_reversal_with_drift_and_diffusion}
\dt{\overleftarrow{\bfT}}{s}  = 
\begin{bmatrix}
g_r(1-s)^2  \nabla_r \log p_{1-s}(\dt{\overleftarrow{\bfT}}{s})\\
g_x(1-s)^2 \nabla_x \log p_{1-s}(\dt{\overleftarrow{\bfT}}{s} ) - f_x(1-s)\dt{\overleftarrow{\bfX}}{s}
\end{bmatrix}\rmd t 
+ 
\begin{bmatrix}
g_r(1-s) \dt{\bfB_{\SO(3)}}{s}\\
g_x(1-s) \dt{\bfB_{\R^3}}{s}
\end{bmatrix}.
\end{equation}


Both schedules are plotted as a function of $t$ in \cref{fig:var_schedules} using hyperparameters in \cref{sec:hyperparameters}.
We additionally plot the rotation schedule when a linear $\sigma(t)=\sigma_{\mathrm{min}} + (\sigma_{\mathrm{max}} -\sigma_{\mathrm{min}})^2$ is used.
The variance is decay slower when a logarithmic schedule is used.
We found this led to slightly improved samples.

\begin{figure}[!ht]
\begin{center}
\centerline{\includegraphics[width=0.7\textwidth]{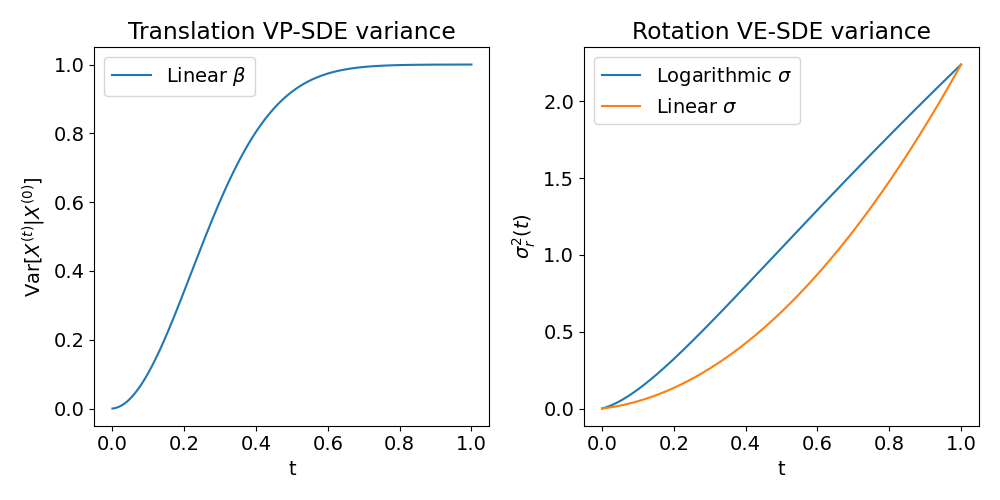}}
\caption{
Variances schedules for translations and rotations using hyperparameters in \cref{sec:hyperparameters}.
For rotations, we use a logarithmic $\sigma$ such that the variance decays slower and more closely matches the translation variance schedule.
}
\label{fig:var_schedules}
\end{center}
\vskip -0.3in
\end{figure}
\newpage

\paragraph{Noise scaling}
To generate samples with the noise-rescaling, we include additional factor on the diffusion coefficients applied to the noise when simulating \eqref{eqn:time_reversal_with_drift_and_diffusion}.

\begin{equation}\label{eqn:time_reversal_noise_rescaling}
\dt{\overleftarrow{\bfT}}{s}  = 
\begin{bmatrix}
g_r(1-s)^2  \nabla_r \log p_{1-s}(\dt{\overleftarrow{\bfT}}{s})\\
g_x(1-s)^2 \nabla_x \log p_{1-s}(\dt{\overleftarrow{\bfT}}{s} ) - f_x(1-s)\dt{\overleftarrow{\bfX}}{s}
\end{bmatrix}dt
+ 
\zeta \begin{bmatrix}
g_r(1-s) \dt{\rmd \bfB_{\SO(3)}}{s}\\
g_x(1-s) \dt{\rmd \bfB_{\R^3}}{s}
\end{bmatrix},
\end{equation}
where $\zeta \in [0, 1]$ is the `noise-scale'.
Notably, when $\zeta\ne 1,$ 
$\dt{\overleftarrow{\bfT}}{s} \overset{d}{\neq} \dt{\overleftarrow{\bfT}}{1-s}.$


\subsection{Hyperparameters}
\label{sec:hyperparameters}

\paragraph{Neural network hyperparameters.}
\begin{center}
\begin{tabular}{c c c c c}
 Global parameters: &$D_h=256$ & $D_z=128$ & $D_{\mathrm{skip}}=64$ & $L=4$ \\ 
 IPA parameters: & heads=8 & query points=8 & value points= 12 & \\
 Transformer parameters: & heads=4 & layers=2 & 
\end{tabular}
\end{center}
With these parameters, our neural network has 17446190 trainable weights.

\paragraph{SDE parameters.}
\begin{center}
\begin{tabular}{c c c c}
Translations: & schedule=linear & $\beta_{\mathrm{min}} = 0.1$ & $\beta_{\mathrm{max}}=20$ \\ 
Rotations: & schedule=logarithmic & $\sigma_{\mathrm{min}} = 0.1$ & $\sigma_\mathrm{max}=1.5$
\end{tabular}
\end{center}

\subsection{Connection to DiffAb and RFdiffusion rotation loss}
\label{sec:loss_comparison}
We briefly compare two different possible losses for learning rotations. The first is a straightforward Frobenius norm loss, $\mathcal{L}_{F}$ on rotation matrices used in both RFdiffusion \citep{watson2022rfdiffusion} and DiffAb \citep{luo2022antigen}.
\begin{equation}
\label{eq:frobenius_loss}
\textstyle{
\mathcal{L}_F(\theta) = \E [ \normLigne{ \dt{\bfR}{0} -
  \dt{\hat{\bfR}}{0} }^2],
}
\end{equation}
Our work utilizes the denoising score matching loss (DSM) \cref{eq:denoising_sm} as discussed in \cref{sec:dsm_SE3}. We copy it here for rotations,
\begin{equation}
\textstyle{
\mathcal{L}_{DSM}(\theta) = \E [ \lambda_t \normLigne{ \nabla \log p_{t|0}(\dt{\bfR}{t}|\dt{\bfR}{0}) -
  s_\theta^\rmr(t, \dt{\bfX}{t}) }^2],
}
\end{equation}
We now discuss the difference between $\mathcal{L}_F$ and $\mathcal{L}_{DSM}$.
By definition, the minimizer of $\mathcal{L}_F$ recovers the true rotation $\dt{\bfR}{0}$ while the minimizer of $\mathcal{L}_{DSM}$ is the score $s_\theta^\rmr$.
It is crucial to observe these objects are defined in different spaces: $s_\theta^\rmr$ is an element of the tangent space, while $\dt{\bfR}{0}$ is an element of SO(3).
However, if one has access to $\dt{\bfR}{0}$ then $s_\theta^\rmr$ is perfectly recoverable as seen in \cref{eq:rotscore_pred}.
In practice, we can only approximate $\dt{\bfR}{0}$ with deep learning.
Hence, $\mathcal{L}_F$ and $\mathcal{L}_{DSM}$ are likely to learn different objects except under certain settings.
We perform an ablation of using $\mathcal{L}_F$ in \Cref{table:ablations} where we see it results in a slight reduction in designability.
Due to its compatibility with the theory of score-based generative models, the DSM loss is more appealing.

\section{Additional experiment details and results}
\label{sec:additional_experiments}

\subsection{Training details}
\label{sec:training}

In this section we provide details on training \framediff{} in our experiments \cref{sec:experiment}.

\paragraph{Training data.}
We train \framediff{} over monomers\footnote{Oligomeric state is determined by the metadata in the mmcif file.} between length 60 and 512 with resolution $<5$\AA{} downloaded from PDB \citep{berman2000protein} on August 8, 2021.
This resulted in 23913 proteins.
We further filtered the data by only including proteins with high secondary structure compositions.
For each monomer, we ran DSSP \citep{kabsch1983dictionary} then removed monomers with more than 50\% loops -- resulting in 20312 proteins.
We found removing such proteins improved training and sample quality.
Extending our method to larger proteins and multimers is a direction of future research.

\paragraph{Batched training.} 
Since \framediff{} operates on fully connected graphs, the memory requirement scales quadratically.
Our implementation is based on OpenFold \citep{ahdritz2022openfold} which is not compatible with efficient batching strategies implemented in PyTorch Geometric \citep{Fey/Lenssen/2019}.
We instead adopt a simple batching strategy in \cref{alg:timebatch}.
Each element in a batch is a different timestep of the same protein backbone.
Each batch is therefore a collection of different diffused instances of the same backbone.
This way each element is the same length and no masking is required.
We set a threshold $N_{\mathrm{MaxEdges}}$ to be the maximum number of edges in each batch.
New diffused instances are added to the batch until this threshold is reached.
In our experiments, we set $N_{\mathrm{MaxEdges}} = 1000000$.

\paragraph{Optimization.} We use Adam optimizer \citep{kingma2014adam} during training with learning rate 0.0001, $\beta_1 = 0.9$, $\beta_2 = 0.999$.
Our network was trained over a period of 2 weeks on two A100 Nvidia GPUs.

\begin{minipage}{0.46\textwidth}
\begin{algorithm}[H]
    \centering
    \caption{TimestepBatch}
    \label{alg:timebatch}
    \begin{algorithmic}[1]
        \onehalfspacing
        \REQUIRE $\bfT, \varepsilon, N_\mathrm{MaxRes}$
        \STATE $\zeta = 0$
        \STATE \# Initialize batch
        \STATE $\bar{\bfT} = []$ 
        \STATE $\bar{\mathbf{t}} = []$ 
        \STATE $\{\dt{T_n}{0}\}_{n=1}^N = \bfT$
        \WHILE{$\zeta < N_\mathrm{MaxRes}$}
            \STATE \# Sample time
            \STATE $t \sim \mathcal{U}([\varepsilon, 1])$ 
            \STATE \# Apply forward diffusion
            \FOR{$n = 1,\dots,\numres $} 
                \STATE $(\dt{R_n}{0}, \dt{X_n}{0}) = \dt{T_n}{0}$
                \STATE $\dt{{X}_n}{t} \sim \mathcal{N}(\dt{X_n}{0} e^{-t/2}, (1 - e^{-t})\Id_3)$
                \STATE $\dt{R_n}{t} \sim \IGSO_3(\dt{R_n}{0},t)$
                \STATE $\dt{T_n}{t} = (\dt{R_n}{t},  \dt{{X}_n}{t})$
            \ENDFOR
             \STATE$\bfT^{(t)} = \{(\dt{R_n}{t}, \dt{X_n}{t}) \}_{n=1}^N$
            \STATE \# Remove CoM
            \STATE $\bfT^{(t)} = \Pin (\bfT^{(t)})$
            \STATE \# Append to batch
            \STATE  $\bar{\bfT}\mathrm{.append}(\bfT^{(t)})$
            \STATE \# Append time step
            \STATE  $\bar{\mathbf{t}}\mathrm{.append}(t)$ 
            \STATE \# Increase residue count
            \STATE $\zeta = \zeta + N^2$
        \ENDWHILE
        \STATE \textbf{Return} $(\bar{\bfT}, \bar{\mathbf{t}})$
\end{algorithmic}
\end{algorithm}
\end{minipage}
\hfill
\begin{minipage}{0.46\textwidth}
\begin{algorithm}[H]
    \centering
    \caption{Training.}
    \label{alg:train}
    \begin{algorithmic}[1]
        \onehalfspacing
        \REQUIRE $p_0, \varepsilon, N_\mathrm{MaxRes}, \theta$
        \WHILE{not converged}
        \STATE \# Sample data point
        \STATE $\bfT^{(0)} \sim p_0$
        \STATE \# Sample batch over time steps
        \STATE $(\bar{\bfT}, \bar{\mathbf{t}}) = \mathrm{TimestepBatch}(\bfT, \varepsilon, N_\mathrm{MaxRes})$
        \STATE \# Optimize weights $\theta$ with loss $\mathcal{L}$ over batch
        \STATE $\theta = \verb|optimizer|(\theta, \bar{\bfT}^{(t)}, \bar{\mathbf{t}}, \mathcal{L})$
        \ENDWHILE
        \STATE \textbf{return} $\theta$ 
\end{algorithmic}
\end{algorithm}
\end{minipage}

\subsection{Sampling details}
\label{sec:sampling}
\Cref{alg:sampling} outlines sampling from \framediff.  For an $N$-residue
backbone, frames $\bfT^{(1)}\in \SE(3)^N$ are first initialized from the
reference distribution, $\piinv^N$.
Starting at time $t=\Tfinal$, we run discretized Langevin dynamics with
\framediff{} with a step size of $\delta t=\gamma$ to predict the next frame at time $(t - \gamma).$
At each step, intermediate frames are always re-centered in line 12. Once the diffusion is time-reversed to $t=\epsilon$,
we perform a final forward pass of \framediff{} on lines 14 with
$t=0$.
This final output is used to construct idealized
backbone atom coordinates via frame2atom \cref{eq:frame2atom}.



\subsection{Designability}
\label{sec:evaluation}

\begin{figure*}[!ht]
\begin{center}
\centerline{\includegraphics[width=\textwidth]{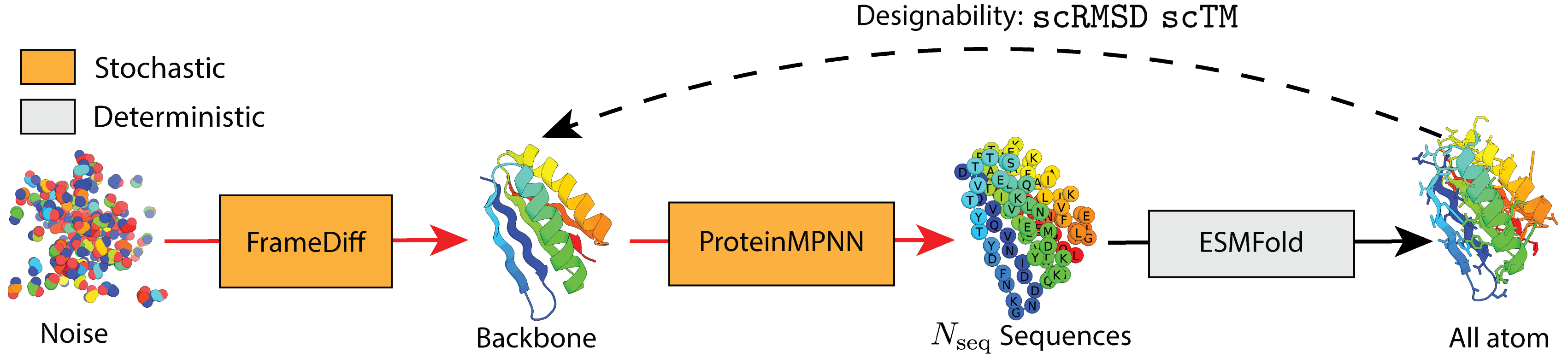}}
\caption{Designability test. Using \framediff{}, we sample a backbone starting from noise then proceed to sample multiple ($\numseqs$) sequences with \proteinmpnn \citep{dauparas2022protmpnn}.
Each sequence is then folded with ESMFold \citep{lin2022evolutionary} to obtain the \emph{predicted} backbone which is scored again the \emph{sampled} backbone with RMSD (\scrmsd) or TM-score (\sctm).
This framework also gives a method for generating a full protein with sequence and sidechains starting from a generated backbone.
}
\label{fig:designability}
\end{center}
\vskip -0.3in
\end{figure*}

\subsection{Additional results}
\label{sec:additional_results}
In this section, we provide additional results to supplement the main text. 
In \cref{sec:unconditional_results} and \cref{fig:monomer_results}, our main results are based samples obtained using on the hyperparameters $\zeta=0.1$, $N_\mathrm{steps}=500$, $N_\mathrm{seq}=100$.
We perform additional analysis on these backbone samples across lengths 100 to 500.

On the left plot of \cref{fig:additional_results}, we see \framediff{} can generate backbones up to length 500 that are well designable according to the \sctm{}$>0.5$ criterion.
We see it is more difficult to achieve designability according to the more stringent \scrmsd{}$<2$ criterion.
Reliably achieving designability of \scrmsd{}$<2$ past length 400 is only reported by RFdiffusion.
Improving \framediff's \scrmsd{} designability past length 400 is a direction of further research.

The right plot of \cref{fig:additional_results} depicts the secondary structure composition of \framediff{} samples across lengths.
We observe a wide range of secondary compositions from different helical and sheet percentages with a preference to sample more helical backbones.
More so, we notice longer backbones past 400 tend to be mostly helical.
The wide range of secondary structures and folds are visualized in random samples from \cref{fig:framediff_samples}.
Note the operating characteristics of this plot insists our samples always have loop composition $<$50\% due to the filtering of training data in \cref{sec:training}.

\begin{figure*}[!ht]
\begin{center}
\centering
\centerline{\includegraphics[width=0.9\textwidth]{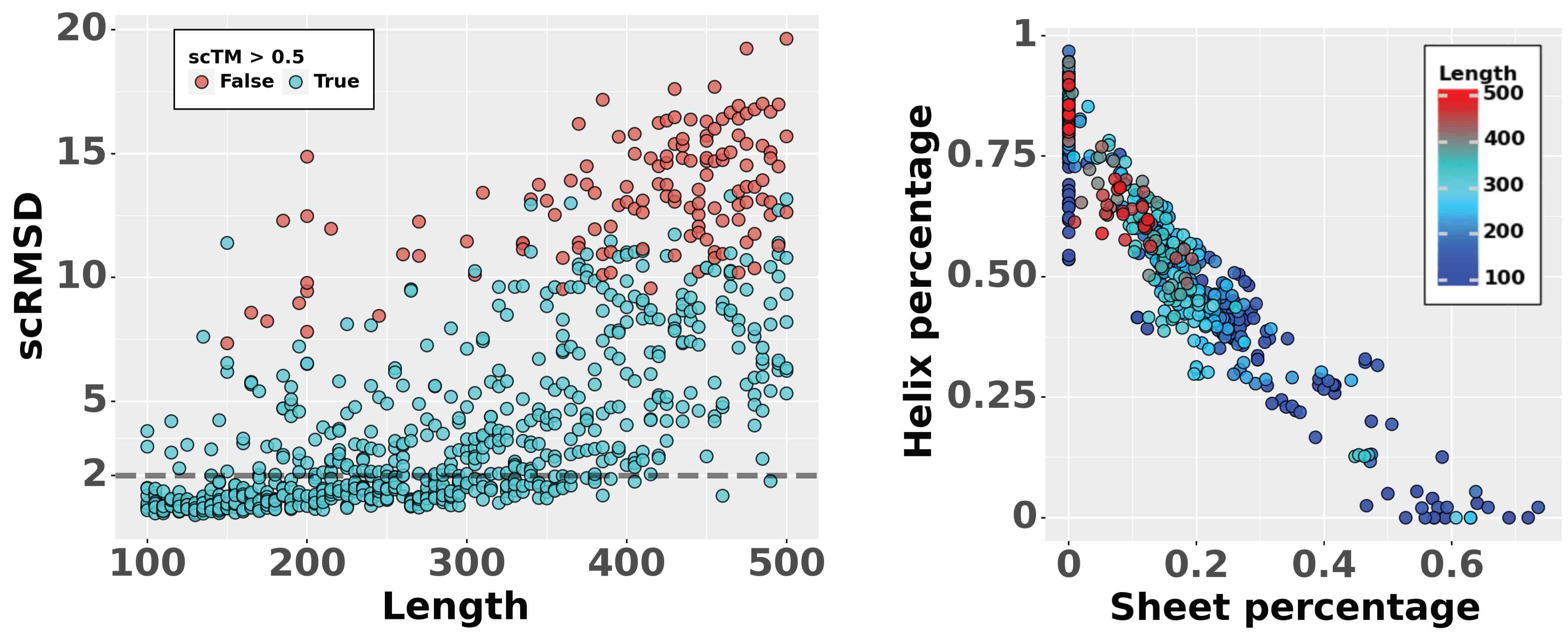}}
\caption{Additional analysis of designability across samples. Left: plot of \scrmsd{} vs. length with color to indicate \sctm{}$>0.5$ designability. Right: secondary structure composition of all samples across lengths.}
\label{fig:additional_results}
\end{center}
\vskip -0.3in
\end{figure*}

Since RFdiffusion's code is not released at time of writing, we reimplement the reported evaluation procedure in \citet{watson2022rfdiffusion} by sampling 100 backbones for each length 70, 100, 200, 300 and evaluating designability in \cref{fig:monomer_results}A.
With these samples, we additionally calculate diversity as the proportion of clusters out of 100 samples in \cref{table:rfdiffusion_diversity}.
Interestingly, our diversity remains high despite lower noise scales in contrast to the decreased diversity reported in RFdiffusion (their diversity is reported as a bar plot across noise scales).
However, \framediff{} designability across these lengths are lower than RFdiffusion.
Exploring how to jointly improve diversity and designability will be an important research direction.


\begin{table}[!ht]
\vskip -0.15in
\caption{Sample diversity compared with RFdiffusion.}
\begin{center}
\label{table:rfdiffusion_diversity}
\begin{small}
\begin{sc}
\begin{tabular}{c|ccc|ccc}
\toprule
 & \multicolumn{3}{c}{RFdiffusion} & \multicolumn{3}{c}{\framediff}\\
\backslashbox{Length}{Noise scale} & 0.0 & 0.5 & 1.0 & 0.1 & 0.5 & 1.0 \\
\midrule
70 & 0.16 & 0.26 & 0.34 & 0.72 & 0.67 & 0.9 \\
100 & 0.26 & 0.35 & 0.65 & 0.59 & 0.52 & 0.81 \\
200 & 0.29 & 0.65 & 0.83 & 0.49 & 0.67 & 0.86 \\
300 & 0.17 & 0.67 & 0.91 & 0.62 & 0.52 & 0.8 \\
\bottomrule
\end{tabular}
\end{sc}
\end{small}
\end{center}
\vskip -0.15in
\end{table}

\begin{figure*}[!ht]
\begin{center}
\centerline{\includegraphics[width=\textwidth]{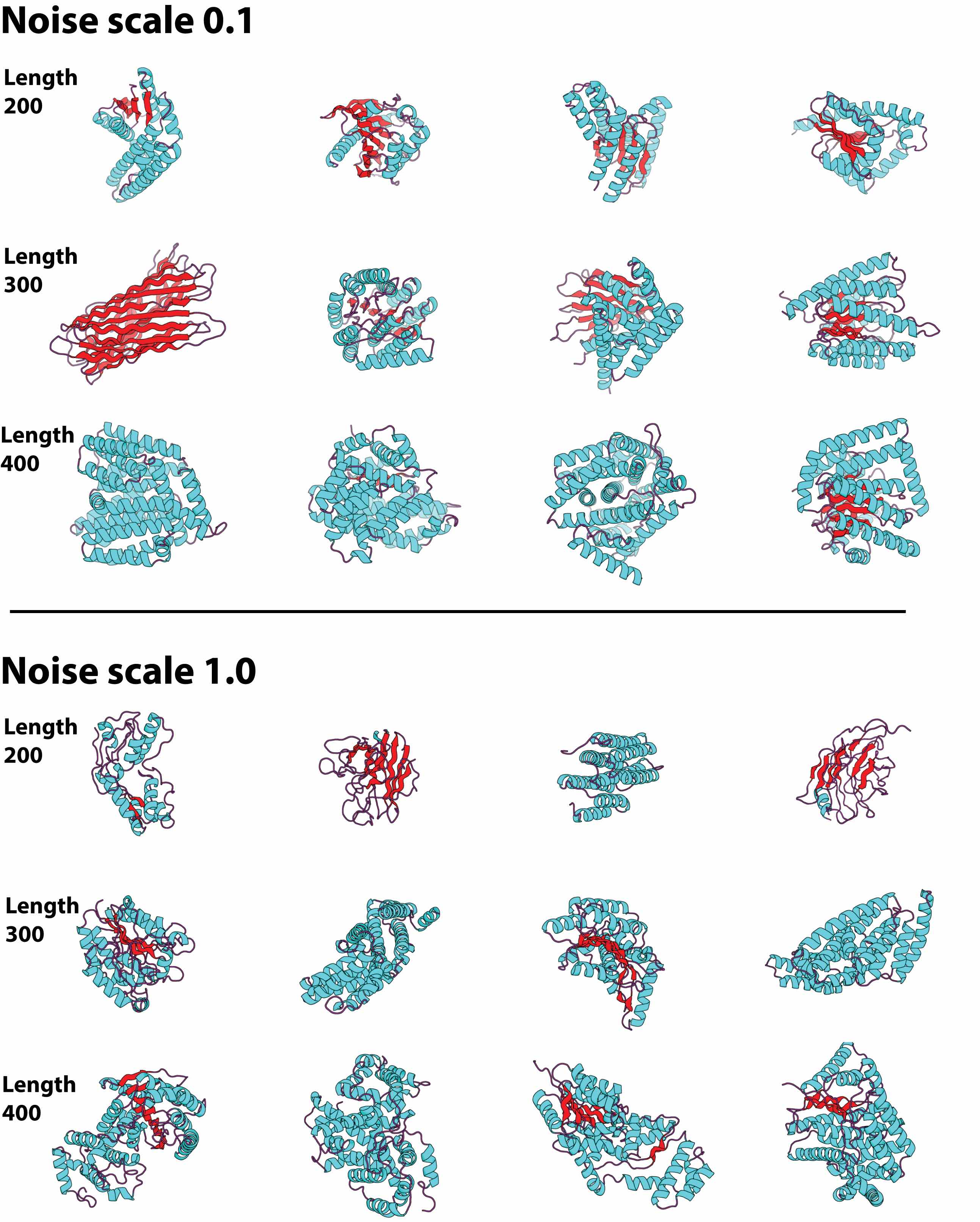}}
\caption{Visualization of samples across 4 random samples from each length group 200, 300, 40 at two different noise scales $\zeta=0.1,1.0$. Beta sheets are colored in red, alpha helices in cyan, and loops in magenta.}

\label{fig:framediff_samples}
\end{center}
\vskip -0.3in
\end{figure*}



\subsection{Comparison to FoldingDiff}
\label{sec:foldingdiff}

In this section, we compare our results with FoldingDiff \citep{wu2022protein}, a torsion angle based protein backbone diffusion model, which has publicly available code\footnote{\href{https://github.com/microsoft/foldingdiff}{https://github.com/microsoft/foldingdiff}} allowing for direct comparison.
However, the published FoldingDiff weights are limited to generating proteins up to length 128. 
We re-trained FoldingDiff and performed evaluation on proteins up to length 500 on the same dataset used to train FrameDiff.
To be as fair as possible, we use the evaluation code in FoldingDiff (i.e. OmegaFold for structure prediction, ProteinMPNN C-alpha only for sequence design) to evaluate FrameDiff for which we used noise scale $\zeta=1.0, N_{\mathrm{STEPS}}=500, N_{\mathrm{SEQ}}=8$.
The results are in \Cref{table:foldingdiff} where we see FrameDiff greatly outperforms FoldingDiff.
The increase to 60\% from 49\% designability (
\Cref{table:designability}) is due to the switch from full backbone ProteinMPNN to Ca only ProteinMPNN.
\begin{table}[!ht]
\caption{\framediff{} comparison to FoldingDiff.}
\begin{center}
\label{table:foldingdiff}
\begin{small}
\begin{sc}
\begin{tabular}{lcc}
\toprule
 & FoldingDiff & FrameDiff  \\
\midrule
$>0.5$ \sctm{} ($\uparrow$) & 6\% & 60\%  \\
\bottomrule
\end{tabular}
\end{sc}
\end{small}
\end{center}
\vskip -0.2in
\end{table}



\end{document}